\documentclass[journal]{IEEEtai}

\usepackage{amsmath,amssymb,amsfonts,amsthm,bm,mathtools}
\usepackage{algorithm}
\usepackage[noend]{algpseudocode}
\usepackage{booktabs}
\usepackage{multirow}
\usepackage{array}
\usepackage{makecell}
\usepackage{graphicx}
\usepackage{adjustbox}
\usepackage{xcolor}
\usepackage{textcomp}
\usepackage[T1]{fontenc}
\usepackage[figuresright]{rotating}
\usepackage{cite}
\usepackage{url}

\usepackage{tikz}
\usetikzlibrary{shapes.geometric, arrows.meta, positioning, fit, backgrounds, calc, decorations.pathreplacing, patterns}

\usepackage{hyperref}
\hypersetup{colorlinks=true,linkcolor=blue,citecolor=blue,urlcolor=blue}

\extrafloats{100}

\definecolor{dstInk}    {HTML}{1F2A37}
\definecolor{dstLine}   {HTML}{4B5563}
\definecolor{dstSoft}   {HTML}{E5E7EB}
\definecolor{dstTemp}   {HTML}{DCE7F4}
\definecolor{dstTempB}  {HTML}{6F8FB6}
\definecolor{dstSpec}   {HTML}{E3E1EE}
\definecolor{dstSpecB}  {HTML}{7A6FA6}
\definecolor{dstFuse}   {HTML}{D9E7DE}
\definecolor{dstFuseB}  {HTML}{5E8B73}
\definecolor{dstHead}   {HTML}{F1E7D6}
\definecolor{dstHeadB}  {HTML}{A88B55}
\definecolor{dstPre}    {HTML}{F3F4F6}
\definecolor{dstAccent} {HTML}{B45454}
\definecolor{dstSpeedup}{HTML}{2F6F4E}
\definecolor{dstStreamA} {HTML}{DCE7F4}
\definecolor{dstStreamB} {HTML}{E3E1EE}
\definecolor{dstStreamAB}{HTML}{7A6FA6}

\newtheorem{theorem}{Theorem}

\newtheorem{proposition}{Proposition}
\newtheorem{corollary}{Corollary}
\theoremstyle{remark}
\newtheorem{remark}{Remark}
\theoremstyle{plain}

\makeatletter
\@ifundefined{IEEEImpStatement}{%
    {\par\noindent\bfseries\itshape Impact Statement:\ \normalfont\bfseries}%
    {\par\medskip}}{}
\makeatother

\providecommand{\cmark}{\textcolor{dstFuseB}{\checkmark}}
\providecommand{\pmark}{\textcolor{dstHeadB}{$\circ$}}
\providecommand{\xmark}{\textcolor{dstAccent}{$\times$}}

\newcommand{\mc}[3]{%
\makecell[c]{#1\\[-1.2pt]{\fontsize{6}{6.5}\selectfont[#2, #3]}}}
\newcommand{\mcb}[3]{%
\makecell[c]{\textbf{#1}\\[-1.2pt]{\fontsize{6}{6.5}\selectfont[#2, #3]}}}

\begin{document}

\title{DSTNet: Dynamic Spectral Trajectory Network for Causal Multi-Horizon Financial Forecasting}

\author{Aashish Bohra, Lokenra Vishwakarma%
\thanks{A. Bohra is with the Department of Computer Science and Engineering, Indian Institute of Technology Jodhpur, Jodhpur, Rajasthan 342030, India (e-mail: bohra.1@iitj.ac.in).}
\thanks{L. Vishwakarma is with the School of Computing and Data Sciences, FLAME University, Pune Maharashtra, India (e-mail: lokendra.vishwakarma@flame.edu.in).}}

\maketitle

\begin{abstract}
Wavelet-based financial forecasters typically use the transform only to denoise, or reduce it to a single spectral snapshot at the forecast origin, and the convolution that produces the coefficients is usually bilateral, so it can read past the forecast origin. DSTNet instead retains the recent evolution of filter-bank magnitudes as a causal Dynamic Spectral Trajectory, built from seven trailing technical indicators over a twenty-day lookback with a one-sided Morlet-derived filter bank and an explicit burn-in for the left-boundary transient. A factorized Scale-Temporal Spectral Transformer attends along the time and filter-bank axes separately, a learned gate fuses the spectral branch with a CNN-BiLSTM, and horizon-specific gates emit one, three, five, and ten day forecasts in a single pass. We evaluate seven equity indices and gold under a common expanding-window protocol and an untouched one-year hold-out, against nine learned baselines and a random-walk persistence benchmark. Under MAE and MAPE, persistence is the strongest of the ten fixed competitors in 29 of the 32 series-horizon cells and DSTNet is the only model below it in every cell, by 0.7 to 0.9 percent at one day and 3.4 to 4.5 percent at ten days. At one day, paired testing favours DSTNet against the weaker learned baselines but is inconclusive against persistence and the strongest learned forecasters. A downstream allocation diagnostic does not support an equity-timing advantage on any of the seven indices.
\end{abstract}


\begin{IEEEkeywords}
Time-frequency analysis, causal complex bandpass filtering, multi-horizon forecasting, factorized attention, financial time series, deep learning.
\end{IEEEkeywords}

\section{Introduction}
\label{sec:intro}
\IEEEPARstart{P}{redicting} equity indices at multiple horizons in a single forward pass requires one model to represent target dynamics that may differ across forecast scales. Information that is useful at one horizon need not receive the same emphasis at another. This motivates a representation that preserves how spectral structure evolves across the observed window and a decoder that can modulate the fused representation separately for each forecast horizon. We treat the relation between forecast horizon and spectral reliance as an empirical question rather than assuming a fixed frequency-to-horizon correspondence~\cite{corsi2009simple, genccay2002introduction}. The distinction we draw is between a spectral state and a spectral process. A static filter-bank slice records where coefficient magnitudes are concentrated at the forecast origin, but it does not retain how that configuration developed across the preceding window. Two windows can therefore share a similar final spectral state while following different recent spectral paths~\cite{torrence1998,ang2012regime,rua2009international}.

Recent transformer forecasters have made impressive progress on the second half of this problem~\cite{informer2021,autoformer2021,patchtst2023,itransformer2024,timemixer2024}. The first half has fared less well. Wavelet-based hybrids exist~\cite{ziolkowski2026enhancing,li2026importance}. Bilateral wavelet preprocessing can introduce look-ahead leakage when coefficients are computed using observations beyond the forecast origin or when preprocessing is performed on the full series before chronological splitting. Such leakage can inflate reported forecasting accuracy. The few works that do use causal wavelets typically consume only a single time-slice of the result, discarding the temporal evolution of coefficient magnitudes that is the very feature making wavelets attractive in the first place. The proposed DSTNet addresses both halves with one architecture. The contributions are as follows:

\begin{enumerate}
    \item We introduce the Dynamic Spectral Trajectory (DST) tensor, $\mathbf{D}^{(i)} \in \mathbb{R}^{w \times S \times M}$, which represents the evolution of causal spectral responses over the complete $w$-step lookback window across $S$ fixed FIR scales and $M$ technical indicators. Unlike static spectral representations that retain only a single temporal snapshot, DST explicitly preserves the scale-temporal evolution of each input sample, enabling the forecasting model to exploit changes in spectral structure over time. The contribution of this trajectory representation is isolated through the No-DST ablation.
    \item We develop a strictly causal end-to-end time--frequency processing pipeline designed to prevent future-information leakage during forecasting. The pipeline combines a one-sided FIR wavelet denoiser with thresholds estimated exclusively from the training data and a strictly causal Morlet filter bank with explicit burn-in exclusion. Proposition~\ref{thm:causality} establishes that the spectral transformation depends only on observations available at or before the forecast origin, while Corollary~\ref{lem:burnin} derives the required burn-in constant. This provides an explicit causality guarantee for the complete spectral representation rather than relying solely on chronological train--test partitioning.
    \item We propose a factorized Scale--Temporal Spectral Transformer (STST) that separately models temporal dependencies and cross-scale interactions within the DST representation before integrating the resulting streams through a cross-stream MLP. This factorization reduces the attention-score complexity from $\mathcal{O}((wS)^2d)$ for joint scale--time attention to $\mathcal{O}((w^2+S^2)d)$. Theorem~\ref{thm:complexity} formally derives this reduction and distinguishes the equal-width theoretical ratio from the reductions obtained under the configured embedding widths and at the complete-model level. A joint-STST ablation quantifies the empirical trade-off introduced by the factorized attention structure.
\end{enumerate}
The rest of the paper proceeds as follows. Section~\ref{sec:related} reviews prior work. Section~\ref{sec:methodology} presents DSTNet, with formal correctness arguments deferred to Section~\ref{sec:theory}. Section~\ref{sec:experiments} reports the empirical evaluation. Section~\ref{sec:conclusion} concludes.

\section{Related Work}\label{sec:related}
We position DSTNet against the wavelet and deep-learning hybrids built for equity series, reviewed in four threads and followed by a feature comparison.

\subsection{Time-Frequency and Wavelet Methods in Financial Forecasting}\label{ssec:rw-wavelet}
Multiresolution analysis through the discrete wavelet transform rests on the work of Mallat~\cite{mallat1989} and Daubechies~\cite{daubechies1992}, and the shrinkage denoisers that finance later borrowed were formalised by Donoho and Johnstone~\cite{donoho1994ideal,donoho1995}. The continuous wavelet transform~\cite{torrence1998} buys a finer time-frequency picture at the price of redundancy. Financial applications followed quickly, because price series are non-stationary and multi-scale by nature, and a wavelet basis exposes both properties at once~\cite{minu2010,sezer2020}. Two habits recur across two decades of this work. First, the wavelet step is used as a denoiser and nothing more. The cleaned series is handed to an off-the-shelf network, and the explicit spectrum is thrown away. Ziolkowski et al.~\cite{ziolkowski2026enhancing} are a recent and clean instance of this pattern, pairing a discrete wavelet transform with an LSTM and reporting gains over the raw-input recurrent baseline. Li et al.~\cite{li2026importance} push the same idea harder, applying a level-one stationary wavelet transform to multivariate inputs before a multi-layer bidirectional LSTM, with large error reductions over the undenoised recurrent model on two US equities. The denoiser-only design is effective, and it is also a ceiling. WaVeFuse~\cite{bohra2026wavefuse} extends this direction by pairing channel-wise wavelet denoising with learned vertical attention fusion for regime-adaptive index forecasting, but it does not construct the causal spectral-trajectory representation developed here. Recent forecasting research has also moved directly toward frequency-aware and multi-resolution representations. FEDformer introduces frequency-enhanced decomposition within a Transformer~\cite{zhou2022fedformer}, while FreTS and FilterNet learn forecasting structure directly through frequency-domain operations~\cite{yi2023frequency,yi2024filternet}. Multi-scale structure is addressed from different architectural perspectives by NHITS, MTST, Pathformer, and WaveletMixer~\cite{challu2023nhits,zhang2024multi,chen2024pathformer,zhang2025waveletmixer}. These studies strengthen the motivation for scale-aware forecasting but do not remove the need to control temporal causality when the spectral representation itself is constructed at each forecast origin.

The reviewed wavelet-finance studies do not explicitly address whether their transform is causal at the forecast origin (Section~\ref{sec:intro}). We do not estimate the resulting bias in prior work, which would require rerunning each method under a causal transform. We instead remove the bilateral leakage path and prove that the implemented front end is causal.

\subsection{Decomposition-Based Hybrids for Equities}\label{ssec:rw-decomp}
A large and still-growing line attacks non-stationarity by splitting the raw series into frequency-specific subseries before any network sees it. Rezaei et al.~\cite{rezaei2021stock} run CEEMD and EMD on closing prices of four major indices, extract local features from each intrinsic mode function with a shallow 1D convolution, forecast each with an LSTM, and recombine linearly, which reduces error against CEEMD-LSTM but stays shallow and price-only. Ge~\cite{ge2025enhancing} decomposes OHLCV across eight indices with multivariate empirical mode decomposition and tunes the LSTM with an Aquila optimiser, reporting strong fit and resilience through the COVID-19 and Russia-Ukraine windows, at a heavy population-based search cost. Yu et al.~\cite{yu2025intelligent} select variational mode decomposition parameters by minimising sample entropy with a genetic algorithm, then forecast one-to-ten-step closing prices with a temporal convolutional network, clearing the modified Diebold-Mariano test at the one-percent level while remaining univariate. Gong and Xing~\cite{gong2024predicting} go to two stages, applying ICEEMDAN and then a particle-swarm-tuned VMD to the highest-frequency mode, before a BiLSTM-SAM-TCN ensemble predicts daily high and low indices with low percentage error. The common mechanism is sound and the common limitation is consistent. Decomposition is applied to raw price, not to the technical indicators that carry momentum and volatility structure, the per-scale streams are aggregated rather than processed channel by channel, and the recombination is a fixed linear or concatenation step rather than a learned data-dependent gate. None of these methods builds a causal spectral trajectory, and none carries a leakage guarantee. Related work has also shown that explicit treatment of distributional non-stationarity can materially affect forecasting behaviour, including through de-stationary attention and stationarity correction mechanisms~\cite{liu2022non,ma2024u}.

\subsection{Hybrid Deep Architectures for Equity Index Forecasting}\label{ssec:rw-hybrid}
A parallel line drops explicit upfront decomposition and instead stacks complementary neural blocks. Bhandari et al.~\cite{bhandari2022predicting} show that a single-layer LSTM on Haar-denoised, indicator-augmented S\&P~500 data can beat deeper variants, with significance during crisis windows, which is a useful reminder that depth is not the deciding factor. Thakkar and Chaudhari~\cite{thakkar2022information} co-optimise LSTM topology and a binary indicator-selection mask with an information-fusion genetic algorithm across nine datasets, with sizeable error and fit gains over a standard genetic baseline, though the topology is fixed and the indicator choice is binary rather than multi-scale. Tian et al.~\cite{tian2025bimt} fuse a bidirectional LSTM, a modified Transformer encoder, and a TCN decoder over short closing-price windows, with ablations confirming that the three blocks contribute jointly. Ji et al.~\cite{ji2024galformer} propose Galformer, a Transformer with non-autoregressive generative decoding and a hybrid mean-squared-error-plus-trend loss, which speeds multi-step inference by more than two orders of magnitude over Informer and lifts directional accuracy. Zubair and Huang~\cite{zubair2026enhanced} build an encoder-decoder around a BiGRU core with an autoencoder feature stage, attention weighting, and skip connections, with large percentage-error reductions across nine cross-market assets. MASTER~\cite{master2024} takes a different and finance-native route, modelling momentary cross-stock correlations through a market-guided Transformer. What unites this group is that fusion, where it exists at all, is static. Blocks are concatenated, summed, or combined through a fixed attention pattern. The model cannot reweight its branches sample by sample. Our vertical attention gate learns that weighting, and Section~\ref{ssec:ablations} reports how much it varies in practice. VertiFuseX~\cite{bohra2026vertifusex} moves in the same direction, fusing multiple temporal streams for generalisable financial forecasting, and DSTNet extends that line by introducing a causal spectral trajectory as a complementary representation and processing it through factorized scale-temporal attention before fusion. The broader financial machine-learning literature likewise shows that nonlinear recurrent, graph-based, and high-dimensional learning systems can extract useful predictive structure from financial data, while also placing strong emphasis on genuine out-of-sample evaluation~\cite{fischer2018deep,cheng2022financial,gu2020empirical}.

\subsection{Multi-Horizon and Trading-Oriented Forecasting in Finance}\label{ssec:rw-multihorizon}
Multi-horizon prediction comes in two flavours. Recursive schemes feed a forecast back as input and let errors compound. Direct schemes emit every horizon in one pass. The general literature settled this debate years ago in favour of direct heads, with N-BEATS~\cite{nbeats2020} popularising deep residual direct forecasting, the Temporal Fusion Transformer~\cite{lim2021tft} adding attention and covariate encoders, and DLinear~\cite{zeng2023dlinear} showing that a per-horizon linear map is a stubbornly strong baseline. Finance has been slower to follow. Most equity work still forecasts one step ahead, and the studies that do reach for multiple horizons often do so recursively or with a shared head. The exceptions matter for us. Yu et al.~\cite{yu2025intelligent} and Galformer~\cite{ji2024galformer} produce genuine multi-step outputs, and Beniwal et al.~\cite{beniwal2024forecasting} study multistep daily forecasting across global indices for long-horizon decisions. These studies generally use a shared decoder or conventional per-horizon heads. DSTNet's Horizon-Adaptive Decoder (HAD) instead applies an independent learned feature-modulation gate before each horizon head (Section~\ref{ssec:had}). Wang et al.~\cite{wang2025multifactor} round out the recent picture with a multifactor deep-learning model that widens the input set rather than the horizon set, a useful contrast to a model whose novelty is on the output side.

\subsection{Feature Comparison}\label{ssec:rw-table}
Table~\ref{tab:related-features} places thirteen representative finance models against DSTNet along five design dimensions, with each model's defining strength and primary gap stated so the comparison rewards genuine contributions rather than a binary checklist. The picture is not that prior work is weak. Several entries are strong on individual axes. Rather, among the methods in Table~\ref{tab:related-features} we did not identify one that combines spectral processing on technical indicators, an explicit causal guarantee, learned fusion, and direct multi-horizon forecasting under one training objective. The non-equity column records evaluation scope, not a methodological contribution.

\begin{table*}[!t]
\caption{Comparison of representative finance forecasting models against DSTNet. Symbols: \cmark\ full, \pmark\ partial, \xmark\ absent. The final two columns state each model's defining strength and its primary gap relative to DSTNet.}
\label{tab:related-features}
\centering
\footnotesize
\renewcommand{\arraystretch}{1.2}
\setlength{\tabcolsep}{4pt}
\begin{adjustbox}{max width=\textwidth}
\begin{tabular}{@{}l c c l c c >{\raggedright\arraybackslash}p{3.6cm} >{\raggedright\arraybackslash}p{3.6cm}@{}}
\toprule
\textbf{Model} & \makecell{\textbf{Spectral}\\\textbf{on TIs}} & \makecell{\textbf{Causal}\\\textbf{guarantee}} & \textbf{Fusion} & \makecell{\textbf{Multi-}\\\textbf{horizon}} & \makecell{\textbf{Non-equity}\\\textbf{eval.}} & \textbf{Defining strength} & \textbf{Primary gap vs.\ DSTNet} \\
\midrule
Rezaei et al.\ (2021)~\cite{rezaei2021stock}        & \xmark & \xmark & linear     & \xmark & \xmark & Per-IMF CNN features from CEEMD/EMD          & Raw-price decomposition, static aggregation \\
Bhandari et al.\ (2022)~\cite{bhandari2022predicting}& \xmark & \xmark & none       & \xmark & \xmark & Shallow LSTM beats deep on Haar-denoised input & Snapshot only, no spectral trajectory \\
Thakkar \& Chaudhari (2022)~\cite{thakkar2022information}& \xmark & \xmark & none    & \xmark & \xmark & GA co-selection of topology and indicators   & Binary indicator mask, no multi-scale \\
Galformer (2024)~\cite{ji2024galformer}             & \xmark & \xmark & none       & \cmark & \xmark & Non-autoregressive generative decoding        & No denoising or spectral encoding \\
Gong \& Xing (2024)~\cite{gong2024predicting}       & \xmark & \xmark & linear     & \xmark & \xmark & Two-stage ICEEMDAN then PSO-VMD              & Price-only, static recombination \\
MASTER (2024)~\cite{master2024}                     & \xmark & \xmark & static attn & \xmark & \xmark & Market-guided cross-stock correlation         & No wavelet path, single horizon \\
Ge (2025)~\cite{ge2025enhancing}                    & \xmark & \xmark & linear     & \xmark & \xmark & MEMD with Aquila-tuned LSTM                  & Decomposes prices not indicators \\
Yu et al.\ (2025)~\cite{yu2025intelligent}          & \xmark & \pmark & linear     & \cmark & \xmark & Entropy-guided VMD then TCN                  & Univariate, no causal proof \\
BiMT (2025)~\cite{tian2025bimt}                     & \xmark & \xmark & joint train & \xmark & \xmark & BiLSTM, Transformer and TCN co-trained       & No explicit adaptive gate \\
Ziolkowski et al.\ (2026)~\cite{ziolkowski2026enhancing}& \xmark & \xmark & none      & \xmark & \xmark & Clean DWT-LSTM denoiser                       & Spectrum discarded after denoising \\
Wang et al.\ (2025)~\cite{wang2025multifactor}      & \xmark & \xmark & none       & \xmark & \xmark & Widened multifactor input set                & Novelty on input not output side \\
Li et al.\ (2026)~\cite{li2026importance}           & \pmark & \xmark & none       & \xmark & \xmark & Level-1 SWT on multivariate input            & Single branch, two stocks only \\
Zubair \& Huang (2026)~\cite{zubair2026enhanced}    & \xmark & \xmark & attn+skip  & \cmark & \xmark & BiGRU encoder-decoder with attention          & No wavelet denoising or scales \\
\midrule
\textbf{DSTNet (ours)} & \cmark & \cmark & \makecell[l]{learned\\gate} & \cmark & \cmark & Causal DST trajectory tensor with factorized STST & Not applicable \\
\bottomrule
\end{tabular}
\end{adjustbox}
\end{table*}

\section{Methodology}\label{sec:methodology}
DSTNet decomposes multi-horizon equity index forecasting into four operations that act on physically distinct representations of the same underlying price series. We denoise the OHLCV channels with a causal undecimated Symlet-4 residual multiresolution denoiser, derive a strictly causal time-frequency representation via a Morlet-derived complex FIR filter bank, lift that representation into a per-sample Dynamic Spectral Trajectory tensor $\mathbf{D}^{(i)}$, and route the result through a dual-branch network whose temporal and spectral streams are combined by a learned vertical gate before a per-horizon decoder emits simultaneous forecasts at $\mathcal{H} = \{1, 3, 5, 10\}$ trading days.

Fig.~\ref{fig:dstnet-concept} contrasts the DST tensor with the static-snapshot representation used by several reviewed wavelet hybrids, and Fig.~\ref{fig:dstnet-arch} gives the full pipeline. Formal correctness arguments and complexity analyses are stated as theorems and proved in Section~\ref{sec:theory}. Algorithmic realisations appear in Appendix~\ref{app:algorithms}, with Algorithms~\ref{alg:preproc} and~\ref{alg:cwwt} covering preprocessing and the causal filter bank, and Algorithms~\ref{alg:temporal}, \ref{alg:stst}, and~\ref{alg:vaf-had} covering the two branches and the decoder. We use $T$ for the length of a denoised series, $w = 20$ for the lookback window, $\mathcal{S} = \{1, 2, 3, 5, 8, 13, 21, 32\}$ with $S = |\mathcal{S}| = 8$ for the fixed FIR scale parameters, $M = 7$ for the number of technical indicators, and $d_{\text{fus}} = 64$ for the fusion-vector dimension. The network target at every horizon is a cumulative log return, and every price-level metric in Section~\ref{sec:experiments} is computed after reconstruction to the price scale. Section~\ref{ssec:metrics} gives the exact convention.

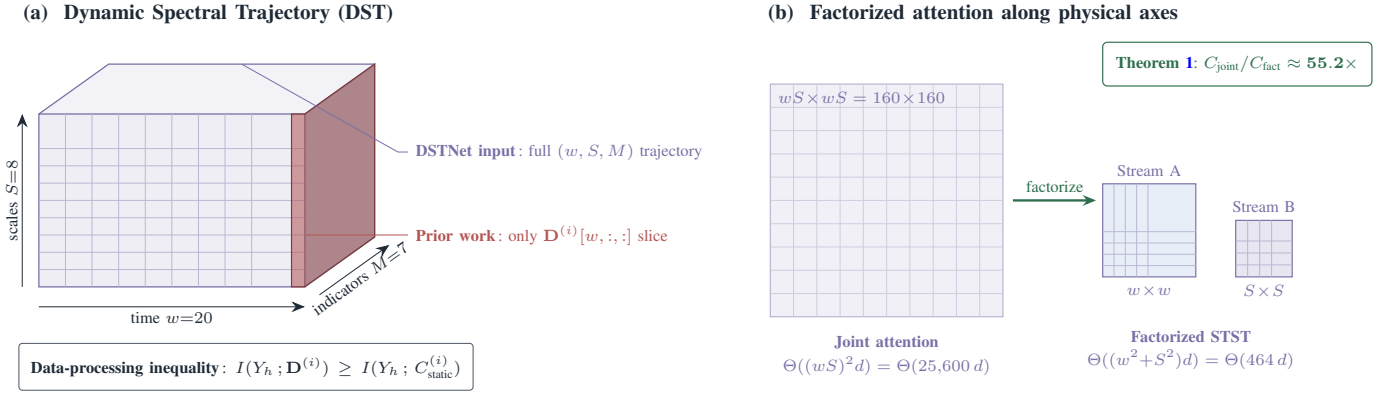
\begin{figure*}[!t]
\centering
\begin{adjustbox}{max width=\textwidth}
\begin{tikzpicture}[
    >={Stealth[length=2mm,width=1.6mm]},
    font=\footnotesize,
    every node/.style={inner sep=2pt, outer sep=0pt, text=dstInk},]

\begin{scope}[shift={(0,0)}]
  \pgfmathsetmacro{\W}{4.0}
  \pgfmathsetmacro{\H}{2.6}
  \pgfmathsetmacro{\Dx}{1.05}
  \pgfmathsetmacro{\Dy}{0.75}

  \fill[dstSpec, opacity=0.55]
        (0,0) -- (\W,0) -- (\W,\H) -- (0,\H) -- cycle;
  \fill[dstSpec!70!dstSpecB, opacity=0.55]
        (\W,0) -- (\W+\Dx,\Dy) -- (\W+\Dx,\H+\Dy) -- (\W,\H) -- cycle;
  \fill[dstSpec!50!white, opacity=0.85]
        (0,\H) -- (\W,\H) -- (\W+\Dx,\H+\Dy) -- (\Dx,\H+\Dy) -- cycle;

  \draw[dstSpecB, line width=0.55pt]
        (0,0) -- (\W,0) -- (\W,\H) -- (0,\H) -- cycle;
  \draw[dstSpecB, line width=0.55pt]
        (\W,0) -- (\W+\Dx,\Dy) -- (\W+\Dx,\H+\Dy) -- (\W,\H);
  \draw[dstSpecB, line width=0.55pt]
        (0,\H) -- (\Dx,\H+\Dy) -- (\W+\Dx,\H+\Dy);

  \foreach \t in {0.4,0.8,...,3.6}
      \draw[dstSpecB!50, line width=0.18pt] (\t,0) -- (\t,\H);
  \foreach \s in {0.26,0.52,...,2.34}
      \draw[dstSpecB!50, line width=0.18pt] (0,\s) -- (\W,\s);

  \fill[dstAccent, opacity=0.55]
        (\W-0.20,0) -- (\W,0) -- (\W,\H) -- (\W-0.20,\H) -- cycle;
  \fill[dstAccent!75!black, opacity=0.55]
        (\W,0) -- (\W+\Dx,\Dy) -- (\W+\Dx,\H+\Dy) -- (\W,\H) -- cycle;
  \draw[dstAccent!70!black, line width=0.5pt]
        (\W-0.20,0) -- (\W,0) -- (\W+\Dx,\Dy) -- (\W+\Dx,\H+\Dy) -- (\W,\H) -- (\W-0.20,\H) -- cycle;

  \draw[->, dstInk, line width=0.4pt] (0,-0.30) -- (\W,-0.30)
        node[midway, below, font=\scriptsize] {time $w{=}20$};
  \draw[->, dstInk, line width=0.4pt] (-0.25,0) -- (-0.25,\H)
        node[midway, left, font=\scriptsize, rotate=90, anchor=south] {scales $S{=}8$};
  \draw[->, dstInk, line width=0.4pt]
        (\W+0.22,-0.12) -- (\W+\Dx+0.22,\Dy-0.12)
        node[midway, sloped, below, font=\scriptsize] {indicators $M{=}7$};

  \node[anchor=south west, font=\small\bfseries] at (-0.3,\H+\Dy+0.55)
       {(a)\ \ Dynamic Spectral Trajectory (DST)};

  \node[anchor=west, font=\scriptsize, text=dstSpecB]
       at (\W+\Dx+0.55,\H*0.78)
       {\textbf{DSTNet input}\,: full $(w,S,M)$ trajectory};
  \node[anchor=west, font=\scriptsize, text=dstAccent]
       at (\W+\Dx+0.55,\H*0.30)
       {\textbf{Prior work}\,: only $\mathbf{D}^{(i)}[w,:,:]$ slice};

  \draw[dstSpecB, line width=0.4pt]
        (\W+\Dx+0.5,\H*0.78) -- ++(-0.15,0)
        -- ($(\Dx,\H+\Dy)!0.5!(\W+\Dx,\H+\Dy)$);
  \draw[dstAccent, line width=0.4pt]
        (\W+\Dx+0.5,\H*0.30) -- ++(-0.15,0)
        -- (\W,\H*0.30);

  \node[anchor=north west, draw=dstLine, line width=0.3pt, rounded corners=1.5pt,
        fill=white, inner sep=1.8mm, font=\scriptsize]
        at (-0.3,-0.85)
        {\textbf{Data-processing inequality}\,: $\,I(Y_h\,;\mathbf{D}^{(i)}) \,\geq\, I(Y_h\,;\,C^{(i)}_{\text{static}})$};
\end{scope}

\begin{scope}[shift={(11.0,-0.45)}]
  \pgfmathsetmacro{\Aw}{3.5}

  \draw[dstSpecB!60, line width=0.4pt, fill=dstSpec!50]
        (0,0) rectangle (\Aw,\Aw);
  \foreach \i in {0.35,0.7,...,3.15}
      \draw[dstSpecB!35, line width=0.15pt] (\i,0) -- (\i,\Aw);
  \foreach \i in {0.35,0.7,...,3.15}
      \draw[dstSpecB!35, line width=0.15pt] (0,\i) -- (\Aw,\i);
  \node[anchor=north west, font=\scriptsize, text=dstSpecB]
        at (0.05,\Aw-0.05) {$wS\!\times\!wS = 160\!\times\!160$};
  \node[anchor=south, font=\scriptsize\bfseries, text=dstSpecB]
        at (\Aw/2,-0.55) {Joint attention};
  \node[anchor=south, font=\scriptsize, text=dstSpecB]
        at (\Aw/2,-0.95) {$\Theta((wS)^2 d) = \Theta(25{,}600\,d)$};

  \pgfmathsetmacro{\Bw}{0.85}
  \pgfmathsetmacro{\Cw}{1.4}

  \begin{scope}[shift={(\Aw+1.5,0.6)}]
    \draw[dstSpecB, line width=0.5pt, fill=dstTemp!70]
          (0,0) rectangle (\Cw,\Cw);
    \foreach \i in {0.17,0.34,0.51,0.68}
        \draw[dstSpecB!50, line width=0.15pt] (\i,0) -- (\i,\Cw);
    \foreach \i in {0.17,0.34,0.51,0.68}
        \draw[dstSpecB!50, line width=0.15pt] (0,\i) -- (\Cw,\i);
    \node[anchor=south, font=\scriptsize, text=dstSpecB]
          at (\Cw/2,\Cw+0.05) {Stream A};
    \node[anchor=north, font=\scriptsize, text=dstSpecB]
          at (\Cw/2,-0.05) {$w\!\times\!w$};

    \draw[dstSpecB, line width=0.5pt, fill=dstSpec!85]
          (\Cw+0.6,0) rectangle (\Cw+0.6+\Bw,\Bw);
    \foreach \i in {0.18,0.36,...,0.72}
        \draw[dstSpecB!50, line width=0.15pt]
            (\Cw+0.6+\i,0) -- (\Cw+0.6+\i,\Bw);
    \foreach \i in {0.18,0.36,...,0.72}
        \draw[dstSpecB!50, line width=0.15pt]
            (\Cw+0.6,\i) -- (\Cw+0.6+\Bw,\i);
    \node[anchor=south, font=\scriptsize, text=dstSpecB]
          at (\Cw+0.6+\Bw/2,\Bw+0.05) {Stream B};
    \node[anchor=north, font=\scriptsize, text=dstSpecB]
          at (\Cw+0.6+\Bw/2,-0.05) {$S\!\times\!S$};

    \pgfmathsetmacro{\Fcx}{(\Cw + 0.6 + \Bw)/2 - 0.1}
    \node[anchor=south, font=\scriptsize\bfseries, text=dstSpecB]
        at ({\Fcx},-1.05) {Factorized STST};
    \node[anchor=south, font=\scriptsize, text=dstSpecB]
        at ({\Fcx},-1.45) {$\Theta((w^2{+}S^2)d)=\Theta(464\,d)$};
  \end{scope}

  \draw[->, dstSpeedup, line width=1pt]
        (\Aw+0.15,\Aw/2) -- (\Aw+1.4,\Aw/2);
  \node[anchor=south, font=\scriptsize, text=dstSpeedup]
        at (\Aw+0.78,\Aw/2+0.05) {factorize};

  \node[draw=dstSpeedup, line width=0.5pt, rounded corners=2pt,
        fill=white, inner sep=2mm, font=\scriptsize, text=dstSpeedup,
        anchor=west] at (\Aw+1.5,\Aw+0.3)
        {\textbf{Theorem~\ref{thm:complexity}}: $C_{\text{joint}}/C_{\text{fact}} \approx \mathbf{55.2\times}$};

  \node[anchor=south west, font=\small\bfseries] at (-0.1,\Aw+0.85)
       {(b)\ \ Factorized attention along physical axes};
\end{scope}
\end{tikzpicture}%
\end{adjustbox}
\caption{Conceptual contributions of DSTNet. (a)~The Dynamic Spectral Trajectory tensor $\mathbf{D}^{(i)}\!\in\!\mathbb{R}^{w\times S\times M}$ retains the full evolution of filter-bank magnitudes across the lookback window. The static-snapshot representation used by several reviewed wavelet hybrids is the single time-slice $\mathbf{D}^{(i)}[w,:,:]$ highlighted in red. Because this slice is a deterministic projection of the full tensor, the data-processing inequality gives the displayed information ordering. (b)~Factorized Scale-Temporal Spectral Transformer (STST) replaces a $160{\times}160$ joint attention matrix by two small attentions of size $w{\times}w$ and $S{\times}S$, yielding the $\approx 55.2\times$ equal-width per-layer attention-score ratio quantified by Theorem~\ref{thm:complexity}. This ratio does not represent total model FLOPs or wall-clock speedup.}
\label{fig:dstnet-concept}
\end{figure*}

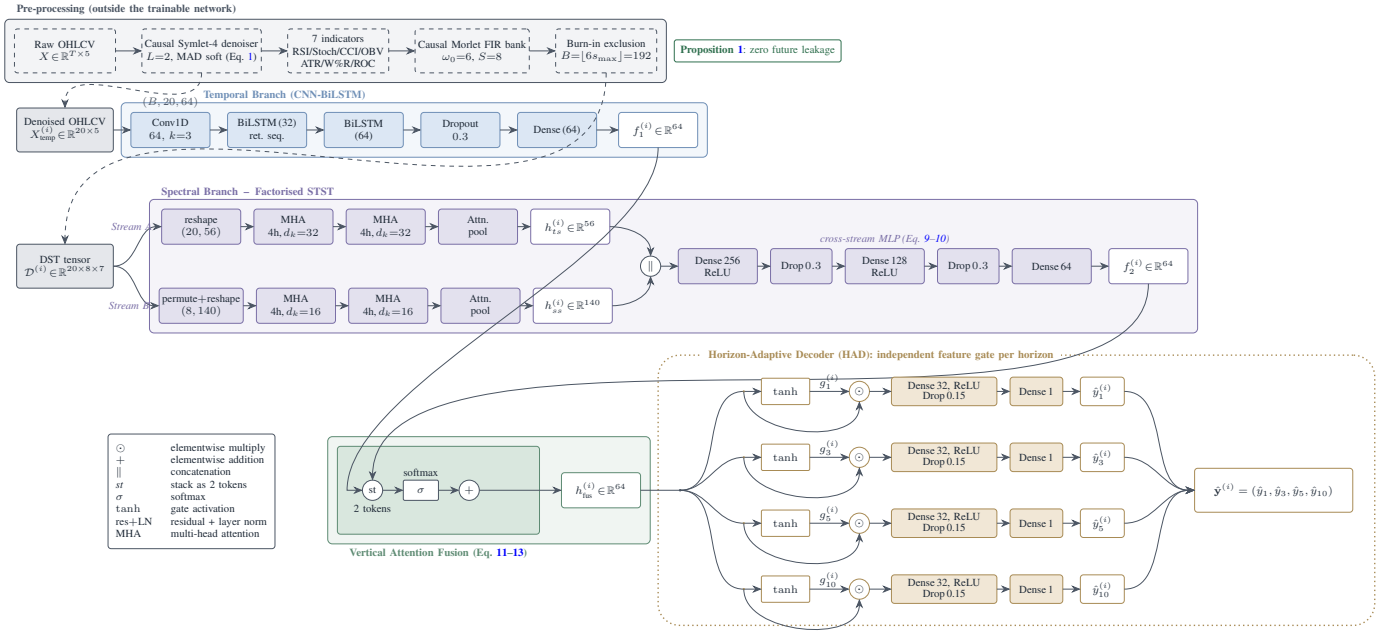
\begin{figure*}[ht]
\centering
\resizebox{\textwidth}{!}{%
\begin{tikzpicture}[
    >={Stealth[length=2mm,width=1.6mm]},
    font=\footnotesize,
    every node/.style={inner sep=2pt, outer sep=0pt, text=dstInk},
    layer/.style    ={rectangle, rounded corners=2pt, draw=dstLine, line width=0.45pt,
                      fill=white, minimum height=8mm, minimum width=18mm,
                      align=center, font=\scriptsize},
    tlayer/.style   ={layer, fill=dstTemp,  draw=dstTempB},
    slayer/.style   ={layer, fill=dstSpec,  draw=dstSpecB},
    flayer/.style   ={layer, fill=dstFuse,  draw=dstFuseB},
    hlayer/.style   ={layer, fill=dstHead,  draw=dstHeadB},
    inp/.style      ={layer, fill=dstSoft,  draw=dstLine, minimum width=22mm,
                      minimum height=10mm},
    pre/.style      ={layer, fill=dstPre,   draw=dstLine, dashed,
                      minimum width=23mm, minimum height=10mm},
    op/.style       ={circle, draw=dstLine, fill=white, line width=0.4pt,
                      minimum size=4.6mm, inner sep=0pt, font=\scriptsize},
    gate/.style     ={rectangle, rounded corners=1pt, draw=dstLine, fill=white,
                      minimum height=4.6mm, minimum width=8mm, font=\scriptsize},
    flow/.style     ={->, draw=dstLine, line width=0.45pt},
    sflow/.style    ={->, draw=dstLine, line width=0.4pt, dashed},
    wire/.style     ={draw=dstLine, line width=0.45pt},
    mystyle/.style  ={font=\scriptsize\itshape, text=dstLine},
    lane/.style     ={rounded corners=3pt, line width=0.3pt},
]

\node[pre] (raw) at (0,6.0)
      {Raw OHLCV\\$X\!\in\!\mathbb{R}^{T\times 5}$};
\node[pre, right=6mm of raw]  (dwt)
      {Causal Symlet-4 denoiser\\\scriptsize $L{=}2$, MAD soft (Eq.~\ref{eq:mad-threshold})};
\node[pre, right=6mm of dwt]  (ind)
      {7 indicators\\\scriptsize RSI/Stoch/CCI/OBV\\\scriptsize ATR/W\%R/ROC};
\node[pre, right=6mm of ind]  (cwt)
      {Causal Morlet FIR bank\\\scriptsize $\omega_0{=}6$, $S{=}8$};
\node[pre, right=6mm of cwt]  (burn)
      {Burn-in exclusion\\\scriptsize $B{=}\lfloor 6 s_{\max}\rfloor{=}192$};

\draw[flow] (raw)  -- (dwt);
\draw[flow] (dwt)  -- (ind);
\draw[flow] (ind)  -- (cwt);
\draw[flow] (cwt)  -- (burn);

\node[draw=dstSpeedup, line width=0.4pt, rounded corners=1.5pt,
      fill=white, inner sep=1.4mm, font=\scriptsize, text=dstSpeedup,
      anchor=west] at ($(burn.east)+(4mm,0)$)
      {\textbf{Proposition~\ref{thm:causality}}: zero future leakage};

\node[inp] (inOHLCV) at (0,4.2)
      {Denoised OHLCV\\$X^{(i)}_{\text{temp}}\!\in\!\mathbb{R}^{20\times 5}$};
\node[tlayer, right=4mm of inOHLCV]  (tConv)  {Conv1D\\$64,\,k{=}3$};
\node[tlayer, right=4mm of tConv]    (tL1)    {BiLSTM\,(32)\\\scriptsize ret.\ seq.};
\node[tlayer, right=4mm of tL1]      (tL2)    {BiLSTM\\(64)};
\node[tlayer, right=4mm of tL2]      (tDrop)  {Dropout\\$0.3$};
\node[tlayer, right=4mm of tDrop]    (tDense) {Dense\,(64)};
\node[layer,  right=5mm of tDense, fill=white,
              draw=dstTempB, line width=0.6pt] (f1)
      {$f_1^{(i)}\!\in\!\mathbb{R}^{64}$};

\draw[flow] (inOHLCV) -- (tConv);
\draw[flow] (tConv)   -- (tL1);
\draw[flow] (tL1)     -- (tL2);
\draw[flow] (tL2)     -- (tDrop);
\draw[flow] (tDrop)   -- (tDense);
\draw[flow] (tDense)  -- (f1);

\node[mystyle, above=0.3mm of tConv] {$(B,20,64)$};

\draw[flow, dashed] (dwt.south) to[out=-90,in=90] (inOHLCV.north);

\node[inp] (inDST) at (0,1.1)
      {DST tensor\\$\mathcal{D}^{(i)}\!\in\!\mathbb{R}^{20\times 8\times 7}$};

\node[slayer] (aRes)  at (3.1,2.0) {reshape\\$(20,56)$};
\node[slayer, right=3mm of aRes]  (aMHA1) {MHA\\\scriptsize 4h,\,$d_k{=}32$};
\node[slayer, right=3mm of aMHA1] (aMHA2) {MHA\\\scriptsize 4h,\,$d_k{=}32$};
\node[slayer, right=3mm of aMHA2] (aPool) {Attn.\\pool};
\node[layer,  right=3mm of aPool, fill=white,
              draw=dstSpecB, line width=0.55pt] (hts)
      {$h_{ts}^{(i)}\!\in\!\mathbb{R}^{56}$};

\node[slayer] (bRes)  at (3.1,0.2) {permute$+$reshape\\$(8,140)$};
\node[slayer, right=3mm of bRes]  (bMHA1) {MHA\\\scriptsize 4h,\,$d_k{=}16$};
\node[slayer, right=3mm of bMHA1] (bMHA2) {MHA\\\scriptsize 4h,\,$d_k{=}16$};
\node[slayer, right=3mm of bMHA2] (bPool) {Attn.\\pool};
\node[layer,  right=3mm of bPool, fill=white,
              draw=dstSpecB, line width=0.55pt] (hss)
      {$h_{ss}^{(i)}\!\in\!\mathbb{R}^{140}$};

\draw[flow] (inDST.east) to[out=0,in=180] (aRes.west);
\draw[flow] (inDST.east) to[out=0,in=180] (bRes.west);

\draw[flow] (aRes)  -- (aMHA1);
\draw[flow] (aMHA1) -- (aMHA2);
\draw[flow] (aMHA2) -- (aPool);
\draw[flow] (aPool) -- (hts);

\draw[flow] (bRes)  -- (bMHA1);
\draw[flow] (bMHA1) -- (bMHA2);
\draw[flow] (bMHA2) -- (bPool);
\draw[flow] (bPool) -- (hss);

\draw[sflow] (burn.south) to[out=-90,in=90] (inDST.north);

\node[font=\scriptsize\itshape, text=dstSpecB,
      left=1.5mm of aRes] {Stream A};
\node[font=\scriptsize\itshape, text=dstSpecB,
      left=1.5mm of bRes] {Stream B};

\node[op] (cat) at ($(hts.east)!0.5!(hss.east) + (0.9,0)$) {$\Vert$};
\draw[flow] (hts.east) to[out=0,in=90]  (cat.north);
\draw[flow] (hss.east) to[out=0,in=-90] (cat.south);

\node[slayer, right=4mm of cat, minimum width=18mm] (mlp1)
      {Dense\,256\\ReLU};
\node[slayer, right=3mm of mlp1, minimum width=14mm] (mlp2) {Drop\,$0.3$};
\node[slayer, right=3mm of mlp2, minimum width=18mm] (mlp3)
      {Dense\,128\\ReLU};
\node[slayer, right=3mm of mlp3, minimum width=14mm] (mlp4) {Drop\,$0.3$};
\node[slayer, right=3mm of mlp4, minimum width=18mm] (mlp5) {Dense\,64};
\node[layer, right=4mm of mlp5, fill=white,
      draw=dstSpecB, line width=0.6pt] (f2)
      {$f_2^{(i)}\!\in\!\mathbb{R}^{64}$};

\draw[flow] (cat.east) -- (mlp1.west);
\draw[flow] (mlp1) -- (mlp2);
\draw[flow] (mlp2) -- (mlp3);
\draw[flow] (mlp3) -- (mlp4);
\draw[flow] (mlp4) -- (mlp5);
\draw[flow] (mlp5) -- (f2);

\node[font=\scriptsize\itshape, text=dstSpecB,
      above=0.3mm of mlp3] {cross-stream MLP\,(Eq.~\ref{eq:cross}--\ref{eq:fv2})};

\node[flayer, minimum width=46mm, minimum height=20mm] (vaf)
      at (8.5,-4) {};
\node[below=2.5mm of vaf.south, font=\scriptsize\bfseries, text=dstFuseB]
      {Vertical Attention Fusion (Eq.~\ref{eq:vaf-stack}--\ref{eq:vaf-fused})};

\node[op]   (stack) at ($(vaf.west)+(8mm,0)$)  {st};
\node[gate] (sm)    at ($(stack)+(11mm,0)$)    {$\sigma$};
\node[op]   (sum)   at ($(sm)+(11mm,0)$)       {$+$};

\draw[flow] (stack) -- (sm);
\draw[flow] (sm)    -- (sum);

\node[font=\scriptsize, above=0.2mm of sm]    {softmax};
\node[font=\scriptsize, below=0.2mm of stack] {2 tokens};

\draw[flow] (f1.south) to[out=-90,in=90,looseness=0.6]
            ($(stack.west)+(-0.35,0mm)$) -- (stack.west);
\draw[flow] (f2.south) to[out=-90,in=90,looseness=0.6]
            ($(stack.north)+(0,1mm)$) -- (stack.north);

\node[layer, right=5mm of vaf, fill=white,
      draw=dstFuseB, line width=0.6pt] (hfused)
      {$h_{\text{fus}}^{(i)}\!\in\!\mathbb{R}^{64}$};
\draw[flow] (sum.east) -- (vaf.east) -- (hfused.west);

\coordinate (hadStart) at ($(hfused.east)+(9mm,0)$);

\draw[wire] (hfused.east) -- (hadStart);
\fill[dstLine] (hadStart) circle (0.8pt);

\foreach \i/\h/\dy in {1/1/22.5, 2/3/7.5, 3/5/-7.5, 4/10/-22.5} {

    \node[gate, minimum width=11mm, minimum height=6mm,
          draw=dstHeadB, line width=0.5pt]
          (gt\i) at ($(hadStart)+(24mm,\dy mm)$) {$\tanh$};

    \coordinate (sp\i) at ($(gt\i.west)+(-4mm,0)$);

    \node[op, draw=dstHeadB, right=9mm of gt\i] (g\i) {$\odot$};

    \node[hlayer, minimum width=24mm, minimum height=6.5mm, right=5mm of g\i]
          (mA\i) {Dense\,32, ReLU\\[-1pt]\scriptsize Drop\,0.15};
    \node[hlayer, minimum width=12mm, minimum height=6.5mm, right=3mm of mA\i]
          (mB\i) {Dense\,1};

    \node[layer, draw=dstHeadB, line width=0.6pt, fill=white,
          minimum width=10mm, minimum height=6.5mm, right=4mm of mB\i]
          (yh\i) {$\hat{y}_{\h}^{(i)}$};

    \draw[wire] (hadStart) to[out=0,in=180] (sp\i);
    \draw[flow] (sp\i) -- (gt\i.west);
    \fill[dstHeadB] (sp\i) circle (0.7pt);

    \draw[flow] (gt\i.east) -- node[above=-0.5pt, font=\scriptsize]{$g_{\h}^{(i)}$} (g\i.west);

    \draw[flow] (sp\i) to[out=-90,in=-90,looseness=1.0] (g\i.south);

    \draw[flow] (g\i.east)   -- (mA\i.west);
    \draw[flow] (mA\i.east)  -- (mB\i.west);
    \draw[flow] (mB\i.east)  -- (yh\i.west);
}

\coordinate (cmid) at ($(yh2.east)!0.5!(yh3.east)$);
\node[layer, draw=dstHeadB, line width=0.6pt, fill=white,
      minimum width=36mm, minimum height=10mm, anchor=west]
      (yout) at ($(cmid)+(16mm,0)$)
      {$\hat{\mathbf{y}}^{(i)}
      =(\hat{y}_1,\hat{y}_3,\hat{y}_5,\hat{y}_{10})$};

\foreach \i in {1,2,3,4} {
    \draw[flow] (yh\i.east) to[out=0,in=180] (yout.west);
}

\begin{pgfonlayer}{background}
  \node[lane, draw=dstLine, fill=dstPre,
        fit=(raw)(burn), inner sep=2.2mm] (lanePre) {};
  \node[anchor=north west, font=\scriptsize\bfseries, text=dstLine]
        at ([xshift=2mm,yshift=4mm]lanePre.north west)
        {Pre-processing (outside the trainable network)};

  \node[lane, draw=dstTempB, fill=dstTemp!35,
        fit=(tConv)(f1), inner sep=2.2mm] (laneT) {};
  \node[anchor=north west, font=\scriptsize\bfseries, text=dstTempB]
        at ([xshift=18mm,yshift=3.5mm]laneT.north west)
        {Temporal Branch (CNN-BiLSTM)};

  \node[lane, draw=dstSpecB, fill=dstSpec!35,
        fit=(aRes)(bRes)(hts)(hss)(cat)(mlp1)(mlp5)(f2),
        inner sep=2.2mm] (laneS) {};
  \node[anchor=north west, font=\scriptsize\bfseries, text=dstSpecB]
        at ([xshift=2mm,yshift=3.5mm]laneS.north west)
        {Spectral Branch \,--\, Factorised STST};

  \node[lane, draw=dstFuseB, fill=dstFuse!40,
        fit=(vaf)(hfused), inner sep=2.2mm] (laneF) {};

  \node[draw=dstHeadB, dotted, line width=0.9pt, rounded corners=14pt,
        fill=none,
        fit=(hadStart)(gt1)(gt4)(yh1)(yh4)(yout), inner sep=5mm] (laneH) {};
  \node[anchor=west, fill=white, inner sep=1.5mm,
        font=\scriptsize\bfseries, text=dstHeadB]
        at ([xshift=10mm]laneH.north west)
        {Horizon-Adaptive Decoder (HAD): independent feature gate per horizon};
\end{pgfonlayer}

\node[draw=dstLine, line width=0.3pt, rounded corners=2pt, fill=white,
      anchor=north west, inner sep=1.8mm, font=\scriptsize]
      at ($(laneF.south west)+(-5,2.5)$)
{
\begin{tabular}{@{}ll@{}}
$\odot$     & elementwise multiply \\
$+$         & elementwise addition \\
$\Vert$     & concatenation \\
\textit{st} & stack as 2 tokens \\
$\sigma$    & softmax \\
$\tanh$     & gate activation \\
res$+$LN    & residual + layer norm \\
MHA         & multi-head attention \\
\end{tabular}};

\end{tikzpicture}%
}%
\caption{Full DSTNet pipeline. Pre-processing (dashed lane) is provably causal (Proposition~\ref{thm:causality}) and excludes a $B{=}192$-sample burn-in. The temporal branch (CNN-BiLSTM) and the factorised Scale-Temporal Spectral Transformer (STST: Stream~A reshapes $\mathcal{D}^{(i)}$ into $w$ time-tokens of dimension $SM{=}56$, while Stream~B permutes into $S$ filter-bank tokens of dimension $wM{=}140$) produce 64-dim summaries $f_1$ and $f_2$. The Vertical Attention Fusion gate combines them into $h_{\text{fus}}$. The Horizon-Adaptive Decoder applies an independent $\tanh$ feature gate $g_h$ per horizon, yielding simultaneous forecasts at $h\in\{1,3,5,10\}$ trading days in a single forward pass.}
\label{fig:dstnet-arch}
\end{figure*}



\subsection{Data and Preprocessing}\label{ssec:data}
Let $\mathbf{X} \in \mathbb{R}^{T \times 5}$ denote a raw OHLCV series with columns $(\text{Open},\text{High},\text{Low},\text{Close},\text{Volume})$. Microstructure noise corrupts both the temporal-branch input and the indicators that feed the spectral branch, so we denoise each channel independently with a causal undecimated Symlet-4 residual multiresolution denoiser of depth $L = 2$. The level-$\ell$ approximation filter $g_\ell$ is the Symlet-4 scaling filter with $2^{\ell-1} - 1$ zeros inserted between adjacent taps, the standard undecimated upsampling, applied as a causal FIR convolution with left padding only.
Writing $\mathbf{a}_0 = \mathbf{X}[:, c]$ for a raw channel, the cascade computes $\mathbf{a}_\ell = \operatorname{CausalFIR}(g_\ell, \mathbf{a}_{\ell-1})$ and defines the level-$\ell$ detail band as the residual $\mathbf{d}_\ell = \mathbf{a}_{\ell-1} - \mathbf{a}_\ell$. The residual construction telescopes, so the pre-threshold identity
\begin{equation*}
\mathbf{a}_0
\;=\;
\mathbf{a}_L + \sum_{\ell=1}^{L} \mathbf{d}_\ell
\end{equation*}
holds exactly at every index. Denoising soft-thresholds the residual bands with the universal threshold of \cite{donoho1994ideal}, estimated on the training split. For the training-split restriction $\mathbf{d}$ of the level-$\ell$ detail band, with $T_d$ its length, the threshold is
\begin{equation}
\tau_{\ell}
\;=\;
\frac{\operatorname{median}(|\mathbf{d} - \operatorname{median}(\mathbf{d})|)}{0.6745}
\sqrt{2 \log T_d}.
\label{eq:mad-threshold}
\end{equation}
Writing $S_{\tau}(d) = \operatorname{sign}(d) \cdot \max(|d| - \tau, 0)$ for the soft-threshold operator applied elementwise, the denoised channel is the post-threshold reconstruction
\begin{equation*}
\widetilde{\mathbf{X}}[:, c]
\;=\;
\mathbf{a}_L + \sum_{\ell=1}^{L} S_{\tau_\ell}(\mathbf{d}_\ell),
\end{equation*}
which is not an identity, since thresholding removes energy from the detail bands by design. Every stage is causal. The filters $g_\ell$ read only current and past samples, the residual subtraction is pointwise, and the reconstruction is a sum of causally computed bands, so the denoised value at time $t$ never depends on a sample from $t+1$ onward. The noise scale and the threshold $\tau_{\ell}$ of \eqref{eq:mad-threshold} are estimated within each walk-forward training fold and refitted on the complete pre-hold-out training history before the final test transform, then held fixed, which closes the second leakage path that a per-series threshold would open.

The principal preprocessing constants are fixed rather than tuned per dataset. Symlet-4 is chosen for its near-symmetry and four vanishing moments, while its short support keeps the causal FIR realisation compact. The carrier parameter $\omega_{0}=6$ follows the conventional Morlet prototype. The set $\mathcal{S}$ comprises eight approximately log-spaced FIR scale parameters that control carrier, envelope width, and support length. Section~\ref{ssec:cwwt} explains why they are treated as fixed filter-bank indices rather than as physical periods. This sparse bank cuts the filter-axis token count by a factor of four against a dense thirty-two-parameter bank, while the largest parameter of 32 gives the burn-in $B=\lfloor6\cdot32\rfloor=192$ samples. The lookback $w=20$ is one trading month, long enough to contain a regime transition and short enough to keep the longest-horizon test set from collapsing to a handful of windows. None of these constants is re-tuned across datasets. After denoising we compute seven momentum and volume indicators on the denoised series. The set is RSI-10, Stochastic \%K, CCI-14, OBV, ATR-14, Williams \%R, and ROC-12. We collect them column-wise into the indicator matrix $\mathbf{Z} \in \mathbb{R}^{T \times M}$ with $M = 7$, then Z-score each column using statistics computed only on the training split to avoid look-ahead leakage. Algorithm~\ref{alg:preproc} states the full procedure.

\subsection{Causal Morlet-Derived Complex FIR Filter Bank}\label{ssec:cwwt}
Standard continuous wavelet transform (CWT) implementations in the deep-learning literature compute \emph{bilateral} convolutions over the entire signal. This breaks causality for forecasting, because the coefficient at time $t$ then depends on samples $x[\tau]$ with $\tau > t$. We replace the bilateral CWT with a strictly causal Morlet filter bank realised as a finite impulse response. For a positive scale $s \in \mathbb{N}$ and Morlet central frequency $\omega_0 = 6$ we define the discrete causal kernel
\begin{equation}
h_s[k]
\;=\;
\frac{\pi^{-1/4}}{\sqrt{s}}\,
\exp\!\left(\frac{i\,\omega_0\,k}{s}\right)
\exp\!\left(-\frac{k^2}{2 s^2}\right),
\label{eq:morlet-fir}
\end{equation}
where $k \in \{0, 1, \ldots, K_s\}$ with effective support $K_s = \lfloor 6 s \rfloor$. The causal complex filter-bank coefficient at time $t$, scale $s$, and indicator channel $m$ is
\begin{equation}
C[t,s,m]
\;=\;
\sum_{k=0}^{K_s} h_s[k]\,Z[t-k,m],
\label{eq:cwwt}
\end{equation}
with the convention $Z[t,m]=0$ for $t<0$.
The implementation uses \texttt{scipy.signal.lfilter}, so the sum contains only non-negative lags. Consequently, $C[t,s,m]$ depends only on observations available at or before time $t$. Proposition~\ref{thm:causality} formalises this property. Restricting the Morlet-derived kernel to non-negative lags breaks the zero-mean admissibility condition of the conventional Morlet wavelet. We therefore describe the implemented operator as a causal Morlet-derived complex FIR filter bank, rather than as a continuous wavelet transform, and use it only as a time-frequency feature extractor. We make no invertibility or exact CWT-equivalence claim. There is also a discrete-frequency qualification. Carrier frequencies are defined modulo $2\pi$, and at $s=1$ the nominal carrier $\omega_0/s=6$ radians/sample exceeds the Nyquist limit $\pi$ and folds to $6-2\pi\approx-0.283$ radians/sample. The one-sided truncation further changes the conventional Morlet response. Consequently, the eight channels are not assigned a clean daily-to-monthly pseudo-frequency ladder, and the model's empirical results are interpreted for this fixed implemented filter bank rather than for conventional Morlet scale bands. We retain the coefficient magnitude as the real-valued spectral feature:
\begin{equation}
W[t,s,m]
\;=\;
\bigl|C[t,s,m]\bigr|
\;\in\;
\mathbb{R}_{\geq 0}.
\label{eq:filter-magnitude}
\end{equation}
The coefficient phase may contain additional information, but we do not evaluate it as a separate design factor. The resulting spectral-magnitude tensor is
$\mathbf{W}\in\mathbb{R}_{\geq 0}^{T\times S\times M}$.

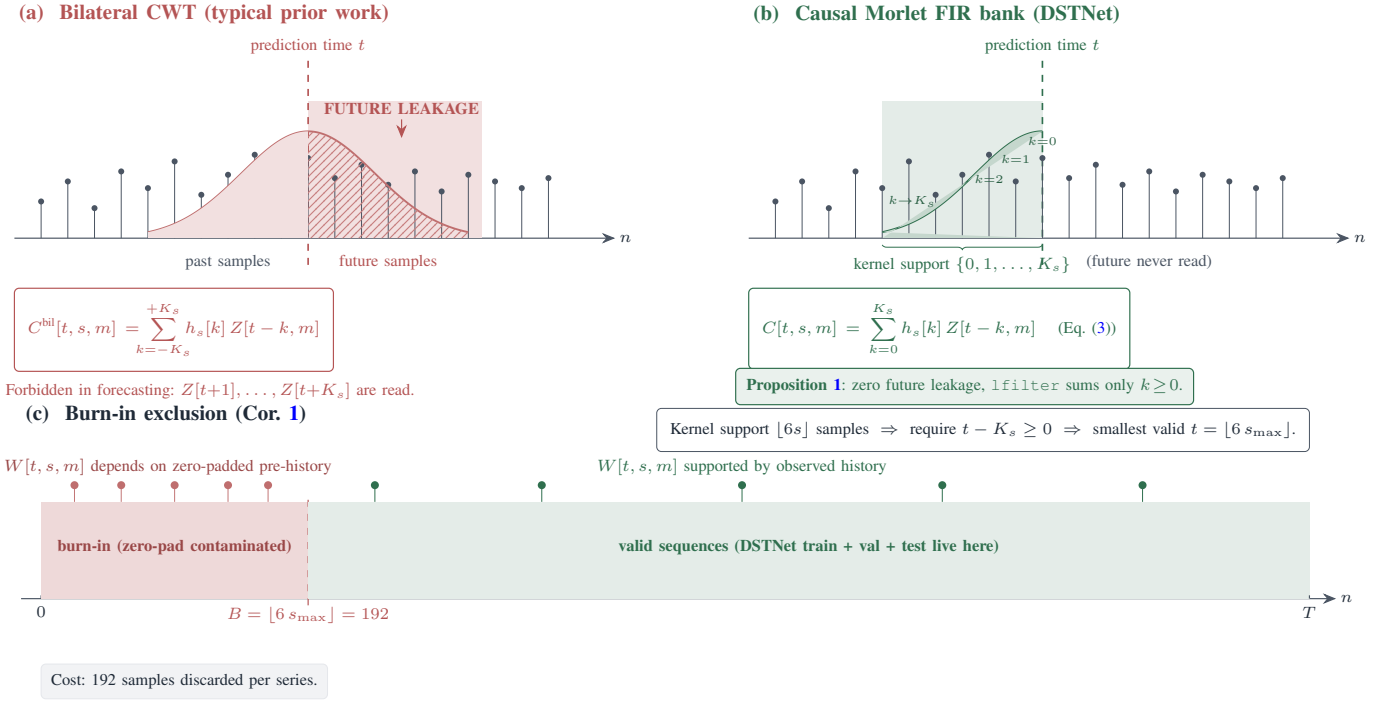
\begin{figure*}[!t]
\centering
\resizebox{\textwidth}{!}{%
\begin{tikzpicture}[
    >={Stealth[length=2mm,width=1.6mm]},
    font=\footnotesize,
    every node/.style={inner sep=2pt, outer sep=0pt, text=dstInk},
]

\begin{scope}[shift={(0,0)}]
  \draw[dstLine, line width=0.5pt, ->] (-0.4,0) -- (8.6,0)
        node[right, font=\scriptsize] {$n$};
  \fill[dstAccent!18] (4.0,0) rectangle (6.6,2.05);
  \draw[dstAccent, line width=0.7pt, dashed] (4.0,-0.4) -- (4.0,2.7);
  \node[font=\scriptsize, text=dstAccent, anchor=south]
       at (4.0,2.7) {prediction time $t$};
  \foreach \n/\h in {%
    0.0/0.55, 0.4/0.85, 0.8/0.45, 1.2/1.00, 1.6/0.75, 2.0/1.15,
    2.4/0.65, 2.8/0.95, 3.2/1.25, 3.6/0.85, 4.0/1.20, 4.4/0.90,
    4.8/1.10, 5.2/0.80, 5.6/1.00, 6.0/0.70, 6.4/0.95, 6.8/0.85,
    7.2/0.75, 7.6/0.90} {
    \draw[dstLine, line width=0.4pt] (\n,0) -- (\n,\h);
    \fill[dstLine] (\n,\h) circle (0.045);
  }
  \draw[dstAccent!85, line width=0.8pt, samples=140, domain=-2.4:2.4, smooth]
        plot (\x+4.0, {1.6*exp(-0.5*\x*\x)});
  \begin{scope}
    \clip plot[domain=0:2.4, samples=80] (\x+4.0,{1.6*exp(-0.5*\x*\x)})
          -- (6.4,0) -- (4.0,0) -- cycle;
    \fill[pattern=north east lines, pattern color=dstAccent!75]
          (4.0,0) rectangle (6.5,2.0);
  \end{scope}
  \begin{scope}
    \clip plot[domain=-2.4:0, samples=80] (\x+4.0,{1.6*exp(-0.5*\x*\x)})
          -- (4.0,0) -- (1.6,0) -- cycle;
    \fill[dstAccent!18] (1.6,0) rectangle (4.0,2.0);
  \end{scope}
  \node[font=\scriptsize\bfseries, text=dstAccent, anchor=center]
       at (5.4, 1.95) {FUTURE LEAKAGE};
  \draw[->, dstAccent, line width=0.5pt] (5.4,1.78) -- (5.4,1.5);
  \node[font=\scriptsize, text=dstLine, anchor=center]
       at (2.8,-0.36) {past samples};
  \node[font=\scriptsize, text=dstAccent, anchor=center]
       at (5.2,-0.36) {future samples};
  \node[anchor=south west, font=\small\bfseries, text=dstAccent]
       at (-0.4,3.15) {(a)\ \ Bilateral CWT (typical prior work)};
  \node[draw=dstAccent, line width=0.5pt, rounded corners=2pt,
        fill=white, inner sep=2mm, font=\scriptsize, text=dstAccent,
        anchor=north west] at (-0.4,-0.75)
       {$C^{\text{bil}}[t,s,m]
         \,=\!\!\!\displaystyle\sum_{k=-K_s}^{+K_s}\!\!\! h_s[k]\,Z[t-k,m]$};
  \node[anchor=north west, font=\scriptsize, text=dstAccent]
       at (-0.6,-2.10)
       {Forbidden in forecasting: $Z[t{+}1],\ldots,Z[t{+}K_s]$ are read.};
\end{scope}

\begin{scope}[shift={(11.0,0)}]
  \draw[dstLine, line width=0.5pt, ->] (-0.4,0) -- (8.6,0)
        node[right, font=\scriptsize] {$n$};
  \fill[dstSpeedup!13] (1.6,0) rectangle (4.0,2.05);
  \draw[dstSpeedup, line width=0.7pt, dashed] (4.0,-0.4) -- (4.0,2.7);
  \node[font=\scriptsize, text=dstSpeedup, anchor=south]
       at (4.0,2.7) {prediction time $t$};
  \foreach \n/\h in {%
    0.0/0.55, 0.4/0.85, 0.8/0.45, 1.2/1.00, 1.6/0.75, 2.0/1.15,
    2.4/0.65, 2.8/0.95, 3.2/1.25, 3.6/0.85, 4.0/1.20, 4.4/0.90,
    4.8/1.10, 5.2/0.80, 5.6/1.00, 6.0/0.70, 6.4/0.95, 6.8/0.85,
    7.2/0.75, 7.6/0.90} {
    \draw[dstLine, line width=0.4pt] (\n,0) -- (\n,\h);
    \fill[dstLine] (\n,\h) circle (0.045);
  }
  \draw[dstSpeedup, line width=0.85pt, samples=80, domain=0:2.4, smooth]
        plot ({4.0-\x}, {1.6*exp(-0.5*\x*\x)});
  \begin{scope}
    \clip plot[domain=0:2.4, samples=80] ({4.0-\x},{1.6*exp(-0.5*\x*\x)})
          -- (4.0,0) -- (1.6,0) -- cycle;
    \fill[dstSpeedup!28] (1.6,0) rectangle (4.0,2.0);
  \end{scope}
  \node[font=\tiny, text=dstSpeedup!85!black, anchor=south] at (4.00,1.32) {$k{=}0$};
  \node[font=\tiny, text=dstSpeedup!85!black, anchor=south] at (3.60,1.05) {$k{=}1$};
  \node[font=\tiny, text=dstSpeedup!85!black, anchor=south] at (3.20,0.75) {$k{=}2$};
  \node[font=\tiny, text=dstSpeedup!85!black, anchor=south] at (2.05,0.40) {$k{\to}K_s$};
  \draw[dstSpeedup, decorate, decoration={brace, amplitude=2pt, mirror}]
        (1.6,-0.10) -- (4.0,-0.10);
  \node[font=\scriptsize, text=dstSpeedup, anchor=north]
       at (2.8,-0.20) {kernel support $\{0,1,\ldots,K_s\}$};
  \node[font=\scriptsize, text=dstLine, anchor=center]
       at (5.6,-0.36) {(future never read)};
  \node[anchor=south west, font=\small\bfseries, text=dstSpeedup]
       at (-0.4,3.15) {(b)\ \ Causal Morlet FIR bank (DSTNet)};
  \node[draw=dstSpeedup, line width=0.5pt, rounded corners=2pt,
        fill=white, inner sep=2mm, font=\scriptsize, text=dstSpeedup,
        anchor=north west] at (-0.4,-0.75)
        {$C[t,s,m]
         \,=\,\displaystyle\sum_{k=0}^{K_s} h_s[k]\, Z[t-k,m]$
        \quad (Eq.~\eqref{eq:cwwt})};
  \node[draw=dstSpeedup, line width=0.5pt, rounded corners=2pt,
        fill=dstSpeedup!10, inner sep=1.6mm, font=\scriptsize, text=dstSpeedup,
        anchor=north west] at (-0.595,-1.95)
       {\textbf{Proposition~\ref{thm:causality}}: zero future leakage,
        \texttt{lfilter} sums only $k\!\geq\!0$.};
\end{scope}

\begin{scope}[shift={(0,-5.4)}]
  \draw[dstLine, line width=0.5pt, ->] (-0.3,0) -- (19.4,0)
        node[right, font=\scriptsize] {$n$};
  \draw[dstLine, line width=0.5pt] (0,0) -- (0,0.05);
  \node[font=\scriptsize, anchor=north] at (0,-0.05) {$0$};
  \fill[dstAccent!22] (0,0) rectangle (4.0,1.45);
  \draw[dstAccent, line width=0.5pt, dashed] (4.0,-0.10) -- (4.0,1.55);
  \node[font=\scriptsize, text=dstAccent, anchor=north]
       at (4.0,-0.05) {$B = \lfloor 6\,s_{\max}\rfloor = 192$};
  \fill[dstSpeedup!13] (4.0,0) rectangle (19.0,1.45);
  \draw[dstLine, line width=0.5pt] (19.0,0) -- (19.0,0.05);
  \node[font=\scriptsize, anchor=north] at (19.0,-0.05) {$T$};
  \node[font=\scriptsize\bfseries, text=dstAccent!85!black, anchor=center]
       at (2.0,0.78) {burn-in (zero-pad contaminated)};
  \node[font=\scriptsize\bfseries, text=dstSpeedup, anchor=center]
       at (11.5,0.78)
       {valid sequences (DSTNet train + val + test live here)};
  \foreach \x in {0.5,1.2,2.0,2.8,3.4} {
    \draw[dstAccent!80, line width=0.4pt] (\x,1.45) -- (\x,1.65);
    \fill[dstAccent!85] (\x,1.70) circle (0.06);
  }
  \node[font=\scriptsize, text=dstAccent!85!black, anchor=south]
       at (1.9,1.78) {$W[t,s,m]$ depends on zero-padded pre-history};
  \foreach \x in {5.0,7.5,10.5,13.5,16.5} {
    \draw[dstSpeedup, line width=0.4pt] (\x,1.45) -- (\x,1.65);
    \fill[dstSpeedup] (\x,1.70) circle (0.06);
  }
  \node[font=\scriptsize, text=dstSpeedup, anchor=south]
       at (10.5,1.78) {$W[t,s,m]$ supported by observed history};
  \node[anchor=south west, font=\small\bfseries, text=dstInk]
       at (-0.3, 2.55) {(c)\ \ Burn-in exclusion (Cor.~\ref{lem:burnin})};
  \node[draw=dstLine, line width=0.4pt, rounded corners=2pt,
        fill=white, inner sep=2mm, font=\scriptsize, anchor=north east]
        at (19.0, 2.85)
        {Kernel support $\lfloor 6 s\rfloor$ samples
         $\;\Rightarrow\;$ require $t-K_s \geq 0$
         $\;\Rightarrow\;$ smallest valid $t = \lfloor 6\, s_{\max}\rfloor$.};
  \node[draw=dstSoft, line width=0.4pt, rounded corners=2pt,
        fill=dstSoft!50, inner sep=1.5mm, font=\scriptsize,
        anchor=south west, text=dstInk] at (0,-1.5)
        {Cost: 192 samples discarded per series.};
\end{scope}

\end{tikzpicture}%
}%
\caption{Causal Morlet-derived complex FIR filtering versus bilateral continuous wavelet analysis, illustrating the issues addressed by Proposition~\ref{thm:causality} and Corollary~\ref{lem:burnin}.
\textbf{(a)}~A bilateral coefficient at time $t$ uses observations on both sides of $t$, so its future-side terms leak information past the forecast origin.
\textbf{(b)}~DSTNet's causal Morlet-derived FIR bank sums only over non-negative lags $k\in\{0,1,\ldots,K_s\}$, as defined in \eqref{eq:cwwt}. Proposition~\ref{thm:causality} establishes that no observation after time $t$ contributes to the complex coefficient $C[t,s,m]$ or, consequently, to its magnitude $W[t,s,m]=|C[t,s,m]|$.
\textbf{(c)}~For $t<K_s$, some FIR taps depend on the assumed zero prehistory. DSTNet excludes the first $B=\lfloor6s_{\max}\rfloor=192$ indices so that every retained coefficient is supported entirely by observed historical samples. This burn-in removes the left-boundary initialization transient. It does not assert equivalence between the causal filter and a bilateral CWT.}
\label{fig:causal-cwt}
\end{figure*}

\subsection{Dynamic Spectral Trajectory (DST) Tensor}\label{ssec:dst}
Fig.~\ref{fig:causal-cwt} contrasts this construction with bilateral wavelet analysis and marks the burn-in region that Corollary~\ref{lem:burnin} derives. The spectral-magnitude tensor $\mathbf{W}$ is a function of three indices. Several reviewed wavelet-based hybrids consume it as a static snapshot, retaining only the slice $\mathbf{C}_{\text{static}}^{(i)} = \mathbf{W}[t_i, :, :] \in \mathbb{R}^{S \times M}$ at the prediction time $t_i$. This collapses the temporal evolution of filter-bank magnitude into a single time slice and discards the path information that distinguishes a freshly entered regime from a long-established one. DSTNet retains the full evolution. For the $i$-th sequence we define the Dynamic Spectral Trajectory tensor
\begin{equation}
\begin{split}
\mathbf{D}^{(i)}
\;=\;
\bigl[\,&\mathbf{W}[t_i - w + 1, :, :],\,
       \mathbf{W}[t_i - w + 2, :, :],\\
       &\ldots,\,
       \mathbf{W}[t_i, :, :]\,\bigr]
\;\in\;
\mathbb{R}^{w \times S \times M},
\end{split}
\label{eq:dst}
\end{equation}
and stack the $N$ sequences into a four-dimensional batch $\mathbf{D} \in \mathbb{R}^{N \times w \times S \times M}$ with shape $(N, 20, 8, 7)$. The static snapshot is the deterministic final-slice extraction $\mathbf{D}^{(i)}[w, :, :]$, so by the data processing inequality~\cite{cover2006} it cannot carry more information about the target than the full tensor, with equality exactly when the target is conditionally independent of the spectral history given the current spectral state. That ordering is weak by construction and implies nothing about what a finite encoder trained on finite data will extract, so we treat it as motivation and leave the empirical question to the No-DST ablation of Section~\ref{ssec:ablations}.
The loop index $i$ in Algorithm~\ref{alg:dst-builder} relates to the in-window origin by $i = t_i + 1$, so the window $\mathbf{W}[i-w:i]$ of the algorithm equals the slice $\mathbf{W}[t_i-w+1 : t_i]$ of \eqref{eq:dst}.

Two construction details matter. First, we enforce the burn-in bound $t_i - w + 1 \geq B = 192$ in every sequence builder, so every entry of every DST tensor depends only on samples that are themselves uncontaminated by left-boundary zero padding. Second, the DST tensor is constructed on \emph{normalised} indicators, meaning the channel-wise filter-bank magnitudes are directly comparable across $M$. Algorithm~\ref{alg:dst-builder} states the procedure with explicit loop bounds.

\subsection{Temporal Branch (CNN-BiLSTM)}\label{ssec:temporal-branch}
The temporal branch consumes the denoised OHLCV window $\mathbf{X}_{\text{temp}}^{(i)} \in \mathbb{R}^{w \times 5}$ and produces a fixed-dimensional summary $\mathbf{f}^{(i)}_1 \in \mathbb{R}^{d_{\text{fus}}}$. The architecture is conventional. A one-dimensional convolution with 64 filters, kernel size 3, ReLU activation, and \texttt{same} padding extracts local price patterns. Two stacked bidirectional LSTM layers with hidden sizes 32 and 64 then propagate information forward and backward across the window. The final BiLSTM returns a single context vector, which dropout (rate 0.3) regularises before a dense projection to $d_{\text{fus}} = 64$ with ReLU activation:
\begin{equation}
\mathbf{f}^{(i)}_1
=
\operatorname{ReLU}\!\bigl(\mathbf{W}_{\text{fv1}}\,\bm{\phi}_{\text{BiLSTM}}\bigl(\bm{\phi}_{\text{Conv}}(\mathbf{X}_{\text{temp}}^{(i)})\bigr) + \mathbf{b}_{\text{fv1}}\bigr).
\label{eq:fv1}
\end{equation}
We use no pooling between the convolution and the recurrence, so the BiLSTM stack receives all $w = 20$ daily positions and performs the temporal aggregation itself. Bidirectionality is confined to the observed lookback window. The backward pass traverses that window in reverse, reads no sample dated after the forecast origin, and therefore introduces no look-ahead.

\subsection{Spectral Branch: Factorized STST}\label{ssec:stst}
Fig.~\ref{fig:dstnet-stst-mlp} shows the two-stream layout described below. The spectral branch is the core technical contribution. It consumes $\mathbf{D}^{(i)} \in \mathbb{R}^{w \times S \times M}$ and produces a fused spectral summary $\mathbf{f}^{(i)}_2 \in \mathbb{R}^{d_{\text{fus}}}$. A naive joint self-attention over the flattened DST tensor would treat each $(t, s)$ pair as an independent token, giving $w S = 160$ tokens of feature dimension $M$ and per-layer cost $\mathcal{O}((wS)^2 d) = \mathcal{O}(160^2 d)$. This is computationally unattractive at our default sizes and, more importantly, it gives the attention mechanism no inductive bias toward the two distinct axes of variation along which financial spectra actually evolve. We instead apply two factorized attention streams in parallel.

\subsubsection{Stream A (temporal-spectral)}
We reshape the DST into $w$ tokens of dimension $SM = 56$, so each token represents the full spectral state at one time step:
\begin{equation}
\mathbf{T}^{(i)}_{\text{A}}
\;=\;
\operatorname{Reshape}\bigl(\mathbf{D}^{(i)},\,(w,\, S M)\bigr)
\;\in\;
\mathbb{R}^{20 \times 56}.
\label{eq:streamA-reshape}
\end{equation}
Two stacked multi-head self-attention blocks (4 heads, key dimension 32, attention dropout 0.1) process $\mathbf{T}^{(i)}_{\text{A}}$ with residual connections and pre-layer normalisation. A learned single-token attention pool collapses the result to a vector $\mathbf{h}^{(i)}_{\text{ts}} \in \mathbb{R}^{56}$.

\begin{figure*}[!t]
\centering
\resizebox{0.95\textwidth}{!}{%
\begin{tikzpicture}[
    >={Stealth[length=2mm,width=1.6mm]},
    every node/.style={inner sep=2pt, outer sep=0pt, text=dstInk},
    layer/.style    ={rectangle, rounded corners=2pt, draw=dstLine, line width=0.45pt,
                      fill=white, minimum height=8mm, minimum width=18mm,
                      align=center},
    slayer/.style   ={layer, fill=dstSpec,  draw=dstSpecB},
    salayer/.style  ={slayer, fill=dstStreamA},
    sblayer/.style  ={slayer, fill=dstStreamB},
    inp/.style      ={layer, fill=dstSoft,  draw=dstLine, minimum width=22mm,
                      minimum height=10mm},
    op/.style       ={circle, draw=dstLine, fill=white, line width=0.4pt,
                      minimum size=4.6mm, inner sep=0pt},
    flow/.style     ={->, draw=dstLine, line width=0.45pt},
    mystyle/.style  ={font=\itshape, text=dstLine},
    lane/.style     ={rounded corners=3pt, line width=0.3pt},
]

\node[inp] (inDST) at (0,1.1)
      {DST tensor\\$\mathbf{D}^{(i)}\!\in\!\mathbb{R}^{20\times 8\times 7}$};

\node[salayer] (aRes)  at (3.1,2.0) {reshape\\$(20,56)$};
\node[salayer, right=3mm of aRes]  (aMHA1) {MHA\\\scriptsize 4h,\,$d_k{=}32$};
\node[salayer, right=3mm of aMHA1] (aMHA2) {MHA\\\scriptsize 4h,\,$d_k{=}32$};
\node[salayer, right=3mm of aMHA2] (aPool) {Attn.\\pool};
\node[layer,  right=3mm of aPool, fill=white,
              draw=dstSpecB, line width=0.55pt] (hts)
      {$h_{ts}^{(i)}\!\in\!\mathbb{R}^{56}$};

\node[sblayer] (bRes)  at (3.1,0.2) {permute$+$reshape\\$(8,140)$};
\node[sblayer, right=3mm of bRes]  (bMHA1) {MHA\\\scriptsize 4h,\,$d_k{=}16$};
\node[sblayer, right=3mm of bMHA1] (bMHA2) {MHA\\\scriptsize 4h,\,$d_k{=}16$};
\node[sblayer, right=3mm of bMHA2] (bPool) {Attn.\\pool};
\node[layer,  right=3mm of bPool, fill=white,
              draw=dstSpecB, line width=0.55pt] (hss)
      {$h_{ss}^{(i)}\!\in\!\mathbb{R}^{140}$};

\draw[flow] (inDST.east) -- ++(3mm,0) |- (aRes.west);
\draw[flow] (inDST.east) -- ++(3mm,0) |- (bRes.west);

\draw[flow] (aRes)  -- (aMHA1);
\draw[flow] (aMHA1) -- (aMHA2);
\draw[flow] (aMHA2) -- (aPool);
\draw[flow] (aPool) -- (hts);

\draw[flow] (bRes)  -- (bMHA1);
\draw[flow] (bMHA1) -- (bMHA2);
\draw[flow] (bMHA2) -- (bPool);
\draw[flow] (bPool) -- (hss);

\node[mystyle, above=0.2mm of aMHA1.north east, xshift=-9mm] {res$+$LN};
\node[mystyle, above=0.2mm of aMHA2.north east, xshift=-9mm] {res$+$LN};
\node[mystyle, below=0.2mm of bMHA1.south west, xshift=9mm]  {res$+$LN};
\node[mystyle, below=0.2mm of bMHA2.south west, xshift=9mm]  {res$+$LN};

\node[mystyle, text=dstLine!70, below=1.5mm of aMHA1.south, xshift=-1mm]
      {Attn.\ dropout\,0.1};
\node[mystyle, text=dstLine!70, above=1.5mm of bMHA2.north, xshift=1mm]
      {Attn.\ dropout\,0.1};

\node[font=\itshape, text=dstSpecB, left=7mm of aRes] {Stream A};
\node[font=\itshape, text=dstSpecB, left=3mm of bRes] {Stream B};

\node[op, double, double distance=0.4pt, double=dstSpecB!50,
      font=\large\bfseries] (cat)
      at ($(hts.east)!0.5!(hss.east) + (0.9,0)$) {$\Vert$};
\draw[flow] (hts.east) -| (cat.north);
\draw[flow] (hss.east) -| (cat.south);

\node[slayer, right=4mm of cat, minimum width=22mm] (mlp1)
      {Dense\,256\\ReLU};
\node[slayer, right=3mm of mlp1, minimum width=14mm] (mlp2) {Drop\,$0.3$};
\node[slayer, right=3mm of mlp2, minimum width=22mm] (mlp3)
      {Dense\,128\\ReLU};
\node[slayer, right=3mm of mlp3, minimum width=14mm] (mlp4) {Drop\,$0.3$};
\node[slayer, right=3mm of mlp4, minimum width=18mm] (mlp5) {Dense\,64};
\node[layer, right=4mm of mlp5, fill=white,
      draw=dstSpecB, line width=0.6pt] (f2)
      {$f_2^{(i)}\!\in\!\mathbb{R}^{64}$};

\draw[flow] (cat.east) -- (mlp1.west);
\draw[flow] (mlp1) -- (mlp2);
\draw[flow] (mlp2) -- (mlp3);
\draw[flow] (mlp3) -- (mlp4);
\draw[flow] (mlp4) -- (mlp5);
\draw[flow] (mlp5) -- (f2);

\node[font=\itshape, text=dstSpecB,
      above=0.3mm of mlp3] {cross-stream MLP\,(Eqs.~\ref{eq:cross}--\ref{eq:fv2})};

\begin{pgfonlayer}{background}
  \node[lane, draw=dstSpecB, fill=dstSpec!70,
        fit=(inDST)(aRes)(bRes)(hts)(hss)(cat)(mlp1)(mlp5)(f2),
        inner sep=2.2mm] (laneS) {};
  \node[anchor=north west, font=\bfseries, text=dstSpecB]
        at ([xshift=2mm,yshift=5mm]laneS.north west)
        {Spectral Branch \,--\, factorized STST};
\end{pgfonlayer}

\end{tikzpicture}%
}%
\caption{Factorized STST module and cross-stream fusion.
Stream~A (blue tint) reshapes $\mathbf{D}^{(i)}$ into $w$ time-tokens of dimension $SM{=}56$, while Stream~B (purple tint) permutes into $S$ scale-tokens of dimension $wM{=}140$. Each stream uses two multi-head attention layers with residual connections, layer normalisation, and attention dropout~0.1. Concatenated embeddings are joined by a cross-stream MLP \eqref{eq:cross}--\eqref{eq:fv2} to produce the 64-dimensional spectral summary $f_2^{(i)}$. Theorem~\ref{thm:complexity} quantifies the attention-core cost of this factorization.}
\label{fig:dstnet-stst-mlp}
\end{figure*}

\subsubsection{Stream B (scale-spectral)}
We permute and reshape into $S$ tokens of dimension $w M = 140$, so each token represents the full temporal trajectory at one wavelet scale:
\begin{equation}
\mathbf{T}^{(i)}_{\text{B}}
=
\operatorname{Reshape}\!\bigl(\operatorname{Permute}_{(2,1,3)}(\mathbf{D}^{(i)}),\,(S,\, w M)\bigr)
\in
\mathbb{R}^{8 \times 140}.
\label{eq:streamB-reshape}
\end{equation}
Two stacked multi-head self-attention blocks (4 heads, key dimension 16, attention dropout 0.1) and a learned attention pool then yield $\mathbf{h}^{(i)}_{\text{ss}} \in \mathbb{R}^{140}$.

\subsubsection{Cross-stream integration}
The two stream outputs are concatenated and passed through a two-layer non-linear projection:
\begin{equation}
\mathbf{c}^{(i)}
=
\operatorname{Drop}\!\Bigl(
\operatorname{Dense}_{128}\!\bigl(
\operatorname{Drop}\!\bigl(
\operatorname{Dense}_{256}([\mathbf{h}^{(i)}_{\text{ts}},\, \mathbf{h}^{(i)}_{\text{ss}}])
\bigr)\bigr)
\Bigr),
\label{eq:cross}
\end{equation}
followed by a final dense projection to the fusion dimension,
\begin{equation}
\mathbf{f}^{(i)}_2
\;=\;
\operatorname{ReLU}\bigl(\mathbf{W}_{\text{fv2}}\,\mathbf{c}^{(i)} + \mathbf{b}_{\text{fv2}}\bigr)
\;\in\;
\mathbb{R}^{d_{\text{fus}}}.
\label{eq:fv2}
\end{equation}
Theorem~\ref{thm:complexity} quantifies the attention-core saving of this design. The saving has a price. Neither stream forms a $(t,s)$ pairwise score, so direct attention between the spectral state at scale 5 four days ago and the state at scale 21 today is unavailable inside either stream. Remark~\ref{rem:tradeoff} explains why the cross-stream MLP cannot fully recover it, and the joint-STST ablation of Section~\ref{ssec:ablations} measures the cost on our data.

\subsection{Vertical Attention Fusion (VAF)}\label{ssec:vaf}
The temporal summary $\mathbf{f}^{(i)}_1$ and the spectral summary $\mathbf{f}^{(i)}_2$ live in the same fusion space $\mathbb{R}^{d_{\text{fus}}}$ but encode complementary inductive biases. We combine them with a learned two-token attention gate. We stack the two vectors as a sequence of length two,
\begin{equation}
\mathbf{S}^{(i)}
\;=\;
\bigl[\mathbf{f}^{(i)}_1,\,\mathbf{f}^{(i)}_2\bigr] \in \mathbb{R}^{2 \times d_{\text{fus}}},
\label{eq:vaf-stack}
\end{equation}
score each row with a shared linear scorer $a^{(i)}_k = \mathbf{w}_{\text{vaf}}^\top \mathbf{S}^{(i)}_k + b_{\text{vaf}}$ for $k \in \{1, 2\}$, where the scalar $b_{\text{vaf}}$ enters both scores identically and therefore cancels under the two-token softmax, so the fusion weights depend on $\mathbf{w}_{\text{vaf}}$ alone. Normalising the two scores gives
\begin{equation}
(\alpha^{(i)}_{\text{temp}},\, \alpha^{(i)}_{\text{spec}})
\;=\;
\operatorname{Softmax}\!\bigl(a^{(i)}_1,\, a^{(i)}_2\bigr),
\label{eq:vaf-weights}
\end{equation}
and emit the fused vector
\begin{equation}
\mathbf{h}^{(i)}_{\text{fus}}
\;=\;
\alpha^{(i)}_{\text{temp}}\,\mathbf{f}^{(i)}_1 + \alpha^{(i)}_{\text{spec}}\,\mathbf{f}^{(i)}_2
\;\in\;
\mathbb{R}^{d_{\text{fus}}}.
\label{eq:vaf-fused}
\end{equation}
The VAF gate is shared across horizons. Whether a separate gate per horizon would be useful is an architectural question that connects directly to the decoder we describe next.

\subsection{Horizon-Adaptive Decoder (HAD)}\label{ssec:had}
DSTNet emits forecasts at $H = |\mathcal{H}|$ horizons in a single forward pass. The decoder is given the capacity to emphasise different coordinates of $\mathbf{h}^{(i)}_{\text{fus}}$ at different horizons, without prescribing which coordinates any horizon should use. We implement that capacity through a per-horizon gating mechanism. For each $h \in \mathcal{H}$ we learn an independent gating layer with $\tanh$ activation,
\begin{equation}
\mathbf{g}^{(i)}_h
\;=\;
\tanh\!\bigl(\mathbf{W}_{\text{g},h}\,\mathbf{h}^{(i)}_{\text{fus}} + \mathbf{b}_{\text{g},h}\bigr)
\;\in\;
[-1, 1]^{d_{\text{fus}}},
\label{eq:had-gate}
\end{equation}
and use it to gate the fused vector elementwise,
\begin{equation}
\mathbf{s}^{(i)}_h
\;=\;
\mathbf{h}^{(i)}_{\text{fus}} \odot \mathbf{g}^{(i)}_h
\;\in\;
\mathbb{R}^{d_{\text{fus}}}.
\label{eq:had-select}
\end{equation}
A small per-horizon head, comprising a 32-unit ReLU layer with dropout (rate 0.15) followed by a scalar output, then maps $\mathbf{s}^{(i)}_h$ to the prediction
\begin{equation}
\hat{y}^{(i)}_h
=
\mathbf{w}_{\text{out},h}^\top\,\operatorname{ReLU}(\mathbf{W}_{\text{p},h}\,\mathbf{s}^{(i)}_h + \mathbf{b}_{\text{p},h}) + b_{\text{out},h}.
\label{eq:had-out}
\end{equation}
Stacking across horizons produces the multi-horizon output $\hat{\mathbf{y}}^{(i)} \in \mathbb{R}^H$. Each $\hat{y}^{(i)}_h$ is the predicted cumulative log return $\hat{r}_{t_i,h}$ of \eqref{eq:target-return}, and price-level metrics use the reconstruction of \eqref{eq:price-reconstruction}. All four heads are trained jointly under a single Huber loss (with $\delta = 0.1$) averaged across horizons.

\subsubsection{Scope of the decoder}
The gate acts on the post-VAF vector $\mathbf{h}^{(i)}_{\text{fus}}$, not on $\mathbf{f}^{(i)}_1$ and $\mathbf{f}^{(i)}_2$ separately, so HAD performs horizon-conditioned elementwise modulation inside the fused feature space. Its diagonal $\tanh$ map is not a projection, because gate entries need not be zero, and it cannot re-weight the temporal against the spectral branch per horizon, because the branch identities are already mixed. Whether horizon-specific branch reliance exists is therefore not identifiable here. Section~\ref{ssec:ablations} reports a descriptive gate-variation analysis and Section~\ref{ssec:limitations} states what it does not establish.

\section{Theoretical Foundations}
\label{sec:theory}

\begin{proposition}[Causality of the Morlet-derived complex FIR filter bank]
\label{thm:causality}
Let $h_s[k]$ be the finite impulse response defined in
\eqref{eq:morlet-fir}, with
$K_s=\lfloor6s\rfloor$, and let $Z[t,m]=0$ for $t<0$.
The complex coefficient
\begin{align}
C[t,s,m]
&=
\sum_{k=0}^{K_s} h_s[k]\,Z[t-k,m] \nonumber\\
&=
\sum_{\tau=t-K_s}^{t}h_s[t-\tau]\,Z[\tau,m]
\label{eq:thm-causal-eq}
\end{align}
depends only on observations $\{Z[\tau,m]:\tau\leq t\}$.
Consequently, neither $C[t,s,m]$ nor its magnitude
$W[t,s,m]=|C[t,s,m]|$ depends on an observation after time $t$.
\end{proposition}

\begin{proof}
Substituting $\tau=t-k$ into \eqref{eq:cwwt} gives
\eqref{eq:thm-causal-eq}. Because
$k\in\{0,1,\ldots,K_s\}$, every corresponding index satisfies
$\tau=t-k\leq t$. No observation with index greater than $t$
appears in the sum. Since taking the modulus is a pointwise
operation on $C[t,s,m]$, the same causality property holds for
$W[t,s,m]=|C[t,s,m]|$.
\end{proof}

\begin{corollary}[Left-boundary burn-in length]
\label{lem:burnin}
Under the initialization convention $Z[\tau,m]=0$ for $\tau<0$,
the coefficient $C[t,s,m]$ uses at least one zero-padded
prehistory value whenever $t<K_s=\lfloor6s\rfloor$. The smallest
index at which every FIR tap is supported by an observed
historical sample for every $s\in\mathcal{S}$ is
\begin{equation}
B
\;=\;
\max_{s\in\mathcal{S}}K_s
\;=\;
\max_{s\in\mathcal{S}}\lfloor6s\rfloor.
\label{eq:burnin}
\end{equation}
For $\mathcal{S}=\{1,2,3,5,8,13,21,32\}$, this gives $B=192$.
The magnitude $W[t,s,m]=|C[t,s,m]|$ inherits the same
left-boundary dependence.
\end{corollary}

\begin{proof}
At time $t$, the earliest input index used by the scale-$s$
filter is $t-K_s$. Every tap is supported by an observed input
if and only if $t-K_s\geq0$, equivalently $t\geq K_s$.
Satisfying this condition simultaneously for all
$s\in\mathcal{S}$ requires
$t\geq\max_{s\in\mathcal{S}}K_s$. Because the largest scale is
$s=32$, the resulting burn-in is
$B=\lfloor6\cdot32\rfloor=192$.
\end{proof}

\begin{theorem}[Factorized attention complexity]
\label{thm:complexity}
Consider joint self-attention over the flattened DST tensor
$\mathbf{D}^{(i)}\in\mathbb{R}^{w\times S\times M}$, represented as
$wS$ tokens and projected to query, key, and value width $d_J$.
The quadratic attention core, comprising score construction
$\mathbf{Q}\mathbf{K}^{\top}$ and attention--value multiplication
$\mathbf{A}\mathbf{V}$, has per-layer cost
\begin{equation}
\mathcal{C}_{\mathrm{joint}}
\;=\;
\Theta\!\bigl((wS)^2d_J\bigr).
\label{eq:cjoint}
\end{equation}
The factorized STST applies temporal attention to $w$ tokens at
projected width $d_A$ and scale attention to $S$ tokens at projected
width $d_B$. The combined quadratic-core cost of one corresponding
temporal--scale layer pair is
\begin{equation}
\mathcal{C}_{\mathrm{fact}}
\;=\;
\Theta\!\bigl(w^2d_A+S^2d_B\bigr).
\label{eq:cfact}
\end{equation}
At a common projected width $d_J=d_A=d_B=d$, the per-layer ratio is
\begin{equation}
\frac{\mathcal{C}_{\mathrm{joint}}}
     {\mathcal{C}_{\mathrm{fact}}}
\;=\;
\frac{w^2S^2}{w^2+S^2}.
\label{eq:speedup}
\end{equation}
For $w=20$ and $S=8$, this gives
\begin{equation}
\frac{20^2\cdot8^2}{20^2+8^2}
\;=\;
\frac{25{,}600}{464}
\;\approx\;
55.2.
\end{equation}
The factor $55.2$ is therefore an equal-width, per-layer ratio for
the quadratic attention core. It is not a ratio of total model FLOPs
or wall-clock time. The implemented widths are $d_J=64$ for the
joint-STST ablation, $d_A=128$ for Stream~A, and $d_B=64$ for
Stream~B. At these configured widths, one joint layer compared with
one temporal--scale layer pair gives
\begin{equation}
\frac{(wS)^2d_J}{w^2d_A+S^2d_B}
\;=\;
\frac{160^2\cdot64}
     {20^2\cdot128+8^2\cdot64}
\;\approx\;
29.6.
\label{eq:configured-attention-ratio}
\end{equation}
Because the implemented factorized module contains two attention
layers in each stream whereas the joint-STST ablation contains one
joint attention layer, the corresponding aggregate quadratic-core
ratio is approximately $14.8$ before projection, feed-forward,
pooling, preprocessing, and other model costs are included.

Projection costs also differ. A joint layer projects $wS$ tokens
with input dimension $M$ at width $d_J$, for cost proportional to
$wSMd_J$. A corresponding factorized layer pair projects the temporal
and scale views at widths $d_A$ and $d_B$, for combined cost
proportional to $wSM(d_A+d_B)$. Consequently, the reduction in total
model computation is smaller than the quadratic-core ratios above.
\end{theorem}

\begin{proof}
For a token sequence of length $L$ and total projected width $d$,
forming $\mathbf{Q}\mathbf{K}^{\top}$ requires
$\Theta(L^2d)$ operations. Multiplying the resulting attention
matrix by the value matrix, $\mathbf{A}\mathbf{V}$, also requires
$\Theta(L^2d)$ operations and is therefore of the same asymptotic
order, not a lower-order term. The softmax operation contributes
$\Theta(hL^2)$ for $h$ heads, while the query, key, value, output,
and feed-forward projections contribute separate non-quadratic
terms.

For joint attention, $L=wS$ and the projected width is $d_J$, which
gives \eqref{eq:cjoint}. For the factorized design, the
temporal stream has $L=w$ and width $d_A$, while the scale stream has
$L=S$ and width $d_B$. Adding their costs gives
\eqref{eq:cfact}. Setting $d_J=d_A=d_B=d$ yields
\eqref{eq:speedup}. Substitution of the configured widths
and layer counts gives the implementation-specific ratios stated
above.
\end{proof}

\begin{remark}[Trade-off, not equivalence]
\label{rem:tradeoff}
Stream A and Stream B never form a single $(t, s)$ pairwise score, so neither stream models joint cross-axis interactions on its own. The cross-stream MLP of \eqref{eq:cross} mixes the two pooled summaries and recovers part of those interactions, not all of them. The reason is structural. Both pooled vectors $\mathbf{h}^{(i)}_{\text{ts}}$ and $\mathbf{h}^{(i)}_{\text{ss}}$ pass through attention pooling that collapses the time axis and the scale axis before the MLP sees the data, and any information that pooling discards cannot be recovered downstream. The factorized path is therefore a restricted hypothesis class, not an equivalent one, and we make no universal-approximation claim. What the factorization gives up is an empirical question, answered for our datasets by the joint-STST ablation in Section~\ref{ssec:ablations}.
\end{remark}

\section{Experiments}\label{sec:experiments}
Sections~\ref{ssec:setup}--\ref{ssec:baselines} fix the data, protocol, metrics, and baselines. The remaining subsections report forecasting accuracy, seed stability, the gold probe, regime diagnostics, ablations, and a downstream allocation diagnostic.

\subsection{Setup, Data, and Reproducibility}\label{ssec:setup}
\subsubsection{Data}
We evaluate on seven equity indices spanning three continents (KOSPI, Nikkei~225, S\&P~500, DJI, DAX, NASDAQ Composite, NYSE Composite), downloaded from Yahoo Finance, and on gold, which serves as the single non-equity robustness probe. Prices are daily bars with auto-adjusted closes. We drop duplicate dates, sort chronologically, and impute nothing: days absent from an exchange's calendar are simply absent, and each series keeps its own calendar.
Gold shares the daily sampling and heavy tails of the indices while its volatility is driven by macro and supply dynamics rather than earnings flow, so it is the cleanest single-series probe of whether the causal-trajectory design remains useful once trained on a different asset class under the same fixed configuration. Gold is included on this basis because its economic role and behaviour relative to equity markets have been studied separately in the financial literature~\cite{baur2010gold}. Every series is daily, so the lookback window is $w = 20$ trading days and the forecast horizons are $\mathcal{H} = \{1, 3, 5, 10\}$ throughout, with no separate protocol for gold. The test window spans 2023-06-01 to 2024-05-31. That one-year interval contains roughly 250 trading observations per series, with the exact per-series counts in Table~\ref{tab:data}, since each exchange keeps its own calendar. Model selection and threshold calibration use a five-fold expanding-window walk-forward inside the pre-2023-06 history, each fold's validation block strictly after its training block.

\begin{table}[!t]
\caption{OHLCV Datasets, Ticker Symbols, and Observation Periods}
\label{tab:data}
\centering
\footnotesize
\setlength{\tabcolsep}{3pt}
\renewcommand{\arraystretch}{1.1}
\begin{adjustbox}{max width=\columnwidth}
\begin{tabular}{@{}llcrrr@{}}
\toprule
\textbf{Series} & \textbf{Ticker} & \textbf{Span} & \textbf{Obs.} &
\textbf{Windows} & \textbf{Test} \\
\midrule
KOSPI            & \texttt{\^{}KS11}   & \makecell{2010-01-29 to\\2024-05-30} & 3530 & 3063 & 235 \\
Nikkei 225       & \texttt{\^{}N225}   & \makecell{2010-02-01 to\\2024-05-30} & 3505 & 3037 & 236 \\
S\&P 500         & \texttt{\^{}GSPC}   & \makecell{2010-02-01 to\\2024-05-30} & 3607 & 3133 & 242 \\
DJI              & \texttt{\^{}DJI}    & \makecell{2010-02-01 to\\2024-05-30} & 3607 & 3133 & 242 \\
DAX              & \texttt{\^{}GDAXI}  & \makecell{2010-01-29 to\\2024-05-30} & 3638 & 3160 & 246 \\
NASDAQ Composite & \texttt{\^{}IXIC}   & \makecell{2010-02-01 to\\2024-05-30} & 3607 & 3133 & 242 \\
NYSE Composite   & \texttt{\^{}NYA}    & \makecell{2010-02-01 to\\2024-05-30} & 3607 & 3133 & 242 \\
Gold             & \texttt{GC=F}       & \makecell{2010-02-01 to\\2024-05-30} & 3604 & 3130 & 242 \\
\bottomrule
\end{tabular}
\end{adjustbox}
\end{table}

\subsubsection{Hardware and software}
All experiments were conducted on a single workstation equipped with one NVIDIA RTX 5060 Ti GPU (16\,GB VRAM). The implementation is in PyTorch 2.12.0+cu130 with NumPy 2.2.5 under Python 3.10.20. For reproducibility, we enable deterministic mode with \texttt{torch.backends.cudnn.deterministic = True} and \texttt{torch.backends.cudnn.benchmark = False}. Every main-text table uses seed 42, and Section~\ref{ssec:multiseed} reports a three-seed rerun.

\subsubsection{Hyperparameters}
Table~\ref{tab:hyperparams} consolidates every hyperparameter used in the implementation. Architectural choices were selected by aggregating validation performance across the seven equity-index development folds. No gold observation from either validation or test data influenced architecture selection. The resulting configuration was then frozen for every series, including gold.

\begin{table}[!t]
\caption{DSTNet Hyperparameters Used for All Eight Series}
\label{tab:hyperparams}
\centering
\footnotesize
\setlength{\tabcolsep}{3pt}
\renewcommand{\arraystretch}{1.05}
\begin{tabular}{@{}>{\raggedright\arraybackslash}p{0.25\columnwidth}>{\raggedright\arraybackslash}p{0.31\columnwidth}>{\raggedright\arraybackslash}p{0.38\columnwidth}@{}}
\toprule
\textbf{Component} & \textbf{Hyperparameter} & \textbf{Value} \\
\midrule
Residual denoiser    & wavelet, level                    & Symlet-4, $L=2$ \\
                     & threshold rule                    & MAD soft \eqref{eq:mad-threshold} \\
Indicators ($M=7$)   & set                               & RSI-10, Stoch \%K, CCI-14, OBV, ATR-14, Williams \%R, ROC-12 \\
Standardisation      & method                            & train-fold Z-score \\
\midrule
FIR kernel           & source, $\omega_0$                & Morlet-derived, $6.0$ \\
FIR scale parameters & set                               & $\{1,2,3,5,8,13,21,32\}$ \\
                     & count                             & $S = 8$ \\
Implementation       & filter                            & FIR via \texttt{lfilter} \\
Burn-in              & $B = \lfloor 6s_{\max}\rfloor$    & $192$ samples \\
\midrule
Lookback             & $w$                               & 20 (all 8 datasets) \\
Fusion dim           & $d_{\text{fus}}$                  & 64 \\
Temporal branch      & Conv1D filters, kernel            & 64, 3 \\
                     & BiLSTM units                      & 32, 64 \\
                     & dropout                           & 0.3 \\
STST Stream A        & tokens, dim                       & 20, 56 \\
                     & MHA layers, heads, $d_k$          & 2, 4, 32 \\
                     & attn dropout                      & 0.1 \\
STST Stream B        & tokens, dim                       & 8, 140 \\
                     & MHA layers, heads, $d_k$          & 2, 4, 16 \\
                     & attn dropout                      & 0.1 \\
Cross-stream MLP     & widths                            & 256, 128 \\
                     & dropout                           & 0.3 \\
HAD                  & gate activation                   & $\tanh$ \\
                     & per-horizon head                  & Dense(32, ReLU) $\to$ Dense(1) \\
                     & dropout                           & 0.15 \\
\midrule
Loss                 & Huber, $\delta$                   & 0.1 \\
Optimiser            & Adam, $\eta$                      & lr $=5\!\times\!10^{-4}$ \\
Batch size           & main / joint-STST                 & 32 / 16 \\
Epochs               & per fold                          & 50 \\
Random seed          & primary / stability               & 42 / \{42, 101, 202\} \\
\midrule
WFV folds            & all datasets                      & 5 \\
Test split           & all datasets                      & 2023-06-01 to 2024-05-31 \\
Bootstrap            & method, resamples, block          & moving block, 1000, 20 \\
\bottomrule
\end{tabular}
\end{table}

\subsection{Targets, Metrics, and Confidence Intervals}\label{ssec:metrics}
Let $P_t$ denote the adjusted closing price at the forecast origin $t$. For each horizon $h \in \mathcal{H}$ the network target is the cumulative log return
\begin{equation}
r_{t,h} \;=\; \log\!\bigl(P_{t+h} / P_t\bigr),
\label{eq:target-return}
\end{equation}
the network emits $\hat{r}_{t,h}$, and every price-level metric is computed on the reconstruction
\begin{equation}
\hat{P}_{t+h \mid t} \;=\; P_t \exp\!\bigl(\hat{r}_{t,h}\bigr).
\label{eq:price-reconstruction}
\end{equation}
Under this convention the persistence baseline is the zero-return forecast $\hat{r}_{t,h} = 0$, equivalently $\hat{P}_{t+h \mid t} = P_t$, which is why it anchors every table. The same target and reconstruction apply to every baseline, so no model gains a unit advantage.

We report three price-level metrics per horizon, all on reconstructed prices, computed over the set $\mathcal{T}_h$ of test forecast origins common to every model at horizon $h$, with $n_h = |\mathcal{T}_h|$:
\begin{equation}
\mathrm{MAE}_h = \frac{1}{n_h} \sum_{t \in \mathcal{T}_h} \bigl|P_{t+h} - \hat{P}_{t+h \mid t}\bigr|,
\label{eq:metrics-mae}
\end{equation}
\begin{equation}
\mathrm{RMSE}_h = \Bigl(\frac{1}{n_h} \sum_{t \in \mathcal{T}_h} \bigl(P_{t+h} - \hat{P}_{t+h \mid t}\bigr)^{2}\Bigr)^{1/2},
\label{eq:metrics-mae-rmse}
\end{equation}
\begin{equation}
\mathrm{MAPE}_h = \frac{100}{n_h} \sum_{t \in \mathcal{T}_h} \frac{\bigl|P_{t+h} - \hat{P}_{t+h \mid t}\bigr|}{P_{t+h}},
\label{eq:metrics-mape}
\end{equation}
where the factor of $100$ expresses MAPE in percent. Directional accuracy at horizon $h$ is the fraction of test windows on which $\operatorname{sign}(\hat{r}_{t,h}) = \operatorname{sign}(r_{t,h})$, computed on returns before reconstruction. The trading rule of Section~\ref{ssec:trading} acts on the raw return output $\hat{r}_{t,1}$ directly.

\subsubsection{Block-bootstrap confidence intervals}
At the longest horizon $h = 10$, each series supplies between 235 and 246 eligible forecast origins, as listed in Table~\ref{tab:data}, since split membership is decided by a sample's latest target date and every retained test sample therefore carries a valid ten-step target. Adjacent ten-day errors nevertheless overlap heavily, because their target intervals share observations. The test period contains only about two dozen disjoint ten-day target blocks, so the number of forecast origins overstates how much approximately independent information is available at this horizon. Point estimates of MAE, RMSE, MAPE, and DA consequently require cautious interpretation. We report 95\% confidence intervals computed via the moving-block bootstrap of~\cite{kunsch1989} with block length 20 and 1000 resamples per metric. Each resample preserves local temporal correlation and yields one realisation of the metric. The 2.5\% and 97.5\% quantiles of the resample distribution are the CI bounds. The bootstrap sample length matches the test-set length so the resampled distribution has the same coverage of the data as the original. We apply the same bootstrap protocol to every baseline for fairness.

\subsubsection{Diebold--Mariano comparison}
At the one-day horizon, irrespective of the Huber objective ($\delta=0.1$) used to train DSTNet, we conduct Diebold--Mariano comparisons of equal expected predictive loss~\cite{diebold2002comparing} using absolute reconstructed-price error, aligned with the reported MAE. For model $m$, let
\begin{equation*}
e_t^{(m)} = \left|P_{t+1}-\hat{P}_{t+1\mid t}^{(m)}\right|,
\end{equation*}
and define the loss differential as
\begin{equation}
d_t = e_t^{(b)} - e_t^{(\mathrm{DSTNet})}.
\label{eq:dm-loss}
\end{equation}
A positive statistic corresponds to a positive sample-mean loss differential and therefore favours DSTNet. We estimate the variance of the loss differential using a Bartlett-weighted estimator with the standard horizon-based truncation lag $h-1$. Because the reported comparisons use $h=1$, the truncation lag is zero. We apply the Harvey-Leybourne-Newbold small-sample correction~\cite{harvey1997testing} and calculate two-sided significance levels using a Student-$t$ reference distribution with $n_1-1$ degrees of freedom. A zero truncation lag does not by itself guarantee that the loss differential is serially independent, since adjacent forecast origins use heavily overlapping lookback windows and model errors can remain persistent. The reported $p$-values are unadjusted across the nine baseline comparisons within each series and are therefore exploratory rather than familywise-confirmatory. A non-significant comparison indicates insufficient evidence of unequal expected loss and does not establish model equivalence. Significance is denoted by $^{*}p<0.10$, $^{**}p<0.05$, and $^{***}p<0.01$.

\subsection{Baselines}\label{ssec:baselines}
We use a uniform-protocol reproduction for the empirical comparison. Every baseline is retrained under the chronological protocol of Section~\ref{ssec:setup}, using the same expanding-window validation folds and the same single final hold-out as DSTNet, so that every number in the main tables comes from one pipeline rather than from a patchwork of source papers. Finance-specific methods whose implementations and data protocols cannot be reproduced identically are positioned in Section~\ref{sec:related} but are not treated as controlled numerical baselines.

\subsubsection{Uniform-protocol baselines}
The main per-horizon tables hold nine learned baselines, a naive persistence forecaster, and DSTNet. Of the nine, three are recurrent, one is a tree ensemble, one is a vanilla attention architecture, and four are contemporary forecasters introduced below. The recurrent group is a single-layer LSTM~\cite{hochreiter1997}, a bidirectional LSTM, and a CNN-BiLSTM that prepends a convolutional stage to the recurrence. That last model is also the temporal branch of DSTNet in isolation, which makes it the sharpest ablation-style baseline in the panel. The tree ensemble is a gradient-boosted XGBoost regressor~\cite{chen2016xgboost} on flattened lookback windows, included as a non-neural reference rather than as an expected front-runner. The attention baseline is a vanilla Transformer encoder~\cite{vaswani2017}. We add four modern forecasters under the identical protocol: DLinear~\cite{zeng2023dlinear}, a per-horizon linear map, N-BEATS~\cite{nbeats2020}, a deep residual direct forecaster, PatchTST~\cite{patchtst2023}, a patch-based channel-independent transformer, and iTransformer~\cite{itransformer2024}, which attends across the variate dimension. Every baseline is trained on the same denoised OHLCV history and trailing technical indicators available to DSTNet. DSTNet additionally converts those indicators into the causal DST tensor through the proposed spectral front-end, so the uniform-protocol comparison isolates the value of the learned representation and architecture under the same information set, not under extra external data.

\subsubsection{Baseline fairness and random-walk reference}
Persistence is the zero-return forecast of Section~\ref{ssec:metrics}, equivalent to a driftless random walk on reconstructed price levels, and improving on it out of sample is the classical bar for daily equity forecasting. For reference-based comparisons we use the per-cell minimum across all ten fixed competitors rather than the best learned model alone, which keeps the comparison conservative for DSTNet. All nine learned baselines were tuned under the same validation protocol as DSTNet, with identical search methods, trial budgets, validation criteria, and selection folds. Search spaces were architecture-specific, no baseline relied on author-selected default hyperparameters, and every selected configuration was frozen before final hold-out evaluation.

\subsection{Main Multi-Horizon Forecasting Results}\label{ssec:results}
Table~\ref{tab:results-finance-mae} reports per-horizon test MAE on the eight-series panel against the uniform-protocol baselines, with 95\% moving-block bootstrap intervals. The RMSE and MAPE companions and the $H{=}1$ directional-accuracy table are Tables~\ref{tab:results-finance-rmse}, \ref{tab:results-finance-mape}, and~\ref{tab:baseline-da-h1} in Appendix~\ref{app:full-metrics}. MAE and MAPE give nearly identical patterns, so the main text reports MAE and states where RMSE differs. On the primary seed-42 run DSTNet holds the lowest point estimate in 32 of 32 MAE cells, 32 of 32 MAPE cells, and 31 of 32 RMSE cells, the exception being Nikkei~225 at $H{=}1$, where persistence is lower by under one percent. Section~\ref{ssec:multiseed} reports the three-seed result, which is weaker.

\begin{table}[!t]
\caption{Panel-Mean Error Relative to Persistence Across 32 Series-Horizon Cells for Seed 42. Positive Values Are Worse Than Persistence, With Wins Shown in Parentheses.}
\label{tab:pers-summary}
\centering
\footnotesize
\setlength{\tabcolsep}{4pt}
\begin{adjustbox}{max width=\columnwidth}
\begin{tabular}{@{}lrrr@{}}
\toprule
\textbf{Model} & \textbf{MAE (\%)} & \textbf{RMSE (\%)} & \textbf{MAPE (\%)} \\
\midrule
LSTM                & $+12.58$ (0) & $+12.30$ (0) & $+12.59$ (0) \\
BiLSTM              & $+8.53$ (0)  & $+8.32$ (0)  & $+8.54$ (0)  \\
CNN-BiLSTM          & $+5.01$ (0)  & $+4.76$ (1)  & $+5.01$ (0)  \\
XGBoost             & $+12.83$ (0) & $+12.73$ (0) & $+12.84$ (0) \\
Vanilla Transformer & $+2.88$ (1)  & $+3.03$ (4)  & $+2.88$ (1)  \\
DLinear             & $+3.21$ (0)  & $+3.08$ (2)  & $+3.21$ (0)  \\
N-BEATS             & $+5.02$ (0)  & $+4.80$ (0)  & $+5.02$ (0)  \\
PatchTST            & $+3.09$ (2)  & $+2.95$ (5)  & $+3.09$ (2)  \\
iTransformer        & $+3.98$ (0)  & $+3.54$ (4)  & $+3.98$ (0)  \\
\midrule
DSTNet              & $\mathbf{-2.27}$ (32) & $\mathbf{-2.95}$ (31) & $\mathbf{-2.27}$ (32) \\
\bottomrule
\end{tabular}
\end{adjustbox}
\end{table}

\subsubsection{Under MAE and MAPE the learned baselines do not improve on the random walk}
Table~\ref{tab:pers-summary} aggregates every cell against the persistence benchmark and is the central comparison of this study. Under MAE and MAPE, persistence is the lowest non-DSTNet point estimate in 29 of the 32 cells. Under both metrics no learned baseline is below it at $H\in\{1,3,5\}$ on any series, and the same three exceptions fall at $H{=}10$: PatchTST on DJI and NYSE, and the vanilla Transformer on DAX. The nine learned models sit between 2.88\% and 12.83\% above persistence on panel-mean MAE. Under RMSE the pattern is less uniform: 16 learned-model cells fall below persistence, 13 at $H{=}10$. We therefore restrict the stronger random-walk comparison to MAE and MAPE rather than treating it as metric-independent. DSTNet is 2.27\% below persistence on panel-mean MAE, 2.27\% on panel-mean MAPE, and 2.95\% on panel-mean RMSE, and it is the only model in the panel below persistence in every MAE and MAPE cell. Under MAE the margin grows monotonically with horizon on all eight series, from 0.67--0.92\% at $H{=}1$ to 3.43--4.49\% at $H{=}10$.

Because the random walk, rather than a learned model, is the binding competitor under these metrics, a large Diebold--Mariano statistic against a weak baseline is weak evidence of forecasting skill, and we do not treat it as more. Fig.~\ref{fig:mae-horizon-panels} shows the result per series after normalisation by persistence MAE. The DSTNet gap is small relative to its bootstrap band at $H{=}1$ and widens with the horizon.

\begin{sidewaystable*}[p]
\caption{Per-horizon test MAE across the full evaluation panel (KOSPI, Nikkei~225, S\&P 500, DJI, DAX, NASDAQ, NYSE, Gold). 95\% moving-block bootstrap CIs in brackets. Lower is better. Boldface indicates the lowest point estimate in each row and is descriptive rather than a claim of statistical significance.}
\label{tab:results-finance-mae}
\centering
\fontsize{6.3}{7.2}\selectfont
\setlength{\tabcolsep}{2pt}
\renewcommand{\arraystretch}{1.0}
\begin{adjustbox}{max totalsize={\textheight}{0.88\textwidth}}
\begin{tabular}{@{}ll@{\hspace{4pt}}ccccccccccc@{}}
\toprule
\textbf{Asset} & \textbf{H} & \textbf{Persistence} & \textbf{LSTM} & \textbf{BiLSTM} & \textbf{CNN-BiLSTM} & \textbf{XGB} & \textbf{Transformer} & \textbf{DLinear} & \textbf{N-BEATS} & \textbf{PatchTST} & \textbf{iTransformer} & \textbf{DSTNet} \\
\midrule
\multirow{4}{*}{KOSPI} & 1 & \mc{19.94}{17.88}{23.35} & \mc{23.32}{20.47}{27.30} & \mc{22.18}{19.81}{26.02} & \mc{21.09}{18.74}{24.53} & \mc{22.36}{20.16}{25.89} & \mc{21.10}{18.72}{24.50} & \mc{21.04}{18.63}{24.00} & \mc{21.26}{19.14}{24.31} & \mc{21.15}{18.69}{24.45} & \mc{21.11}{18.86}{24.38} & \mcb{19.78}{17.45}{23.11} \\
 & 3 & \mc{33.73}{29.71}{40.37} & \mc{38.67}{33.60}{45.11} & \mc{36.46}{32.30}{43.51} & \mc{35.45}{31.52}{41.56} & \mc{38.47}{33.81}{46.03} & \mc{35.24}{31.24}{41.03} & \mc{35.08}{31.22}{40.84} & \mc{35.55}{31.02}{41.43} & \mc{35.06}{30.72}{41.75} & \mc{35.79}{31.64}{42.92} & \mcb{33.25}{28.86}{38.84} \\
 & 5 & \mc{44.43}{37.28}{53.82} & \mc{50.53}{42.52}{61.68} & \mc{47.63}{40.14}{58.71} & \mc{46.56}{38.32}{57.80} & \mc{51.17}{43.70}{64.63} & \mc{45.50}{37.59}{55.63} & \mc{46.20}{38.78}{58.09} & \mc{47.16}{40.50}{57.45} & \mc{45.79}{37.98}{55.97} & \mc{46.46}{39.41}{58.31} & \mcb{43.34}{35.64}{53.20} \\
 & 10 & \mc{65.44}{54.37}{80.57} & \mc{70.34}{59.97}{83.80} & \mc{69.55}{57.90}{86.21} & \mc{67.28}{57.02}{82.53} & \mc{72.25}{62.44}{86.55} & \mc{65.79}{55.63}{80.28} & \mc{66.81}{57.69}{81.22} & \mc{67.81}{58.47}{80.87} & \mc{66.61}{57.68}{81.74} & \mc{67.29}{57.89}{81.99} & \mcb{63.08}{52.64}{75.78} \\
\midrule
\multirow{4}{*}{Nikkei~225} & 1 & \mc{303.54}{268.77}{348.33} & \mc{356.56}{321.36}{408.03} & \mc{338.82}{298.56}{393.90} & \mc{327.65}{290.19}{383.78} & \mc{347.55}{307.33}{397.59} & \mc{316.05}{281.92}{362.41} & \mc{315.60}{280.49}{368.01} & \mc{323.68}{289.91}{369.63} & \mc{318.27}{279.48}{367.11} & \mc{320.01}{282.21}{366.91} & \mcb{300.74}{271.14}{349.52} \\
 & 3 & \mc{525.17}{464.42}{625.97} & \mc{598.84}{530.47}{718.50} & \mc{586.89}{513.98}{705.54} & \mc{556.68}{495.24}{648.14} & \mc{582.40}{509.66}{681.85} & \mc{542.66}{484.82}{647.28} & \mc{541.33}{471.24}{641.28} & \mc{549.39}{490.75}{660.14} & \mc{542.04}{477.65}{639.65} & \mc{545.18}{475.02}{648.96} & \mcb{516.61}{455.73}{613.74} \\
 & 5 & \mc{698.02}{598.42}{861.80} & \mc{792.19}{654.72}{997.03} & \mc{753.95}{629.64}{911.52} & \mc{736.47}{630.41}{913.24} & \mc{782.12}{658.11}{985.52} & \mc{713.03}{595.93}{872.37} & \mc{714.61}{607.11}{898.63} & \mc{729.35}{616.21}{918.08} & \mc{714.42}{612.48}{867.14} & \mc{716.39}{601.04}{877.06} & \mcb{682.19}{580.71}{861.28} \\
 & 10 & \mc{955.25}{793.75}{1184.01} & \mc{1057.57}{902.88}{1265.30} & \mc{1018.76}{853.22}{1226.28} & \mc{984.64}{844.08}{1200.65} & \mc{1073.25}{897.73}{1280.45} & \mc{964.20}{810.18}{1173.10} & \mc{976.12}{814.36}{1195.14} & \mc{985.00}{834.83}{1192.39} & \mc{966.30}{809.67}{1169.67} & \mc{967.85}{805.05}{1156.93} & \mcb{922.44}{791.03}{1101.35} \\
\midrule
\multirow{4}{*}{S\&P 500} & 1 & \mc{26.88}{24.18}{30.68} & \mc{30.69}{27.52}{34.99} & \mc{29.60}{26.06}{34.71} & \mc{29.17}{25.73}{34.08} & \mc{31.00}{27.82}{36.17} & \mc{28.15}{25.18}{32.90} & \mc{28.30}{25.19}{32.57} & \mc{28.73}{25.79}{33.72} & \mc{28.07}{25.03}{32.11} & \mc{28.58}{25.40}{33.01} & \mcb{26.66}{23.52}{30.64} \\
 & 3 & \mc{48.94}{43.29}{58.65} & \mc{56.68}{49.36}{67.89} & \mc{54.19}{47.15}{64.90} & \mc{51.90}{45.15}{60.76} & \mc{56.19}{49.46}{65.89} & \mc{50.60}{43.79}{59.27} & \mc{50.79}{44.45}{60.47} & \mc{52.04}{45.85}{60.41} & \mc{51.37}{45.33}{61.31} & \mc{50.94}{44.14}{59.66} & \mcb{48.13}{41.67}{56.74} \\
 & 5 & \mc{62.39}{51.77}{78.46} & \mc{70.12}{59.22}{87.22} & \mc{66.70}{57.30}{83.40} & \mc{65.19}{54.11}{79.18} & \mc{69.73}{57.86}{85.17} & \mc{63.69}{53.22}{78.42} & \mc{64.56}{53.43}{78.06} & \mc{65.66}{54.38}{80.58} & \mc{63.86}{54.74}{80.71} & \mc{65.28}{55.23}{81.24} & \mcb{60.82}{50.83}{76.45} \\
 & 10 & \mc{95.36}{79.58}{114.35} & \mc{104.86}{90.34}{129.62} & \mc{100.95}{86.32}{124.79} & \mc{98.37}{82.19}{121.23} & \mc{105.46}{91.36}{126.36} & \mc{97.18}{80.81}{118.16} & \mc{97.06}{81.16}{116.62} & \mc{98.35}{82.21}{119.29} & \mc{96.61}{81.82}{119.38} & \mc{97.64}{82.72}{119.57} & \mcb{91.94}{76.91}{110.95} \\
\midrule
\multirow{4}{*}{DJI} & 1 & \mc{174.71}{154.44}{200.58} & \mc{205.97}{186.02}{241.15} & \mc{195.26}{172.93}{224.29} & \mc{186.49}{166.45}{213.97} & \mc{194.54}{174.17}{225.04} & \mc{184.62}{164.59}{211.27} & \mc{181.56}{161.59}{212.82} & \mc{183.92}{164.24}{214.93} & \mc{183.47}{163.16}{214.77} & \mc{184.15}{161.72}{212.47} & \mcb{173.14}{151.94}{198.38} \\
 & 3 & \mc{332.69}{291.41}{393.07} & \mc{376.79}{330.78}{448.58} & \mc{362.11}{320.25}{433.13} & \mc{349.20}{310.73}{411.20} & \mc{369.02}{329.72}{432.40} & \mc{345.06}{303.58}{413.53} & \mc{345.35}{300.20}{400.91} & \mc{346.70}{305.43}{413.33} & \mc{344.29}{297.77}{407.74} & \mc{348.47}{310.85}{404.98} & \mcb{327.31}{285.03}{386.09} \\
 & 5 & \mc{429.93}{363.46}{533.70} & \mc{473.47}{393.82}{590.28} & \mc{465.30}{387.41}{570.29} & \mc{452.62}{373.05}{564.45} & \mc{476.26}{401.81}{591.30} & \mc{436.53}{362.19}{531.78} & \mc{441.68}{379.45}{555.04} & \mc{448.16}{375.01}{561.32} & \mc{436.03}{368.27}{537.32} & \mc{443.40}{379.46}{551.18} & \mcb{417.93}{354.10}{509.89} \\
 & 10 & \mc{645.66}{553.90}{790.75} & \mc{697.13}{595.30}{830.92} & \mc{693.55}{593.61}{838.70} & \mc{661.67}{552.09}{807.19} & \mc{735.66}{620.45}{891.77} & \mc{650.04}{561.67}{789.67} & \mc{655.95}{564.70}{784.95} & \mc{663.78}{553.67}{803.25} & \mc{645.54}{541.80}{795.91} & \mc{660.02}{563.42}{816.94} & \mcb{620.80}{520.96}{751.92} \\
\midrule
\multirow{4}{*}{DAX} & 1 & \mc{89.75}{80.49}{103.58} & \mc{101.01}{89.91}{115.70} & \mc{99.10}{89.04}{116.15} & \mc{95.91}{84.94}{112.52} & \mc{101.81}{89.72}{117.77} & \mc{93.86}{82.88}{108.49} & \mc{93.32}{83.19}{106.37} & \mc{96.00}{85.07}{112.14} & \mc{94.59}{83.25}{108.21} & \mc{94.54}{85.02}{110.12} & \mcb{88.92}{79.95}{103.10} \\
 & 3 & \mc{162.41}{141.11}{188.79} & \mc{186.01}{162.66}{215.68} & \mc{179.06}{157.30}{207.98} & \mc{173.38}{151.92}{208.16} & \mc{181.06}{157.06}{215.75} & \mc{168.92}{147.86}{197.84} & \mc{167.69}{147.96}{199.27} & \mc{170.13}{149.36}{202.92} & \mc{168.59}{145.97}{202.69} & \mc{170.50}{147.96}{202.63} & \mcb{159.57}{138.19}{186.24} \\
 & 5 & \mc{226.43}{191.16}{281.76} & \mc{253.39}{210.68}{317.34} & \mc{240.21}{204.31}{301.90} & \mc{236.98}{195.76}{294.67} & \mc{260.38}{219.23}{318.52} & \mc{231.49}{198.41}{283.53} & \mc{233.79}{197.92}{282.74} & \mc{237.81}{195.94}{291.63} & \mc{231.91}{198.49}{292.03} & \mc{232.95}{195.97}{292.40} & \mcb{219.44}{186.57}{273.76} \\
 & 10 & \mc{332.65}{286.10}{407.19} & \mc{361.27}{311.87}{437.19} & \mc{349.51}{297.29}{429.72} & \mc{336.78}{283.25}{408.93} & \mc{381.41}{327.31}{465.73} & \mc{331.39}{284.09}{399.24} & \mc{335.53}{284.91}{413.85} & \mc{345.33}{291.48}{425.47} & \mc{336.00}{281.30}{415.28} & \mc{336.87}{282.39}{411.20} & \mcb{317.71}{275.01}{385.34} \\
\midrule
\multirow{4}{*}{NASDAQ} & 1 & \mc{115.36}{102.93}{131.98} & \mc{129.84}{115.96}{147.57} & \mc{130.60}{116.08}{151.35} & \mc{124.75}{111.35}{142.53} & \mc{128.37}{115.85}{148.95} & \mc{121.16}{107.00}{139.39} & \mc{120.88}{107.36}{138.36} & \mc{121.81}{107.97}{139.66} & \mc{120.62}{106.44}{137.12} & \mc{121.62}{108.60}{141.35} & \mcb{114.36}{100.92}{132.10} \\
 & 3 & \mc{205.51}{181.12}{246.32} & \mc{230.93}{203.98}{275.27} & \mc{228.18}{202.42}{266.70} & \mc{219.34}{193.01}{255.50} & \mc{235.53}{209.52}{277.77} & \mc{213.37}{185.99}{255.07} & \mc{213.98}{187.79}{254.31} & \mc{218.37}{192.46}{261.36} & \mc{213.45}{188.34}{256.29} & \mc{218.37}{193.22}{258.43} & \mcb{201.91}{177.52}{234.86} \\
 & 5 & \mc{267.25}{223.90}{333.56} & \mc{290.64}{240.51}{367.02} & \mc{286.05}{241.27}{346.18} & \mc{275.03}{230.01}{332.81} & \mc{305.66}{262.30}{375.99} & \mc{272.02}{231.34}{330.03} & \mc{273.92}{234.94}{344.91} & \mc{282.04}{239.99}{341.09} & \mc{275.77}{236.35}{338.27} & \mc{276.08}{229.07}{347.08} & \mcb{259.48}{217.32}{326.91} \\
 & 10 & \mc{379.31}{317.13}{465.65} & \mc{420.86}{358.54}{516.27} & \mc{395.37}{338.02}{483.62} & \mc{387.76}{330.69}{472.02} & \mc{419.78}{360.78}{505.88} & \mc{384.31}{332.97}{467.33} & \mc{389.48}{324.61}{468.09} & \mc{394.65}{339.73}{469.14} & \mc{382.21}{328.63}{465.84} & \mc{384.58}{327.69}{462.52} & \mcb{364.00}{304.05}{434.01} \\
\midrule
\multirow{4}{*}{NYSE} & 1 & \mc{93.74}{84.61}{107.05} & \mc{106.71}{95.91}{123.68} & \mc{104.38}{93.42}{121.25} & \mc{100.02}{90.07}{116.54} & \mc{104.42}{93.13}{120.10} & \mc{99.32}{88.04}{115.93} & \mc{98.37}{88.53}{112.75} & \mc{100.14}{88.08}{114.75} & \mc{99.18}{87.49}{114.78} & \mc{98.78}{88.68}{113.52} & \mcb{93.01}{81.63}{107.80} \\
 & 3 & \mc{184.52}{164.20}{219.28} & \mc{207.65}{181.03}{241.75} & \mc{202.91}{175.66}{238.73} & \mc{195.60}{169.95}{231.86} & \mc{210.62}{185.10}{247.29} & \mc{190.23}{166.07}{224.13} & \mc{192.50}{168.11}{229.10} & \mc{193.63}{168.16}{232.27} & \mc{192.97}{171.76}{226.30} & \mc{195.13}{173.18}{226.47} & \mcb{181.66}{158.89}{213.54} \\
 & 5 & \mc{250.63}{210.17}{306.01} & \mc{285.63}{236.74}{360.88} & \mc{267.08}{222.42}{323.93} & \mc{259.41}{221.91}{328.03} & \mc{279.65}{231.46}{351.06} & \mc{255.57}{212.65}{322.91} & \mc{259.75}{221.83}{324.80} & \mc{263.52}{223.29}{331.21} & \mc{258.44}{215.58}{320.38} & \mc{257.78}{212.13}{312.11} & \mcb{243.91}{204.31}{298.94} \\
 & 10 & \mc{382.76}{322.20}{469.92} & \mc{418.50}{361.25}{497.86} & \mc{400.09}{336.91}{488.32} & \mc{387.87}{325.33}{467.83} & \mc{440.51}{380.29}{545.21} & \mc{385.38}{330.40}{473.95} & \mc{386.12}{330.38}{464.36} & \mc{396.90}{341.08}{475.28} & \mc{382.60}{318.78}{467.85} & \mc{386.15}{323.52}{471.09} & \mcb{366.27}{306.97}{438.58} \\
\midrule
\multirow{4}{*}{Gold} & 1 & \mc{12.03}{10.59}{13.78} & \mc{13.90}{12.47}{16.20} & \mc{13.22}{11.91}{15.38} & \mc{12.98}{11.69}{14.86} & \mc{13.62}{12.03}{15.64} & \mc{12.58}{11.19}{14.65} & \mc{12.55}{11.25}{14.51} & \mc{12.76}{11.24}{14.79} & \mc{12.64}{11.33}{14.67} & \mc{12.78}{11.33}{14.53} & \mcb{11.95}{10.65}{13.78} \\
 & 3 & \mc{22.95}{20.20}{27.22} & \mc{26.43}{23.35}{31.65} & \mc{25.17}{22.13}{30.01} & \mc{24.39}{21.31}{28.60} & \mc{25.85}{22.57}{30.27} & \mc{23.84}{21.18}{27.87} & \mc{23.71}{20.73}{28.42} & \mc{24.33}{21.34}{29.17} & \mc{23.78}{21.07}{28.20} & \mc{24.22}{21.30}{28.15} & \mcb{22.61}{19.71}{26.83} \\
 & 5 & \mc{33.21}{28.08}{40.19} & \mc{37.39}{31.16}{45.78} & \mc{35.64}{29.72}{44.39} & \mc{34.86}{29.21}{43.07} & \mc{38.11}{32.41}{46.32} & \mc{33.85}{28.49}{42.31} & \mc{33.74}{28.24}{41.56} & \mc{34.67}{28.95}{42.59} & \mc{33.65}{28.01}{41.85} & \mc{34.05}{29.03}{41.28} & \mcb{32.20}{27.08}{40.60} \\
 & 10 & \mc{50.09}{41.89}{59.99} & \mc{53.81}{46.64}{66.56} & \mc{52.90}{44.93}{65.02} & \mc{50.73}{43.01}{61.39} & \mc{55.76}{47.97}{66.61} & \mc{50.21}{41.76}{61.02} & \mc{50.50}{43.07}{60.15} & \mc{51.46}{43.75}{61.18} & \mc{50.54}{42.07}{62.55} & \mc{50.98}{43.43}{63.05} & \mcb{48.04}{41.53}{58.61} \\
\bottomrule
\end{tabular}
\end{adjustbox}
\end{sidewaystable*}

\begin{figure*}[!t]
\centering
\includegraphics[width=\textwidth]{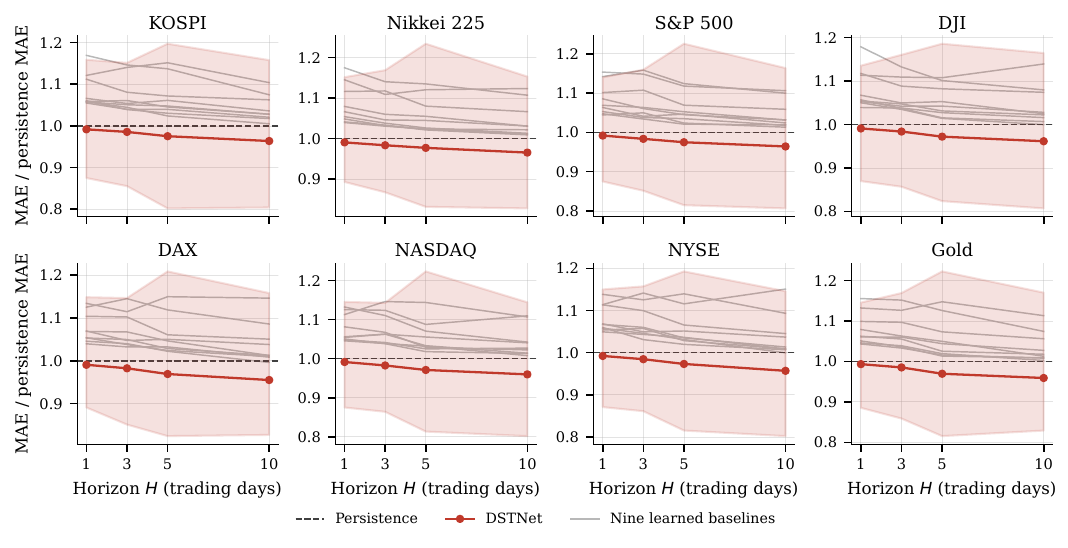}
\caption{Per-series test MAE against forecast horizon, normalised by the persistence MAE of the same series and horizon. The shaded band is DSTNet's 95\% moving-block bootstrap interval after the same normalisation. Every grey trace lies above unity at $H\in\{1,3,5\}$ on every series, and only the red trace lies below unity at all four horizons.}
\label{fig:mae-horizon-panels}
\end{figure*}

Table~\ref{tab:dm-finance} reports Diebold--Mariano tests of DSTNet against the nine learned baselines at $H=1$. Evidence is strongest against LSTM, BiLSTM, and XGBoost, CNN-BiLSTM is intermediate, and DLinear, N-BEATS, PatchTST, and iTransformer form a band in which only four of the 32 comparisons reach the unadjusted 10\% level. The pattern runs across models rather than across series, and it reproduces the ordering of Table~\ref{tab:pers-summary}: the three models with the largest statistics are the three furthest above persistence. The final row reports the same comparison against persistence. The statistics range from 0.16 to 0.24 across the eight series, so although DSTNet has a consistently lower $H{=}1$ MAE point estimate than persistence, the paired test does not reject equal predictive accuracy at this horizon, and the $H{=}1$ ordering in Table~\ref{tab:pers-summary} remains descriptive.

\begin{table*}[!t]
\caption{Diebold--Mariano tests of DSTNet against the nine learned baselines and persistence at $H=1$. Stars denote unadjusted two-sided values: *** $p<0.01$, ** $p<0.05$, and * $p<0.10$. Unstarred cells indicate insufficient evidence to reject equal expected predictive loss, not proof of equivalence.}
\label{tab:dm-finance}
\centering
\small
\setlength{\tabcolsep}{6pt}
\begin{tabular}{@{}lcccccccc@{}}
\toprule
\textbf{Model} & \textbf{KOSPI} & \textbf{Nikkei 225} & \textbf{S\&P 500} & \textbf{DJI} & \textbf{DAX} & \textbf{NASDAQ} & \textbf{NYSE} & \textbf{Gold} \\
\midrule
LSTM & 3.65*** & 4.43*** & 4.24*** & 2.43** & 3.57*** & 3.43*** & 4.18*** & 3.45*** \\
BiLSTM & 2.70*** & 2.73*** & 3.29*** & 2.18** & 2.44** & 2.87*** & 2.80*** & 1.97** \\
CNN-BiLSTM & 2.05** & 1.80* & 1.82* & 1.35 & 1.73* & 1.67* & 1.83* & 1.23 \\
XGBoost & 3.90*** & 3.09*** & 2.81*** & 3.20*** & 2.67*** & 4.83*** & 3.38*** & 2.26** \\
DLinear & 1.68* & 1.19 & 1.61 & 1.11 & 1.35 & 0.96 & 0.99 & 1.21 \\
N-BEATS & 1.56 & 1.50 & 1.94* & 1.61 & 1.93* & 1.33 & 1.44 & 1.61 \\
PatchTST & 1.46 & 1.37 & 1.28 & 1.18 & 1.50 & 1.40 & 1.75* & 1.42 \\
iTransformer & 1.49 & 1.50 & 1.23 & 1.11 & 1.30 & 1.49 & 1.46 & 1.41 \\
Vanilla Transformer & 1.56 & 1.54 & 1.18 & 1.27 & 1.37 & 1.61 & 1.06 & 0.98 \\
\midrule
Persistence & 0.19 & 0.23 & 0.21 & 0.19 & 0.24 & 0.22 & 0.16 & 0.17 \\
\bottomrule
\end{tabular}
\end{table*}

Directional accuracy (Table~\ref{tab:baseline-da-h1}) needs the same restraint. DSTNet posts the highest value on every series, but two are below chance, 49.36\% on KOSPI and 48.76\% on NYSE, and the largest is 59.09\% on DJI. No test accompanies that column, so its ordering is descriptive. Section~\ref{ssec:regime} is where the directional signal is examined for conditional information.

\subsubsection{Qualitative forecast trajectories}
Fig.~\ref{fig:forecast-traces} plots reconstructed price-level forecasts against the realised series for the S\&P~500 and for gold. At $H=1$ the forecast tracks the price closely. By $H=5$ and $H=10$ gaps open around sharp drawdowns and the recoveries after them. Each reconstruction is anchored to the observed price at the forecast origin, so visual proximity is not evidence of skill, and the figure is included only to show where the longer horizons break down.

\begin{figure*}[!t]
\centering
\includegraphics[width=0.94\textwidth]{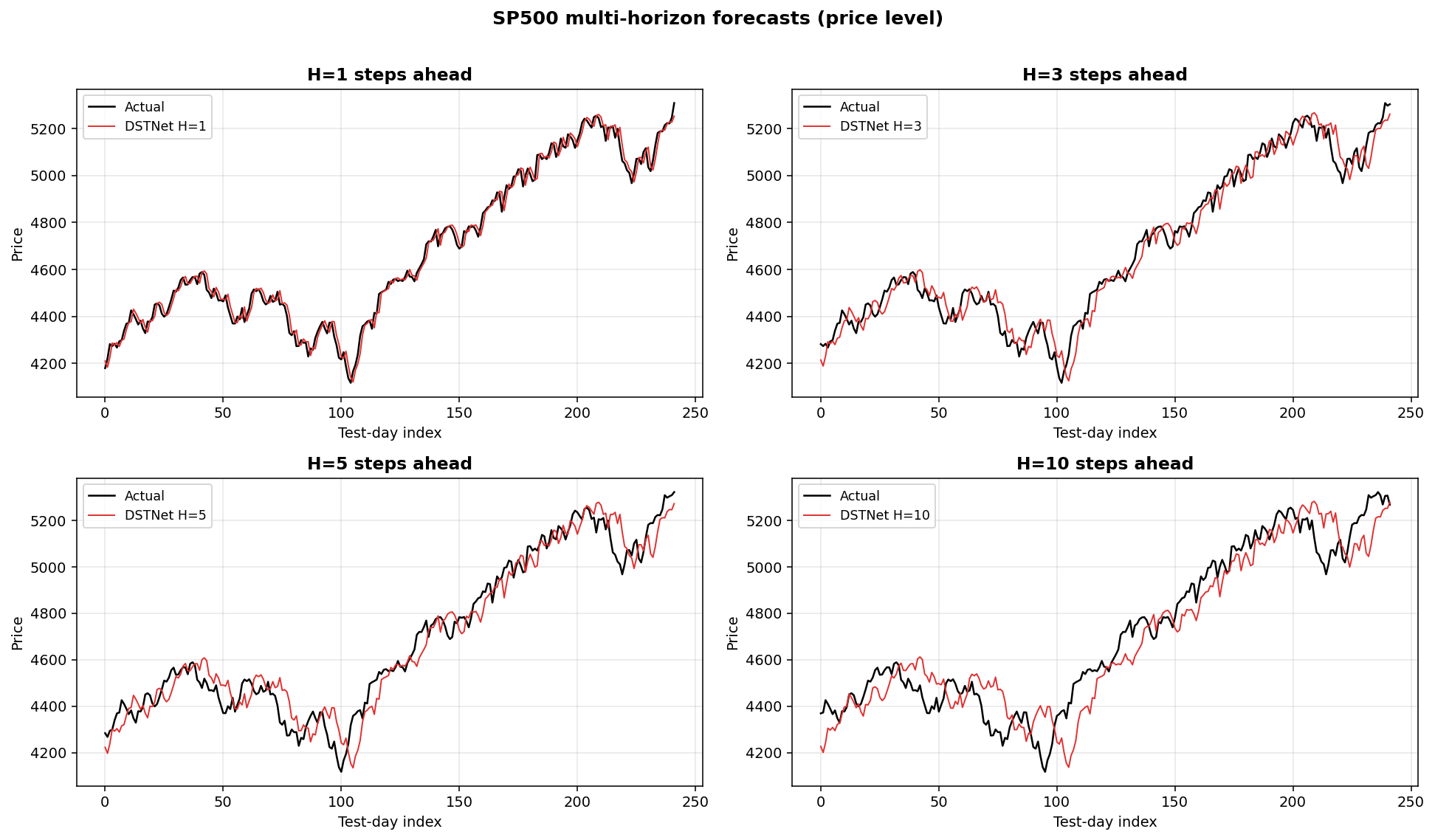}

\vspace{1.5mm}

\includegraphics[width=0.94\textwidth]{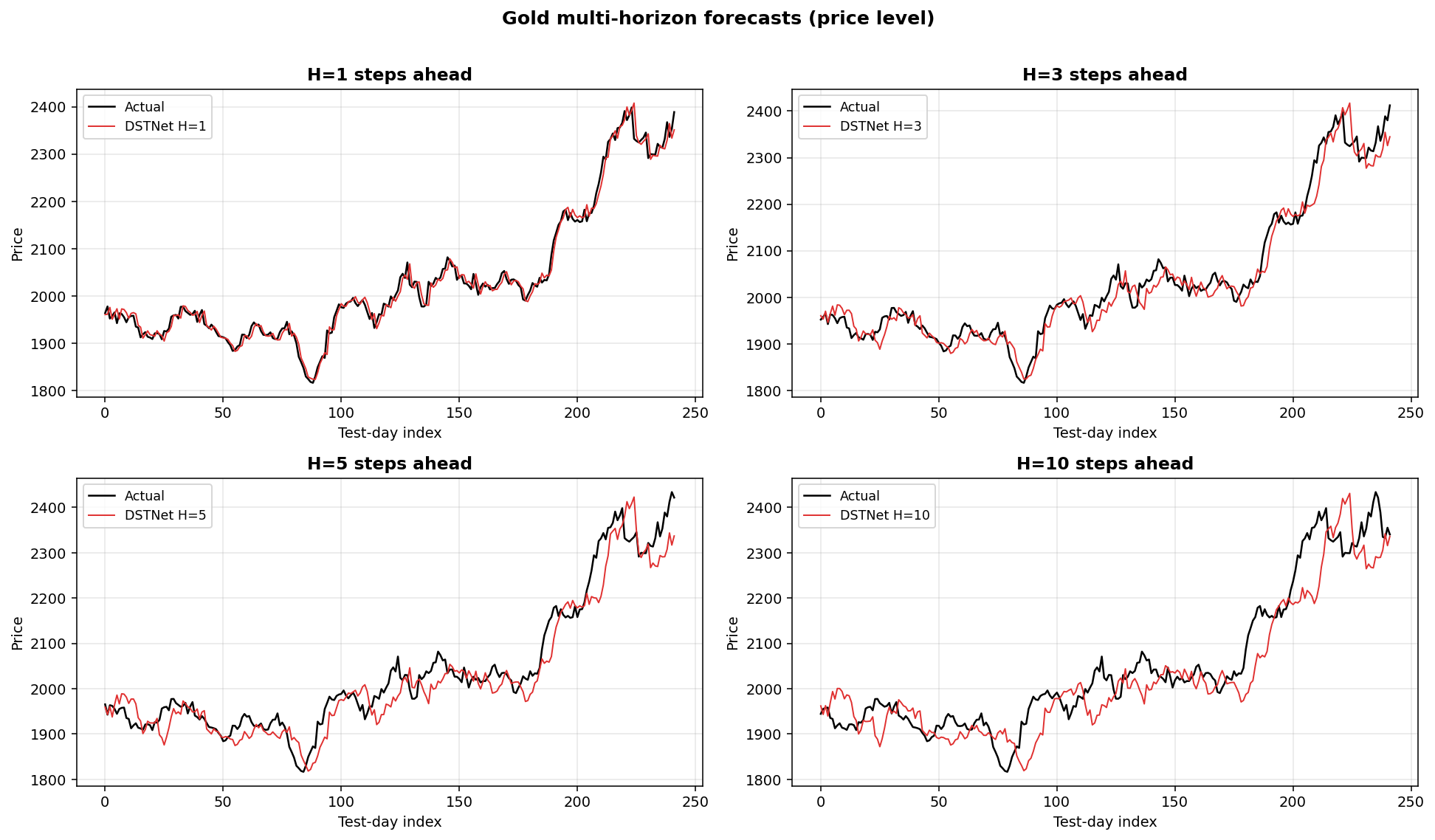}
\caption{Out-of-sample price-level forecasts from DSTNet at $H\in\{1,3,5,10\}$ trading days. Black is the realised target price and red is the reconstructed forecast. Top: S\&P~500. Bottom: gold.}
\label{fig:forecast-traces}
\end{figure*}

\subsection{Random-Seed Stability}\label{ssec:multiseed}
To test whether the primary seed-42 sweep reflects favourable initialisation, we retrain DSTNet with seeds 42, 101, and 202 while holding the data, chronological splits, preprocessing, architecture, hyperparameters, training protocol, and test period fixed. For metric $m$, series--horizon cell $c$, and seed $s$, let $e_{c,s}^{(m)}$ be the DSTNet error and $b_c^{(m)}$ the lowest fixed baseline point estimate in the corresponding primary table. We report the descriptive relative improvement $\Delta_{c,s}^{(m)}=100(b_c^{(m)}-e_{c,s}^{(m)})/b_c^{(m)}$, which is positive when DSTNet is better.

\begin{table*}[!t]
\caption{Random-seed stability of DSTNet on the primary hold-out across seeds 42, 101, and 202.}
\label{tab:multiseed-summary}
\centering
\small
\begin{tabular}{@{}lcccccc@{}}
\toprule
\textbf{Metric} & \textbf{Mean $\Delta$ (\%)} & \textbf{Median CV (\%)} & \textbf{Max CV (\%)} & \textbf{Seed-mean wins} & \textbf{All-seed wins} & \textbf{Seed-run wins} \\
\midrule
MAE  & +2.48 & 0.98 & 10.60 & 26/32 & 23/32 & 87/96 \\
MAPE & +2.43 & 0.81 & 10.05 & 26/32 & 24/32 & 88/96 \\
RMSE & +2.82 & 0.59 &  9.69 & 25/32 & 23/32 & 84/96 \\
\bottomrule
\end{tabular}
\end{table*}

Table~\ref{tab:multiseed-summary} qualifies the seed-42 headline rather than reproducing it, with cellwise means and standard deviations in Appendix~\ref{app:multiseed}. Typical initialisation variance is small, with median cellwise CV below 1\% for all three metrics, but the tail is not: maximum CVs run from 9.69\% to 10.60\%, seven or eight cells per metric exceed 5\%, and the variability concentrates in KOSPI and in longer-horizon NASDAQ and NYSE cells. The seed-mean win counts fall to 26/32 for MAE, 26/32 for MAPE, and 25/32 for RMSE, the stricter all-seed counts to 23, 24, and 23, and across the 96 seed--cell comparisons DSTNet is below the fixed reference in 87, 88, and 84. The 32/32 seed-42 result is therefore not seed-invariant.

\subsection{Non-Equity Robustness Probe on Gold}\label{ssec:results-commodity}
A separate DSTNet is trained on gold's own history with the equity-selected configuration left untouched (Section~\ref{ssec:setup}). The probe tests fixed-design robustness on one non-equity series. It is not zero-shot transfer and not cross-asset generalisation. Reading the gold rows of Tables~\ref{tab:results-finance-mae}, \ref{tab:results-finance-rmse}, and~\ref{tab:results-finance-mape}, DSTNet posts an $H=1$ MAPE of 0.582\%, which places gold fifth of the eight series. It sits just above the four strongest indices (DJI 0.481\%, DAX 0.541\%, NYSE 0.568\%, and S\&P~500 0.572\%) and below the remaining three (KOSPI 0.771\%, NASDAQ 0.788\%, and Nikkei~225 0.873\%). Gold also clears persistence at all four horizons, by 0.67\% at $H{=}1$ and 4.09\% at $H{=}10$. The error growth across horizons matches the equity pattern. The gold panel of Fig.~\ref{fig:forecast-traces} shows the same qualitative behaviour as the equity panel, with visible lag around sharp moves at the longer horizons.

\subsection{Regime-Aware and Bias-Aware Diagnostics}\label{ssec:regime}
A model can post a low panel-level MAPE while exhibiting strongly asymmetric directional behaviour. We therefore split the $H=1$ test set two ways, first into low- and high-volatility regimes by rolling realised volatility, and second into up-return and down-return days by the realised next-day return. Table~\ref{tab:regime-h1} reports the result, with MAPE split by volatility regime. Alongside the raw directional accuracy of Table~\ref{tab:baseline-da-h1}, we report DAcal, a validation-bias-adjusted directional diagnostic in percent that subtracts the validation-set mean predicted return from each test prediction before taking the sign. The adjustment never touches the forecasting metrics. Its only job is to reveal whether the raw model leans on a directional prior. With up days as positives, Up DAcal is sensitivity and Down DAcal is specificity, so the final column, Youden's $J = \text{Up DAcal} + \text{Down DAcal} - 100$ in percentage points, equals $2\,\text{BA} - 1$ with BA the balanced accuracy. Any classifier whose predicted sign is constant over the test set has $J=0$, whatever the class balance.

\begin{table*}[!t]
\caption{$H=1$ regime diagnostics.}
\label{tab:regime-h1}
\centering
\small
\setlength{\tabcolsep}{7pt}
\begin{tabular}{@{}lrrrrrrr@{}}
\toprule
\textbf{Series} & \makecell{\textbf{MAPE}\\low-vol} & \makecell{\textbf{MAPE}\\high-vol} &
\makecell{\textbf{DAcal}\\low-vol} & \makecell{\textbf{DAcal}\\high-vol} &
\makecell{\textbf{DAcal}\\up} & \makecell{\textbf{DAcal}\\down} & \textbf{$J$ (pp)} \\
\midrule
KOSPI      & 0.616 & 0.925 & 48.72 & 56.78 & 16.38 & 88.24 & $+4.62$ \\
S\&P 500   & 0.485 & 0.660 & 44.63 & 44.63 & 2.22  & 98.13 & $+0.35$ \\
Nikkei~225 & 0.803 & 0.943 & 52.54 & 50.00 & 51.61 & 50.89 & $+2.50$ \\
DJI        & 0.397 & 0.565 & 38.02 & 46.28 & 9.79  & 88.89 & $-1.32$ \\
DAX        & 0.458 & 0.624 & 52.03 & 47.97 & 21.80 & 83.19 & $+4.99$ \\
NASDAQ     & 0.624 & 0.951 & 38.84 & 51.24 & 14.49 & 85.58 & $+0.07$ \\
NYSE       & 0.497 & 0.638 & 43.80 & 47.11 & 33.58 & 60.95 & $-5.47$ \\
Gold       & 0.484 & 0.680 & 54.55 & 52.89 & 47.69 & 60.71 & $+8.40$ \\
\bottomrule
\end{tabular}
\end{table*}

The volatility split behaves as expected. High-volatility regimes carry larger MAPE on every series, which is not a failure, because volatility inflates the absolute size of the target move itself. The directional columns are where the diagnostic earns its place. After the bias adjustment the equity indices show a pronounced asymmetry: on the S\&P~500 the calibrated diagnostic is 98.13\% on down-return days and 2.22\% on up-return days, and a softer version of the same split repeats on DJI, DAX, NASDAQ, and KOSPI.

That asymmetry is easy to over-read. A constant-sign classifier has $J = 0$ however lopsided the two columns look, so the S\&P~500 pair does not describe a signal that separates down days from up days. It describes a bias-adjusted forecast whose sign is negative on almost every test day, which is necessarily right on 98\% of down days and wrong on 98\% of up days. Read through $J$, the discrimination remaining after the bias adjustment is near zero across the equity panel: $+0.35$ percentage points on the S\&P~500, $+0.07$ on NASDAQ, $+2.50$ on Nikkei~225, $+4.62$ on KOSPI, $+4.99$ on DAX, and negative on DJI ($-1.32$) and NYSE ($-5.47$). Gold is the largest at $+8.40$, and with NYSE it is one of only two series on which the allocation rule of Section~\ref{ssec:trading} changes exposure after the opening trade. This diagnostic anticipates the negative equity allocation result rather than standing apart from it, and it is why the raw directional ordering of Table~\ref{tab:baseline-da-h1} is not read as tradable skill. Each $J$ is a single-sample point estimate with no interval, so it bounds the claim rather than sizing the effect.

\subsection{Ablation Studies}\label{ssec:ablations}
We run four ablations to isolate the contribution of each architectural component. Results are summarised in Table~\ref{tab:ablations}.

\begin{enumerate}
\renewcommand{\labelenumi}{(\alph{enumi})}
\item \textbf{Burn-in exclusion}: training without burn-in trimming versus with burn-in trimming. Burn-in trimming reduces the available training set by 192 windows per index. We report whether this changes test-set performance materially.
\item \textbf{Joint vs factorized STST}: the joint variant flattens the DST tensor to $wS = 160$ tokens of dimension $M = 7$, projects to $d = 64$, and applies a single multi-head attention block over all 160 tokens, followed by the same attention pooling and cross-stream MLP. The joint variant's projected width $d_J = 64$ matches Stream~B while Stream~A uses $d_A = 128$, so the comparison is width-comparable rather than parameter-matched, and the joint variant trains at batch size 16 to keep its attention memory footprint feasible. We read this comparison as a sensitivity analysis of the factorization. At the configured widths and depths, Theorem~\ref{thm:complexity} puts the aggregate quadratic-core ratio between the two variants at roughly $15\times$, with $55.2\times$ reserved for the equal-width per-layer idealisation.
\item \textbf{No-DST}: DST tensor replaced by the static-snapshot spectral-magnitude slice. Tests whether retaining the trajectory history provides empirical forecasting value beyond the final spectral slice.
\item \textbf{No-VAF}: VAF replaced by fixed equal-weight mean fusion ($\alpha_{\text{temp}} = \alpha_{\text{spec}} = 0.5$). This ablation measures the joint value of learned weighting and sample-varying reallocation. It does not separate the two, and Fig.~\ref{fig:vaf-dynamics} shows that their balance differs across series.
\end{enumerate}

\begin{table*}[!t]
\caption{Ablation study on the primary equity benchmark (S\&P 500, Panel A) and the non-equity probe (Gold, Panel B). Test MAE with 95\% moving-block bootstrap CIs in brackets.}
\label{tab:ablations}
\centering
\footnotesize
\setlength{\tabcolsep}{5pt}
\begin{tabular}{@{}lcccc@{}}
\toprule
& \multicolumn{2}{c}{\textbf{Panel A: S\&P 500}} & \multicolumn{2}{c}{\textbf{Panel B: Gold}} \\
\cmidrule(lr){2-3}\cmidrule(lr){4-5}
\textbf{Variant} & \textbf{H=1} & \textbf{H=10} & \textbf{H=1} & \textbf{H=10} \\
\midrule
DSTNet full                               & 26.66 [23.52, 30.64] & 91.94 [76.91, 110.95]  & 11.95 [10.65, 13.78] & 48.04 [41.53, 58.61] \\
\hspace{1em}(a) Without burn-in exclusion & 26.89 [23.69, 30.58] & 93.13 [79.84, 113.18]  & 12.03 [10.64, 13.83] & 48.78 [41.13, 60.00] \\
\hspace{1em}(b) Joint STST                & 26.66 [23.68, 30.46] & 92.62 [77.99, 110.46]  & 11.90 [10.59, 13.77] & 47.79 [41.41, 57.44] \\
\hspace{1em}(c) No-DST (static snapshot)  & 27.98 [24.61, 32.82] & 99.81 [84.85, 123.05]  & 12.42 [11.16, 14.32] & 51.15 [42.83, 62.86] \\
\hspace{1em}(d) No-VAF (mean fusion)      & 27.61 [24.61, 32.24] & 95.06 [82.07, 117.78]  & 12.27 [10.85, 14.09] & 49.26 [42.10, 60.03] \\
\bottomrule
\end{tabular}
\end{table*}

Removing the trajectory produces the largest degradation. No-DST raises panel-mean MAE by 5.79\% (Table~\ref{tab:ablations-appendix}) and is worse than the full model on all eight series at both reported horizons, although its bootstrap intervals overlap those of the full model. Joint STST changes panel-mean MAE by only 0.05\% in point estimate, which indicates little empirical difference at the evaluated sizes, but overlapping intervals do not establish equivalence. Dropping burn-in exclusion and replacing learned fusion by a mean produce intermediate degradations of 0.92\% and 2.94\%, respectively.

Fig.~\ref{fig:vaf-dynamics} shows the fusion weights over the test window on four series chosen to span the observed range of behaviour. The gate is not uniformly dynamic across assets. On KOSPI it remains near equal weighting for roughly the first hundred test observations and then switches sharply toward the temporal branch in episodes that approach the upper bound. Gold oscillates around equal weighting and crosses repeatedly. On NASDAQ the weights remain close to 0.5 throughout most of the window, while on the S\&P~500 they hold a fairly stable tilt of roughly 0.6 to 0.4 toward the temporal branch. Where the gate is close to static, the No-VAF degradation in Table~\ref{tab:ablations} reflects the value of a learned asymmetric weighting rather than clear evidence of strong per-sample reallocation. We therefore do not claim that VAF performs regime switching on every series.

\begin{figure*}[!t]
\centering
\includegraphics[width=\textwidth]{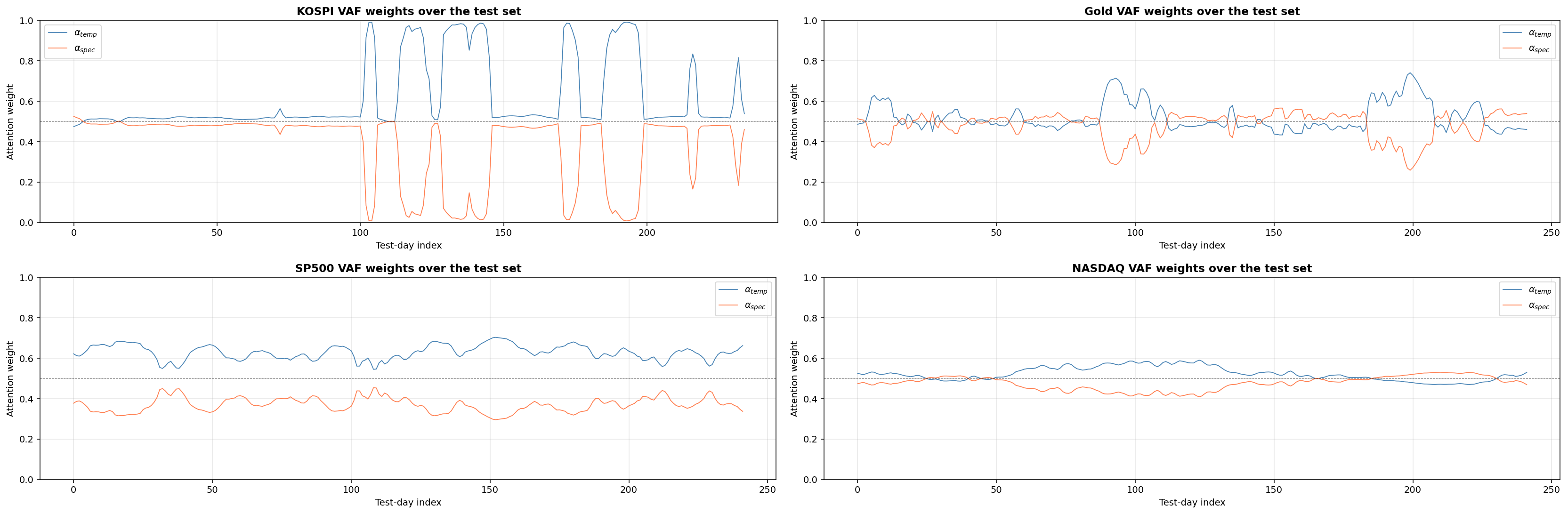}
\caption{Vertical Attention Fusion weights $\alpha_{\text{temp}}$ and
$\alpha_{\text{spec}}$ across the 2023-06-01 to 2024-05-31 hold-out for KOSPI, Gold, S\&P~500, and NASDAQ. The dashed line marks equal weighting. These four series span the range of gate behaviour observed on the evaluation panel, from episodic switching on KOSPI to near-static weighting on NASDAQ.}
\label{fig:vaf-dynamics}
\end{figure*}

\subsubsection{Centred HAD gate variation}
For each horizon we stack the learned gate vectors over the test split, subtract that horizon's sample-mean gate, take the SVD of the centred matrix, and record both the number of right-singular directions holding 95\% of the squared singular-value mass and the principal angles between adjacent horizons. Table~\ref{tab:had-subspace} in Appendix~\ref{app:full-metrics} reports the result: one to four directions per horizon, and mean adjacent angles from $76.89^\circ$ to $84.58^\circ$. This is descriptive only. Compact rank can equally reflect correlated fused inputs, $\tanh$ saturation, or limited test-sample variation, and unrelated low-dimensional directions in $\mathbb{R}^{64}$ are themselves usually near-orthogonal, so neither statistic establishes purposeful horizon specialization. Appendix~\ref{app:ablations-appendix} reports the panel-mean degradation and the per-series ablation results for the remaining six indices.

\subsection{Downstream Allocation Diagnostic: A Negative Result on Equities}
\label{ssec:trading}
Forecasting accuracy does not automatically imply economic value. We therefore use allocation as a falsification-oriented diagnostic, not as a claimed source of alpha and not as a further contribution. It asks whether the out-of-sample forecasts change exposure often enough, and well enough, to improve on static and volatility-scaled benchmarks after costs. Because the bias-adjusted equity signal of Section~\ref{ssec:regime} is almost entirely a negative-sign prior, the equity rule responds only to sufficiently negative forecasts and otherwise holds a fixed base exposure. Every threshold and exposure level is selected on validation data and then frozen, and the test set is not used for calibration. The diagnostic uses an idealised close-to-close execution convention. The signal includes the completed close at day $t$, and the resulting position is treated as effective from that same close. This abstraction omits the latency required to observe the completed bar, construct the input, compute the forecast, and execute the position. The allocation results are therefore interpreted as a diagnostic of forecast-conditioned behaviour rather than implementable trading performance. Algorithm~\ref{alg:trading} states the procedure.

For gold we additionally evaluate a long-short allocation, because gold carries the panel's largest $J$ and a two-sided rule is only defensible where the signal is two-sided. Every strategy pays a flat 10 bps transaction cost on turnover, and every strategy is compared against three fair benchmarks, buy-and-hold, a constant 50\% buy-and-hold, and a volatility-targeted buy-and-hold. Table~\ref{tab:trading-10bps} reports the results on the full eight-series panel. Sharpe is annualised by $\sqrt{252}$. MaxDD, cumulative return, and average exposure are percentages. TotTO is cumulative turnover over the test window (sum of $|\pi_t - \pi_{t-1}|$), and AvgTO is the same quantity per trading day expressed in percent, so a row with TotTO $=1.00$ over 242 test observations reports AvgTO $=0.413$. Boldface marks a descriptive point-estimate improvement over all three benchmarks on the same series and carries no inference or trading-advantage claim.

\begin{table}[!t]
\caption{Downstream Allocation Diagnostic at 10 bps per Unit of Turnover. Strategies: A, Buy-and-Hold. B, 50\% Buy-and-Hold. C, Volatility-Targeted Buy-and-Hold. D, DSTNet Asymmetric Risk-Off. E, DSTNet Long-Only. F, DSTNet Long-Short.}
\label{tab:trading-10bps}
\centering
\footnotesize
\setlength{\tabcolsep}{2.6pt}
\renewcommand{\arraystretch}{0.98}
\begin{adjustbox}{max width=\columnwidth}
\begin{tabular}{@{}llrrrrrr@{}}
\toprule
\textbf{Series} & \textbf{Str.} & \textbf{Sharpe} & \textbf{MaxDD} & \textbf{CumRet} & \textbf{Exp.} & \textbf{TotTO} & \textbf{AvgTO} \\
\midrule
\multirow{4}{*}{KOSPI}
& A & 0.48 & 14.59 & 6.37 & 100.00 & 1.00 & 0.426 \\
& B & 0.48 & 7.51 & 3.46 & 50.00 & 0.50 & 0.213 \\
& C & 0.28 & 12.17 & 2.59 & 79.33 & 6.30 & 2.681 \\
& D & 0.48 & 9.68 & 4.39 & 65.00 & 0.65 & 0.277 \\
\midrule
\multirow{4}{*}{S\&P 500}
& A & 2.15 & 10.28 & 26.09 & 100.00 & 1.00 & 0.413 \\
& B & 2.15 & 5.23 & 12.47 & 50.00 & 0.50 & 0.207 \\
& C & 1.95 & 10.03 & 22.21 & 95.80 & 3.40 & 1.406 \\
& D & 2.15 & 6.76 & 16.44 & 65.00 & 0.65 & 0.269 \\
\midrule
\multirow{4}{*}{Nikkei 225}
& A & 1.40 & 9.56 & 24.11 & 100.00 & 1.00 & 0.424 \\
& B & 1.40 & 4.82 & 11.81 & 50.00 & 0.50 & 0.212 \\
& C & 1.28 & 7.67 & 15.24 & 69.42 & 6.19 & 2.624 \\
& D & 1.40 & 6.25 & 15.46 & 65.00 & 0.65 & 0.275 \\
\midrule
\multirow{4}{*}{DJI}
& A & 2.03 & 9.02 & 20.66 & 100.00 & 1.00 & 0.413 \\
& B & 2.03 & 4.59 & 9.97 & 50.00 & 0.50 & 0.207 \\
& C & 1.95 & 9.02 & 19.25 & 98.48 & 1.60 & 0.659 \\
& D & 2.03 & 5.93 & 13.10 & 65.00 & 0.65 & 0.269 \\
\midrule
\multirow{4}{*}{DAX}
& A & 1.56 & 10.82 & 17.67 & 100.00 & 1.00 & 0.407 \\
& B & 1.56 & 5.52 & 8.64 & 50.00 & 0.50 & 0.203 \\
& C & 1.54 & 9.82 & 16.49 & 95.33 & 3.50 & 1.422 \\
& D & 1.56 & 7.14 & 11.31 & 65.00 & 0.65 & 0.264 \\
\midrule
\multirow{4}{*}{NASDAQ}
& A & 1.72 & 12.27 & 28.49 & 100.00 & 1.00 & 0.413 \\
& B & 1.72 & 6.24 & 13.70 & 50.00 & 0.50 & 0.207 \\
& C & 1.37 & 9.47 & 16.70 & 76.60 & 7.12 & 2.941 \\
& D & 1.72 & 8.08 & 18.02 & 65.00 & 0.65 & 0.269 \\
\midrule
\multirow{4}{*}{NYSE}
& A & 2.02 & 10.66 & 22.44 & 100.00 & 1.00 & 0.413 \\
& B & 2.02 & 5.44 & 10.81 & 50.00 & 0.50 & 0.207 \\
& C & 1.88 & 10.40 & 19.80 & 97.02 & 2.20 & 0.908 \\
& D & 1.80 & 5.65 & 11.31 & 56.77 & 12.21 & 5.045 \\
\midrule
\multirow{5}{*}{Gold}
& A & 1.74 & 8.16 & 21.88 & 100.00 & 1.00 & 0.413 \\
& B & 1.74 & 4.13 & 10.60 & 50.00 & 0.50 & 0.207 \\
& C & 1.71 & 8.14 & 20.00 & 93.36 & 4.10 & 1.696 \\
& E & \textbf{1.99} & \textbf{3.34} & 12.51 & 47.83 & 6.11 & 2.523 \\
& F & \textbf{1.80} & \textbf{0.58} & 2.11 & 7.65 & 4.76 & 1.965 \\
\bottomrule
\end{tabular}
\end{adjustbox}
\end{table}

The turnover columns show where the gate is active. On KOSPI, S\&P~500, Nikkei~225, DJI, DAX, and NASDAQ, DSTNet's total turnover sits at 0.65, a single opening trade and nothing after it, essentially matching the near-zero-adjustment character of the buy-and-hold and 50\% buy-and-hold rows above them and standing in clear contrast to the vol-target baseline's turnover of 1.60 to 7.12 on those same series. NYSE and both gold strategies are the opposite case, with turnover an order of magnitude higher (4.76 to 12.21), which is the table's own evidence that the risk-off or risk-on gate is active there and only there.

Gold is the only series on which the long-only diagnostic improves the reported point-estimate Sharpe over all three benchmarks. At 10 bps its Sharpe is 1.99, compared with 1.74 for buy-and-hold and 1.71 for the volatility-targeted variant, while its maximum drawdown is 3.34\%, compared with 8.16\% for buy-and-hold. The long-short variant reduces maximum drawdown further, to 0.58\%, but its mean exposure is only 7.65\% and its cumulative return falls to 2.11\%. These are descriptive outcomes from one series and one hold-out period, without paired inference. They motivate further testing but do not establish an allocation advantage.

The equity result is negative. DSTNet improves on buy-and-hold Sharpe on none of the seven indices. On six of them the gate never fires after the opening trade, so the strategy is a static 65\% position and its Sharpe equality with buy-and-hold is a scaling identity rather than risk management. NYSE is the only equity series with active changes, where exposure falls to 56.77\% and drawdown from 10.66\% to 5.65\%, but Sharpe falls from 2.02 to 1.80. We keep the diagnostic because the failed translation from price error to economic utility is the informative part, and because reporting it stops the forecasting results being read as tradable alpha.
Fig.~\ref{fig:allocation-curves} in Appendix~\ref{app:full-metrics} contrasts the two outcomes as growth curves.

\subsection{Limitations and Threats to Validity}\label{ssec:limitations}
Several limitations bound the conclusions of this study. First, one hold-out year supplies only about two dozen disjoint ten-day target blocks (Section~\ref{ssec:metrics}). The moving-block bootstrap preserves local dependence but cannot manufacture additional market regimes, so overlapping marginal intervals lead us to withhold a statistical-win claim at $H{=}10$ without treating the overlap as evidence of equivalence.

Second, three seeds do not characterise the full training distribution, and only DSTNet is rerun across them. Baseline estimates stay fixed at their primary runs, so Table~\ref{tab:multiseed-summary} measures DSTNet initialisation sensitivity against fixed references rather than a seed-matched distribution of model differences.

Third, paired inference is confined to the one-day horizon, where the persistence comparison does not reject equal predictive accuracy. The $H\in\{3,5,10\}$ improvements, which are the largest, therefore remain descriptive. The $H{=}1$ tests use a lag-zero variance estimator and unadjusted $p$-values (Section~\ref{ssec:metrics}), so residual dependence from overlapping input windows and multiplicity both remain, and a test for superior predictive ability~\cite{hansen2005spa} would be the appropriate instrument for the within-series comparison set. Directional accuracy is reported without a Pesaran--Timmermann test~\cite{pesaran1992simple} or an interval.

The persistence comparison is also metric-dependent. The claim that the learned baselines fail to beat the random walk holds under MAE and MAPE (Section~\ref{ssec:results}), not under squared loss at long horizons.

Fourth, the factorized attention path is a strictly smaller hypothesis class than joint attention. The joint-STST comparison of Table~\ref{tab:ablations} is a sensitivity analysis of that restriction, and it yields similar point estimates at our sizes, which is not an equivalence result. The efficiency argument matters most in the regime where $wS$ grows.

Fifth, the current ablations isolate DST history, VAF, burn-in exclusion, and the factorized-versus-joint attention choice, but they do not independently isolate the causal Symlet-4 denoiser or HAD. The denoiser is part of the common DSTNet preprocessing pipeline and is not ablated against an otherwise identical raw-input configuration, and no otherwise identical ungated multi-horizon decoder is evaluated. The reported gains therefore support the complete DSTNet pipeline and the components directly tested by the ablation study. No separate accuracy gain is attributed to denoising or horizon gating alone.

Sixth, VAF mixes the branches before HAD sees them, so horizon-specific branch reliance is not identifiable (Section~\ref{ssec:had}).

Finally, the panel is fixed rather than searched. The seven equity indices cover Asia, Europe, and North America, and no scale set, lookback, width, or decoder is tuned per asset. A single non-equity series cannot establish cross-asset generalisation. The allocation diagnostic of Section~\ref{ssec:trading} uses validation-selected thresholds evaluated on one hold-out period, simplified transaction costs, and an idealised close-to-close execution convention. The \texttt{GC=F} proxy further abstracts from contract rolls, margin requirements, roll yield, market impact, and capacity. Its negative equity result and isolated gold outcome do not establish a general trading advantage.

\section{Conclusion}
\label{sec:conclusion}
We presented DSTNet, a dual-branch architecture for causal multi-horizon financial forecasting. It carries the evolution of filter-bank magnitudes as a per-sample Dynamic Spectral Trajectory tensor, encodes that tensor with a factorized Scale-Temporal Spectral Transformer, and emits four horizons in one forward pass. The front end is causal end to end: Proposition~\ref{thm:causality} shows the Morlet-derived FIR bank reads no observation after the forecast origin, and Corollary~\ref{lem:burnin} fixes the 192-sample exclusion that removes the left-boundary transient.

The comparison that decides the paper is against the random walk. Under MAE and MAPE none of the nine learned baselines beats persistence at one, three, or five days on any series, while DSTNet is below it in every cell, by 0.7 to 0.9 percent at one day and 3.4 to 4.5 percent at ten days. That margin is a point estimate. The one-day paired test against persistence is inconclusive, the three-seed rerun keeps DSTNet ahead of the strongest fixed competitor in 25 or 26 of 32 cells rather than all 32, and the random-walk claim does not extend to RMSE at long horizons. Gold shows the same qualitative forecasting pattern.

The allocation diagnostic does not support an equity-timing claim, and the regime diagnostics say why. Once the validation bias is removed, the equity directional signal keeps almost no conditional discrimination, so the rule holds a static scaled position on six of seven indices and reduces drawdown at the cost of Sharpe on NYSE. The gold improvement rests on one hold-out without paired inference.

The immediate next step is extending the persistence comparison to $H\in\{3,5,10\}$, which the stored per-origin errors already permit without retraining. Beyond that: rolling origins across more assets and periods, more seeds with the stochastic baselines rerun under the same set, an ungated-decoder ablation to test whether horizon gating earns its place, and any further allocation work with realistic futures implementation and turnover-dependent costs.

\section*{Data and Code Availability}
All eight daily OHLCV series are publicly available from Yahoo Finance and were retrieved programmatically using the ticker symbols and date ranges listed in Table~\ref{tab:data}. The preprocessing, training, and evaluation code will be made publicly available upon acceptance.

\bibliographystyle{IEEEtran}
\bibliography{references}

@article{hochreiter1997,
  author    = {Sepp Hochreiter and J{\"u}rgen Schmidhuber},
  title     = {Long short-term memory},
  journal   = {Neural Computation},
  volume    = {9},
  number    = {8},
  pages     = {1735--1780},
  year      = {1997},
  publisher = {MIT Press}
}

@article{mallat1989,
  author    = {St{\'e}phane G. Mallat},
  title     = {A theory for multiresolution signal decomposition: the wavelet representation},
  journal   = {IEEE Transactions on Pattern Analysis and Machine Intelligence},
  volume    = {11},
  number    = {7},
  pages     = {674--693},
  year      = {1989}
}

@book{daubechies1992,
  author    = {Ingrid Daubechies},
  title     = {Ten Lectures on Wavelets},
  series    = {CBMS-NSF Regional Conference Series in Applied Mathematics},
  volume    = {61},
  publisher = {SIAM},
  year      = {1992}
}

@article{donoho1994ideal,
  author    = {David L. Donoho and Iain M. Johnstone},
  title     = {Ideal spatial adaptation by wavelet shrinkage},
  journal   = {Biometrika},
  volume    = {81},
  number    = {3},
  pages     = {425--455},
  year      = {1994}
}

@article{donoho1995,
  author    = {David L. Donoho},
  title     = {De-noising by soft-thresholding},
  journal   = {IEEE Transactions on Information Theory},
  volume    = {41},
  number    = {3},
  pages     = {613--627},
  year      = {1995}
}

@article{torrence1998,
  author    = {Christopher Torrence and Gilbert P. Compo},
  title     = {A practical guide to wavelet analysis},
  journal   = {Bulletin of the American Meteorological Society},
  volume    = {79},
  number    = {1},
  pages     = {61--78},
  year      = {1998}
}

@article{minu2010,
  author    = {K. K. Minu and Manesh C. Lineesh and C. Jessy John},
  title     = {Wavelet neural networks for nonlinear time series analysis},
  journal   = {Applied Mathematical Sciences},
  volume    = {4},
  number    = {50},
  pages     = {2485--2495},
  year      = {2010}
}

@article{sezer2020,
  author    = {Omer Berat Sezer and Mehmet Ugur Gudelek and Ahmet Murat Ozbayoglu},
  title     = {Financial time series forecasting with deep learning: A systematic literature
               review: 2005--2019},
  journal   = {Applied Soft Computing},
  volume    = {90},
  pages     = {106181},
  year      = {2020}
}

@book{cover2006,
  author    = {Thomas M. Cover and Joy A. Thomas},
  title     = {Elements of Information Theory},
  publisher = {Wiley-Interscience},
  edition   = {2nd},
  year      = {2006}
}

@article{kunsch1989,
  author    = {Hans R. K{\"u}nsch},
  title     = {The jackknife and the bootstrap for general stationary observations},
  journal   = {The Annals of Statistics},
  volume    = {17},
  number    = {3},
  pages     = {1217--1241},
  year      = {1989}
}

@inproceedings{chen2016xgboost,
  author    = {Tianqi Chen and Carlos Guestrin},
  title     = {{XGBoost}: A scalable tree boosting system},
  booktitle = {Proc. 22nd ACM SIGKDD International Conference on Knowledge Discovery
               and Data Mining (KDD)},
  pages     = {785--794},
  year      = {2016}
}

@inproceedings{vaswani2017,
  author    = {Ashish Vaswani and Noam Shazeer and Niki Parmar and Jakob Uszkoreit and
               Llion Jones and Aidan N. Gomez and {\L}ukasz Kaiser and Illia Polosukhin},
  title     = {Attention is all you need},
  booktitle = {Advances in Neural Information Processing Systems (NeurIPS)},
  volume    = {30},
  year      = {2017}
}

@inproceedings{informer2021,
  author    = {Haoyi Zhou and Shanghang Zhang and Jieqi Peng and Shuai Zhang and
               Jianxin Li and Hui Xiong and Wancai Zhang},
  title     = {{Informer}: Beyond efficient transformer for long sequence time-series
               forecasting},
  booktitle = {Proc. AAAI Conference on Artificial Intelligence},
  volume    = {35},
  number    = {12},
  pages     = {11106--11115},
  year      = {2021}
}

@inproceedings{autoformer2021,
  author    = {Haixu Wu and Jiehui Xu and Jianmin Wang and Mingsheng Long},
  title     = {{Autoformer}: Decomposition transformers with auto-correlation for
               long-term series forecasting},
  booktitle = {Advances in Neural Information Processing Systems (NeurIPS)},
  volume    = {34},
  pages     = {22419--22430},
  year      = {2021}
}

@inproceedings{zhou2022fedformer,
  title={Fedformer: Frequency enhanced decomposed transformer for long-term series forecasting},
  author={Zhou, Tian and Ma, Ziqing and Wen, Qingsong and Wang, Xue and Sun, Liang and Jin, Rong},
  booktitle={International conference on machine learning},
  pages={27268--27286},
  year={2022},
  organization={PMLR}
}

@inproceedings{patchtst2023,
  author    = {Yuqi Nie and Nam H. Nguyen and Phanwadee Sinthong and Jayant Kalagnanam},
  title     = {A time series is worth 64 words: Long-term forecasting with transformers},
  booktitle = {Proc. International Conference on Learning Representations (ICLR)},
  year      = {2023}
}

@inproceedings{itransformer2024,
  author    = {Yong Liu and Tengge Hu and Haoran Zhang and Haixu Wu and Shiyu Wang and
               Lintao Ma and Mingsheng Long},
  title     = {{iTransformer}: Inverted transformers are effective for time series
               forecasting},
  booktitle = {Proc. International Conference on Learning Representations (ICLR)},
  year      = {2024},
  note      = {Spotlight}
}

@inproceedings{timemixer2024,
  author    = {Shiyu Wang and Haixu Wu and Xiaoming Shi and Tengge Hu and Huakun Luo
               and Lintao Ma and James Y. Zhang and Jun Zhou},
  title     = {{TimeMixer}: Decomposable multiscale mixing for time series forecasting},
  booktitle = {Proc. International Conference on Learning Representations (ICLR)},
  year      = {2024}
}

@inproceedings{master2024,
  author    = {Tong Li and Zhaoyang Liu and Yanyan Shen and Xue Wang and Haokun Chen
               and Sen Huang},
  title     = {{MASTER}: Market-guided stock transformer for stock price forecasting},
  booktitle = {Proc. AAAI Conference on Artificial Intelligence},
  volume    = {38},
  number    = {1},
  pages     = {162--170},
  year      = {2024}
}

@inproceedings{nbeats2020,
  author    = {Boris N. Oreshkin and Dmitri Carpov and Nicolas Chapados and Yoshua Bengio},
  title     = {{N-BEATS}: Neural basis expansion analysis for interpretable time series
               forecasting},
  booktitle = {Proc. International Conference on Learning Representations (ICLR)},
  year      = {2020}
}

@article{lim2021tft,
  author    = {Bryan Lim and Sercan {\"O}. Ar{\i}k and Nicolas Loeff and Tomas Pfister},
  title     = {Temporal fusion transformers for interpretable multi-horizon time series
               forecasting},
  journal   = {International Journal of Forecasting},
  volume    = {37},
  number    = {4},
  pages     = {1748--1764},
  year      = {2021}
}

@inproceedings{zeng2023dlinear,
  author    = {Ailing Zeng and Muxi Chen and Lei Zhang and Qiang Xu},
  title     = {Are transformers effective for time series forecasting?},
  booktitle = {Proc. AAAI Conference on Artificial Intelligence},
  volume    = {37},
  number    = {9},
  pages     = {11121--11128},
  year      = {2023}
}

@article{hansen2005spa,
   author={Hansen, Peter R.}, title={A Test for Superior Predictive Ability},
   journal={Journal of Business \& Economic Statistics}, volume={23},
   number={4}, pages={365--380}, year={2005}}

@article{wang2025multifactor,
  title={Multifactor prediction model for stock market analysis based on deep learning techniques},
  author={Wang, Kangyi},
  journal={Scientific Reports},
  volume={15},
  number={1},
  pages={5121},
  year={2025},
  publisher={Nature Publishing Group UK London}
}

@article{beniwal2024forecasting,
  title={Forecasting multistep daily stock prices for long-term investment decisions: A study of deep learning models on global indices},
  author={Beniwal, Mohit and Singh, Archana and Kumar, Nand},
  journal={Engineering Applications of Artificial Intelligence},
  volume={129},
  pages={107617},
  year={2024},
  publisher={Elsevier}
}

@article{diebold2002comparing,
  title={Comparing predictive accuracy},
  author={Diebold, Francis X and Mariano, Robert S},
  journal={Journal of Business \& economic statistics},
  volume={20},
  number={1},
  pages={134--144},
  year={2002},
  publisher={Taylor \& Francis}
}

@article{tian2025bimt,
  title={BiMT-TCN: A cutting-edge hybrid model for enhanced stock price prediction},
  author={Tian, Guangyang and Huang, Tingwen and Peng, Chengyu and Yang, Yin and Wen, Shiping},
  journal={Knowledge-Based Systems},
  pages={114263},
  year={2025},
  publisher={Elsevier}
}

@article{ji2024galformer,
  title={Galformer: a transformer with generative decoding and a hybrid loss function for multi-step stock market index prediction},
  author={Ji, Yi and Luo, Yuxuan and Lu, Aixia and Xia, Duanyang and Yang, Lixia and Wee-Chung Liew, Alan},
  journal={Scientific Reports},
  volume={14},
  number={1},
  pages={23762},
  year={2024},
  publisher={Nature Publishing Group UK London}
}

@article{rezaei2021stock,
  title={Stock price prediction using deep learning and frequency decomposition},
  author={Rezaei, Hadi and Faaljou, Hamidreza and Mansourfar, Gholamreza},
  journal={Expert Systems with Applications},
  volume={169},
  pages={114332},
  year={2021},
  publisher={Elsevier}
}

@article{thakkar2022information,
  title={Information fusion-based genetic algorithm with long short-term memory for stock price and trend prediction},
  author={Thakkar, Ankit and Chaudhari, Kinjal},
  journal={Applied Soft Computing},
  volume={128},
  pages={109428},
  year={2022},
  publisher={Elsevier}
}

@article{bhandari2022predicting,
  title={Predicting stock market index using LSTM},
  author={Bhandari, Hum Nath and Rimal, Binod and Pokhrel, Nawa Raj and Rimal, Ramchandra and Dahal, Keshab R and Khatri, Rajendra KC},
  journal={Machine Learning with Applications},
  volume={9},
  pages={100320},
  year={2022},
  publisher={Elsevier}
}

@article{ge2025enhancing,
  title={Enhancing stock market Forecasting: A hybrid model for accurate prediction of S\&P 500 and CSI 300 future prices},
  author={Ge, Qing},
  journal={Expert Systems with Applications},
  volume={260},
  pages={125380},
  year={2025},
  publisher={Elsevier}
}

@article{yu2025intelligent,
  title={An intelligent framework based on optimized variational mode decomposition and temporal convolutional network: Applications to stock index multi-step forecasting},
  author={Yu, Yuanyuan and Dai, Dongsheng and Yang, Qu and Zeng, Qing and Lin, Yu and Chen, Yanxiang},
  journal={Expert Systems with Applications},
  volume={268},
  pages={126222},
  year={2025},
  publisher={Elsevier}
}

@article{gong2024predicting,
  title={Predicting the highest and lowest stock price indices: A combined BiLSTM-SAM-TCN deep learning model based on re-decomposition},
  author={Gong, Hao and Xing, Haiyang},
  journal={Applied Soft Computing},
  volume={167},
  pages={112393},
  year={2024},
  publisher={Elsevier}
}

@article{li2026importance,
  title={The importance of data noise reduction--wavelet transformation in stock price forecasting with Bidirectional Long Short-Term Memory network},
  author={Li, Honglei and Marudhasalam, Yukesh and Huang, Zhao and Yu, Jiuhong},
  journal={Engineering Applications of Artificial Intelligence},
  volume={165},
  pages={113390},
  year={2026},
  publisher={Elsevier}
}

@article{zubair2026enhanced,
  title={Enhanced financial market forecasting using a hybrid deep learning prediction model with encoder-decoder architecture},
  author={Zubair, Muhammad and Huang, Zhensheng},
  journal={Engineering Applications of Artificial Intelligence},
  volume={166},
  pages={113565},
  year={2026},
  publisher={Elsevier}
}

@misc{genccay2002introduction,
  title={An introduction to wavelets and other filtering methods in finance and economics},
  author={Gen{\c{c}}ay, Ramazan and Sel{\c{c}}uk, Faruk and Whitcher, Brandon},
  year={2002},
  publisher={IOP Publishing}
}

@article{ang2012regime,
  title={Regime changes and financial markets},
  author={Ang, Andrew and Timmermann, Allan},
  journal={Annu. Rev. Financ. Econ.},
  volume={4},
  number={1},
  pages={313--337},
  year={2012},
  publisher={Annual Reviews}
}

@article{pesaran1992simple,
  title={A simple nonparametric test of predictive performance},
  author={Pesaran, M Hashem and Timmermann, Allan},
  journal={Journal of Business \& Economic Statistics},
  volume={10},
  number={4},
  pages={461--465},
  year={1992},
  publisher={Taylor \& Francis}
}

@article{ziolkowski2026enhancing,
  title={Enhancing financial time series forecasting: a comparative study of discrete wavelet transform and LSTM models for selected global indices: K. Zi{\'o}{\l}kowski, P. Wieczy{\'n}ski},
  author={Zi{\'o}{\l}kowski, Krzysztof and Wieczy{\'n}ski, Pawe{\l}},
  journal={Quality \& Quantity},
  volume={60},
  number={1},
  pages={1995--2015},
  year={2026},
  publisher={Springer}
}

@article{yi2023frequency,
  title={Frequency-domain MLPs are more effective learners in time series forecasting},
  author={Yi, Kun and Zhang, Qi and Fan, Wei and Wang, Shoujin and Wang, Pengyang and He, Hui and An, Ning and Lian, Defu and Cao, Longbing and Niu, Zhendong},
  journal={Advances in neural information processing systems},
  volume={36},
  pages={76656--76679},
  year={2023}
}

@article{yi2024filternet,
  title={Filternet: Harnessing frequency filters for time series forecasting},
  author={Yi, Kun and Fei, Jingru and Zhang, Qi and He, Hui and Hao, Shufeng and Lian, Defu and Fan, Wei},
  journal={Advances in neural information processing systems},
  volume={37},
  pages={55115--55140},
  year={2024}
}

@inproceedings{challu2023nhits,
  title={Nhits: Neural hierarchical interpolation for time series forecasting},
  author={Challu, Cristian and Olivares, Kin G and Oreshkin, Boris N and Ramirez, Federico Garza and Canseco, Max Mergenthaler and Dubrawski, Artur},
  booktitle={Proceedings of the AAAI conference on artificial intelligence},
  volume={37},
  number={6},
  pages={6989--6997},
  year={2023}
}

@inproceedings{zhang2024multi,
  title={Multi-resolution time-series transformer for long-term forecasting},
  author={Zhang, Yitian and Ma, Liheng and Pal, Soumyasundar and Zhang, Yingxue and Coates, Mark},
  booktitle={International conference on artificial intelligence and statistics},
  pages={4222--4230},
  year={2024},
  organization={PMLR}
}

@inproceedings{chen2024pathformer,
  title={Pathformer: Multi-scale transformers with adaptive pathways for time series forecasting},
  author={Chen, Peng and Zhang, Yingying and Cheng, Yunyao and Shu, Yang and Wang, Yihang and Wen, Qingsong and Yang, Bin and Guo, Chenjuan},
  booktitle={International conference on learning representations},
  volume={2024},
  pages={10527--10545},
  year={2024}
}

@inproceedings{zhang2025waveletmixer,
  title={WaveletMixer: a multi-resolution wavelets based MLP-mixer for multivariate long-term time series forecasting},
  author={Zhang, Zichi and Pham, Tuan Dung and An, Yimeng and Doan, Ngoc Phu and Alsharari, Majed and Tran, Viet-Hung and Hoang, Anh-Tuan and Vandierendonck, Hans and Mai, Son T},
  booktitle={Proceedings of the AAAI Conference on Artificial Intelligence},
  volume={39},
  number={21},
  pages={22741--22749},
  year={2025}
}

@article{liu2022non,
  title={Non-stationary transformers: Exploring the stationarity in time series forecasting},
  author={Liu, Yong and Wu, Haixu and Wang, Jianmin and Long, Mingsheng},
  journal={Advances in neural information processing systems},
  volume={35},
  pages={9881--9893},
  year={2022}
}

@inproceedings{ma2024u,
  title={U-mixer: An unet-mixer architecture with stationarity correction for time series forecasting},
  author={Ma, Xiang and Li, Xuemei and Fang, Lexin and Zhao, Tianlong and Zhang, Caiming},
  booktitle={Proceedings of the AAAI conference on artificial intelligence},
  volume={38},
  number={13},
  pages={14255--14262},
  year={2024}
}

@article{fischer2018deep,
  title={Deep learning with long short-term memory networks for financial market predictions},
  author={Fischer, Thomas and Krauss, Christopher},
  journal={European journal of operational research},
  volume={270},
  number={2},
  pages={654--669},
  year={2018},
  publisher={Elsevier}
}

@article{cheng2022financial,
  title={Financial time series forecasting with multi-modality graph neural network},
  author={Cheng, Dawei and Yang, Fangzhou and Xiang, Sheng and Liu, Jin},
  journal={Pattern Recognition},
  volume={121},
  pages={108218},
  year={2022},
  publisher={Elsevier}
}

@article{gu2020empirical,
  title={Empirical asset pricing via machine learning},
  author={Gu, Shihao and Kelly, Bryan and Xiu, Dacheng},
  journal={The review of financial studies},
  volume={33},
  number={5},
  pages={2223--2273},
  year={2020},
  publisher={Oxford University Press}
}

@article{baur2010gold,
  title={Is gold a safe haven? International evidence},
  author={Baur, Dirk G and McDermott, Thomas K},
  journal={Journal of banking \& finance},
  volume={34},
  number={8},
  pages={1886--1898},
  year={2010},
  publisher={Elsevier}
}

@article{harvey1997testing,
  title={Testing the equality of prediction mean squared errors},
  author={Harvey, David and Leybourne, Stephen and Newbold, Paul},
  journal={International Journal of forecasting},
  volume={13},
  number={2},
  pages={281--291},
  year={1997},
  publisher={Elsevier}
}

@article{bohra2026vertifusex,
  title={VertiFuseX: Generalizable Financial Forecasting via Multi-Stream Temporal Fusion},
  author={Bohra, Aashish and Vijay, Vivek},
  journal={arXiv preprint arXiv:2609.12793},
  year={2026}
}

@article{bohra2026wavefuse,
  title={WaVeFuse: Regime-Adaptive Equity Index Forecasting via Channel-Wise Wavelet Denoising and Vertical Attention Fusion},
  author={Bohra, Aashish and Vijay, Vivek},
  journal={arXiv preprint arXiv:2609.14733},
  year={2026}
}

@article{corsi2009simple,
  title={A simple approximate long-memory model of realized volatility},
  author={Corsi, Fulvio},
  journal={Journal of financial econometrics},
  volume={7},
  number={2},
  pages={174--196},
  year={2009},
  publisher={Oxford University Press}
}

@article{rua2009international,
  title={International comovement of stock market returns: A wavelet analysis},
  author={Rua, Ant{\'o}nio and Nunes, Lu{\'\i}s C},
  journal={Journal of Empirical Finance},
  volume={16},
  number={4},
  pages={632--639},
  year={2009},
  publisher={Elsevier}
}
\appendices

\section{Algorithms}
\label{app:algorithms}
We collect the seven algorithms that, together with the equations of Section~\ref{sec:methodology}, allow a careful reader to reproduce DSTNet without the codebase. Notation matches the main text. Output shapes are stated after each major step.

\begin{algorithm*}[!t]
\caption{Causal preprocessing and residual denoising (per series)}
\label{alg:preproc}
\begin{algorithmic}[1]
\Require Raw OHLCV $\mathbf{X}\in\mathbb{R}^{T\times5}$, causal approximation filters $g_\ell$, level $L=2$, training index set $\mathcal{T}_{\mathrm{tr}}$
\Ensure Causally denoised standardised indicator matrix $\mathbf{Z} \in \mathbb{R}^{T \times 7}$
\Statex \textit{$\operatorname{CausalFIR}(f,x)$ convolves $x$ with $f$ using left zero-padding only, so output index $t$ reads inputs at indices not exceeding $t$.}
\For{each channel $c \in \{O, H, L, C, V\}$}
  \State $\mathbf{a}_0 \gets \mathbf{X}[:, c]$
  \For{$\ell \in \{1, \ldots, L\}$}
    \State $\mathbf{a}_\ell \gets \operatorname{CausalFIR}(g_\ell,\,\mathbf{a}_{\ell-1})$ \Comment{causal approximation}
    \State $\mathbf{d}_\ell \gets \mathbf{a}_{\ell-1}-\mathbf{a}_\ell$ \Comment{residual detail, exact additive identity}
  \EndFor
  \For{$\ell \in \{1, \ldots, L\}$}
    \State $\sigma_\ell \gets \operatorname{median}\!\big(|\mathbf{d}_\ell[\mathcal{T}_{\text{tr}}] - \operatorname{median}(\mathbf{d}_\ell[\mathcal{T}_{\text{tr}}])|\big)/0.6745$ \Comment{train split only}
    \State $\tau_\ell \gets \sigma_\ell \sqrt{2 \log |\mathcal{T}_{\text{tr}}|}$ \Comment{fit once, then frozen}
    \State $\mathbf{d}_\ell \gets \operatorname{sign}(\mathbf{d}_\ell) \cdot \max(|\mathbf{d}_\ell| - \tau_\ell, 0)$
  \EndFor
  \State $\widetilde{\mathbf{X}}[:, c] \gets \mathbf{a}_L + \sum_{\ell=1}^{L} \mathbf{d}_\ell$ \Comment{additive reconstruction, causal}
\EndFor
\State Compute seven trailing-window indicators on $\widetilde{\mathbf{X}}$: RSI-10, Stoch \%K, CCI-14, OBV, ATR-14, Williams \%R, ROC-12. Stack into $\mathbf{Z}_{\text{raw}} \in \mathbb{R}^{T \times 7}$.
\State Z-score each column of $\mathbf{Z}_{\text{raw}}$ using train-split mean and standard deviation. Output $\mathbf{Z} \in \mathbb{R}^{T \times 7}$.
\end{algorithmic}
\end{algorithm*}

\begin{algorithm*}[!t]
\caption{Causal channel-wise Morlet-derived complex FIR filter bank}
\label{alg:cwwt}
\begin{algorithmic}[1]
\Require $\mathbf{Z} \in \mathbb{R}^{T \times 7}$, scale set $\mathcal{S} = \{1,2,3,5,8,13,21,32\}$, $\omega_0 = 6$
\Ensure $\mathbf{W} \in \mathbb{R}^{T \times 8 \times 7}_{\geq 0}$ \Comment{burn-in $B = 192$ marked at sequence-build time}
\For{$j \in \{1, \ldots, 8\}$ with $s = \mathcal{S}_j$}
  \State $K_s \gets \lfloor 6 s \rfloor$
  \State $h_s[k] \gets \frac{\pi^{-1/4}}{\sqrt{s}} e^{i \omega_0 k / s} e^{-k^2/(2 s^2)}$ for $k \in \{0, \ldots, K_s\}$
  \For{$m \in \{1, \ldots, 7\}$}
    \State $\mathbf{c}_{s,m} \gets \texttt{lfilter}(h_s,[1],\mathbf{Z}[:,m])$ \Comment{$c_{s,m}[t]=\sum_{k=0}^{K_s}h_s[k]Z[t-k,m]$}
    \State $\mathbf{W}[:,j,m] \gets |\mathbf{c}_{s,m}|$ \Comment{coefficient magnitude, causal by Proposition~\ref{thm:causality}}
  \EndFor
\EndFor
\end{algorithmic}
\end{algorithm*}

\begin{algorithm*}[!t]
\caption{DST sequence builder (with burn-in exclusion)}
\label{alg:dst-builder}
\begin{algorithmic}[1]
\Require $\widetilde{\mathbf{X}} \in \mathbb{R}^{T \times 5}$, $\mathbf{W} \in \mathbb{R}^{T \times 8 \times 7}$, adjusted close $\mathbf{P} \in \mathbb{R}^T$, $w = 20$, $\mathcal{H} = \{1, 3, 5, 10\}$, $B = 192$
\Ensure $(\mathbf{X}_{\text{temp}}, \mathbf{D}, \mathbf{Y}) \in \mathbb{R}^{N \times w \times 5} \times \mathbb{R}^{N \times w \times 8 \times 7} \times \mathbb{R}^{N \times |\mathcal{H}|}$
\State $i_{\min} \gets w + B$ \Comment{Corollary~\ref{lem:burnin}}
\State $i_{\max} \gets T - \max(\mathcal{H}) - 1$ \Comment{zero-based indexing, ensures $t_i + h_{\max} \leq T-1$}
\State Initialise empty lists $\mathcal{X}_T, \mathcal{X}_D, \mathcal{Y}$
\For{$i \in \{i_{\min}, \ldots, i_{\max}\}$}
  \State Append $\widetilde{\mathbf{X}}[i - w : i, :]$ to $\mathcal{X}_T$ \Comment{shape $(w, 5)$}
  \State Append $\mathbf{W}[i - w : i, :, :]$ to $\mathcal{X}_D$ \Comment{shape $(w, 8, 7)$}
  \State $t_i \gets i - 1$ \Comment{forecast origin, the input window ends here}
  \State Append $\bigl(\log(P[t_i + h] / P[t_i]) : h \in \mathcal{H}\bigr)$ to $\mathcal{Y}$ \Comment{Eq.~\eqref{eq:target-return}}
\EndFor
\State Stack into tensors of the target shapes.
\end{algorithmic}
\end{algorithm*}

\begin{algorithm*}[!t]
\caption{Temporal branch forward pass}
\label{alg:temporal}
\begin{algorithmic}[1]
\Require $\mathbf{X}_{\text{temp}} \in \mathbb{R}^{N \times 20 \times 5}$, dropout rate $p = 0.3$, $d_{\text{fus}} = 64$
\Ensure $\mathbf{f}_1 \in \mathbb{R}^{N \times 64}$
\State $\mathbf{u} \gets \operatorname{Conv1D}_{64,3,\text{ReLU,same}}(\mathbf{X}_{\text{temp}})$ \Comment{$(N, 20, 64)$}
\State $\mathbf{u} \gets \operatorname{BiLSTM}_{32,\text{return\_seq}}(\mathbf{u})$ \Comment{$(N, 20, 64)$}
\State $\mathbf{u} \gets \operatorname{BiLSTM}_{64}(\mathbf{u})$ \Comment{$(N, 128)$}
\State $\mathbf{u} \gets \operatorname{Dropout}_{p}(\mathbf{u})$
\State $\mathbf{f}_1 \gets \operatorname{Dense}_{64,\text{ReLU}}(\mathbf{u})$ \Comment{$(N, 64)$}
\end{algorithmic}
\end{algorithm*}

\begin{algorithm*}[!t]
\caption{Factorized STST forward pass (Streams A and B + cross-stream MLP)}
\label{alg:stst}
\begin{algorithmic}[1]
\Require $\mathbf{D} \in \mathbb{R}^{N \times 20 \times 8 \times 7}$, dropout $p = 0.3$, $d_{\text{fus}} = 64$
\Ensure $\mathbf{f}_2 \in \mathbb{R}^{N \times 64}$
\Statex \textit{Stream A: temporal-spectral attention}
\State $\mathbf{T}_A \gets \operatorname{Reshape}(\mathbf{D}, (N, 20, 56))$
\For{$\ell \in \{1, 2\}$}
  \State $\mathbf{T}_A \gets \mathbf{T}_A + \operatorname{MHA}_{4,32,p_\text{a}=0.1}(\operatorname{LN}(\mathbf{T}_A))$
\EndFor
\State $\bm{\alpha}_A \gets \operatorname{Softmax}(\operatorname{Dense}_{1}(\mathbf{T}_A))$ over the 20-token axis
\State $\mathbf{h}_{\text{ts}} \gets \sum_{t=1}^{20} \alpha_{A,t}\, \mathbf{T}_A[:, t, :]$ \Comment{$(N, 56)$}
\Statex \textit{Stream B: scale-spectral attention}
\State $\mathbf{T}_B \gets \operatorname{Reshape}(\operatorname{Permute}_{(2,1,3)}(\mathbf{D}), (N, 8, 140))$
\For{$\ell \in \{1, 2\}$}
  \State $\mathbf{T}_B \gets \mathbf{T}_B + \operatorname{MHA}_{4,16,p_\text{a}=0.1}(\operatorname{LN}(\mathbf{T}_B))$
\EndFor
\State $\bm{\alpha}_B \gets \operatorname{Softmax}(\operatorname{Dense}_{1}(\mathbf{T}_B))$ over the 8-token axis
\State $\mathbf{h}_{\text{ss}} \gets \sum_{s=1}^{8} \alpha_{B,s}\, \mathbf{T}_B[:, s, :]$ \Comment{$(N, 140)$}
\Statex \textit{Cross-stream MLP fusion}
\State $\mathbf{c} \gets \operatorname{Dropout}_{p}(\operatorname{Dense}_{128,\text{ReLU}}(\operatorname{Dropout}_{p}(\operatorname{Dense}_{256,\text{ReLU}}([\mathbf{h}_{\text{ts}}, \mathbf{h}_{\text{ss}}]))))$
\State $\mathbf{f}_2 \gets \operatorname{Dense}_{64,\text{ReLU}}(\mathbf{c})$ \Comment{$(N, 64)$}
\end{algorithmic}
\end{algorithm*}

\begin{algorithm*}[!t]
\caption{Vertical Attention Fusion + Horizon-Adaptive Decoder}
\label{alg:vaf-had}
\begin{algorithmic}[1]
\Require $\mathbf{f}_1, \mathbf{f}_2 \in \mathbb{R}^{N \times 64}$, horizons $\mathcal{H} = \{1, 3, 5, 10\}$, dropout $p_h = 0.15$
\Ensure $\hat{\mathbf{Y}} \in \mathbb{R}^{N \times |\mathcal{H}|}$
\Statex \textit{VAF}
\State $\mathbf{S} \gets \operatorname{Stack}([\mathbf{f}_1, \mathbf{f}_2])$ \Comment{$(N, 2, 64)$}
\State $\mathbf{a} \gets \operatorname{Softmax}(\operatorname{Dense}_{1}(\mathbf{S}))$ over the 2-token axis
\State $\mathbf{h}_{\text{fus}} \gets a_1 \mathbf{f}_1 + a_2 \mathbf{f}_2$ \Comment{$(N, 64)$}
\Statex \textit{HAD}
\For{$h \in \mathcal{H}$}
  \State $\mathbf{g}_h \gets \tanh(\operatorname{Dense}_{64}(\mathbf{h}_{\text{fus}}))$ \Comment{$(N, 64)$}
  \State $\mathbf{s}_h \gets \mathbf{h}_{\text{fus}} \odot \mathbf{g}_h$
  \State $\hat{y}_h \gets \operatorname{Dense}_{1}(\operatorname{Dropout}_{p_h}(\operatorname{Dense}_{32,\text{ReLU}}(\mathbf{s}_h)))$
\EndFor
\State $\hat{\mathbf{Y}} \gets \operatorname{Concat}([\hat{y}_h]_{h \in \mathcal{H}})$
\end{algorithmic}
\end{algorithm*}

\begin{algorithm*}[!t]
\caption{Validation-calibrated asymmetric tilt allocator with 10 bps costs}
\label{alg:trading}
\begin{algorithmic}[1]
\Require Validation predictions and returns, test predictions $\hat{y}^{(t)}_1$, prices $\mathbf{p}$, cost $\kappa$
\Require Calibrated on validation data only: no-trade band $\eta$, dead-zone threshold $\theta$, risk-off trigger $\theta_{\text{off}}$, risk-on trigger $\theta_{\text{on}}$, tilt strength $\lambda$, base exposure $e_0$, risk-off floor $e_{\text{floor}}$, position bounds $[e_{\min}, e_{\max}]$
\Ensure Net daily P\&L series $\mathcal{P}$, Sharpe, max drawdown, cumulative return, mean exposure, total turnover
\State $b \gets$ mean predicted return on the validation set \Comment{bias term, Section~\ref{ssec:regime}}
\State $(\eta^\star, \theta^\star, \theta_{\text{off}}^\star, \theta_{\text{on}}^\star, \lambda^\star, e_0^\star, e_{\text{floor}}^\star) \gets \arg\max$ validation Sharpe over the parameter grid \Comment{fit once on validation, then frozen}
\State $\pi_{-1} \gets 0$
\For{$t \in$ test days}
  \State $z_t \gets \hat{y}^{(t)}_1 - b$ \Comment{bias-adjusted signal}
  \If{$|z_t| < \eta^\star$}
    \State $\pi_t \gets \pi_{t-1}$ \Comment{no-trade band, suppresses churn from noise near zero}
  \ElsIf{$z_t < -\theta_{\text{off}}^\star$}
    \State $\pi_t \gets \max\bigl(e_0^\star - \lambda^\star \cdot |z_t| / \theta^\star,\; e_{\text{floor}}^\star\bigr)$ \Comment{risk-off tilt, only engages when the down-signal clears the trigger}
  \ElsIf{$\theta_{\text{on}}^\star > 0$ \textbf{and} $z_t > \theta_{\text{on}}^\star$}
    \State $\pi_t \gets \min\bigl(e_0^\star + \lambda^\star \cdot |z_t| / \theta^\star,\; e_{\max}\bigr)$ \Comment{risk-on tilt, only defined for strategies with a two-sided rule}
  \Else
    \State $\pi_t \gets e_0^\star$ \Comment{signal inside the dead zone, hold the base exposure}
  \EndIf
  \State $\pi_t \gets \operatorname{clip}(\pi_t,\, e_{\min},\, e_{\max})$
  \State Append $\pi_t \,(p_{t+1} - p_t)/p_t \;-\; \kappa\,|\pi_t - \pi_{t-1}|$ to $\mathcal{P}$
\EndFor
\State Report annualised Sharpe, max drawdown, cumulative return, mean exposure, and $\sum_t |\pi_t - \pi_{t-1}|$
\Statex \textit{Long-only equities and gold long-only}: $e_{\min} = 0$. \textit{Gold long-short}: $e_{\min} = -1$, $e_{\text{floor}} = -1$, $\theta_{\text{off}} = \theta_{\text{on}} = 0$, so the rule collapses to a signed tilt around zero rather than a floor-based cut.
\Statex \textit{Execution timing}: The position $\pi_t$ is chosen from the signal observed at the close of day $t$ and earns the realised return from $t$ to $t+1$, so no return is booked before the position that earns it is set.
\end{algorithmic}
\end{algorithm*}

\section{Full Metric Tables and Secondary Diagnostics}
\label{app:full-metrics}
Tables~\ref{tab:results-finance-rmse} and~\ref{tab:results-finance-mape} give the RMSE and MAPE companions to Table~\ref{tab:results-finance-mae} under the identical protocol. Table~\ref{tab:baseline-da-h1} gives $H{=}1$ directional accuracy, Table~\ref{tab:had-subspace} the centred HAD gate-variation statistics, and Fig.~\ref{fig:allocation-curves} the allocation growth curves discussed in Section~\ref{ssec:trading}.

\begin{sidewaystable*}[p]
\caption{Per-horizon test RMSE across the full evaluation panel (KOSPI, Nikkei~225, S\&P 500, DJI, DAX, NASDAQ, NYSE, Gold). 95\% moving-block bootstrap CIs in brackets. Lower is better. Boldface indicates the lowest point estimate in each row and is descriptive rather than a claim of statistical significance.}
\label{tab:results-finance-rmse}
\centering
\fontsize{6.3}{7.2}\selectfont
\setlength{\tabcolsep}{2pt}
\renewcommand{\arraystretch}{1.0}
\begin{adjustbox}{max totalsize={\textheight}{0.88\textwidth}}
\begin{tabular}{@{}ll@{\hspace{4pt}}ccccccccccc@{}}
\toprule
\textbf{Asset} & \textbf{H} & \textbf{Persistence} & \textbf{LSTM} & \textbf{BiLSTM} & \textbf{CNN-BiLSTM} & \textbf{XGB} & \textbf{Transformer} & \textbf{DLinear} & \textbf{N-BEATS} & \textbf{PatchTST} & \textbf{iTransformer} & \textbf{DSTNet} \\
\midrule
\multirow{4}{*}{KOSPI} & 1 & \mc{27.28}{24.24}{31.71} & \mc{31.85}{28.36}{37.17} & \mc{30.03}{26.64}{35.01} & \mc{28.52}{25.05}{32.63} & \mc{29.48}{26.08}{34.27} & \mc{28.40}{24.97}{33.15} & \mc{27.74}{24.44}{31.56} & \mc{28.68}{25.90}{32.94} & \mc{28.25}{25.43}{32.94} & \mc{28.29}{24.82}{32.18} & \mcb{26.30}{23.12}{30.16} \\
 & 3 & \mc{43.21}{37.52}{51.01} & \mc{50.23}{44.77}{59.60} & \mc{46.94}{41.23}{55.33} & \mc{45.61}{39.88}{53.19} & \mc{50.10}{43.77}{60.09} & \mc{45.01}{40.00}{52.78} & \mc{46.83}{41.70}{56.17} & \mc{45.76}{40.19}{53.41} & \mc{45.26}{40.10}{53.98} & \mc{46.01}{39.87}{53.75} & \mcb{43.18}{37.70}{50.34} \\
 & 5 & \mc{57.63}{48.31}{71.19} & \mc{62.92}{53.76}{78.06} & \mc{60.49}{50.04}{76.51} & \mc{58.34}{50.01}{70.51} & \mc{66.10}{55.28}{82.37} & \mc{58.67}{49.19}{71.87} & \mc{58.66}{48.92}{72.99} & \mc{59.40}{49.70}{72.01} & \mc{58.84}{50.06}{72.95} & \mc{57.75}{49.40}{72.63} & \mcb{54.93}{47.11}{69.13} \\
 & 10 & \mc{81.79}{70.38}{97.41} & \mc{87.32}{72.57}{104.30} & \mc{86.35}{72.93}{105.31} & \mc{84.14}{71.15}{101.24} & \mc{92.53}{78.81}{112.54} & \mc{80.94}{68.16}{99.09} & \mc{82.17}{69.67}{98.78} & \mc{86.58}{73.15}{105.17} & \mc{84.61}{72.60}{104.50} & \mc{86.15}{72.31}{102.57} & \mcb{78.68}{67.64}{95.53} \\
\midrule
\multirow{4}{*}{Nikkei~225} & 1 & \mcb{381.66}{335.02}{445.58} & \mc{467.11}{419.65}{540.99} & \mc{430.28}{377.89}{501.99} & \mc{421.41}{376.02}{481.83} & \mc{437.67}{384.06}{507.28} & \mc{412.34}{367.00}{479.57} & \mc{414.29}{369.94}{484.19} & \mc{413.44}{368.84}{484.36} & \mc{399.54}{360.70}{456.90} & \mc{409.42}{362.39}{479.17} & \mc{384.06}{344.50}{439.70} \\
 & 3 & \mc{678.83}{587.81}{816.43} & \mc{776.35}{682.60}{903.08} & \mc{757.67}{666.99}{888.71} & \mc{711.96}{620.51}{844.70} & \mc{740.28}{658.05}{875.91} & \mc{699.14}{619.14}{835.85} & \mc{687.11}{596.18}{805.10} & \mc{707.07}{617.84}{822.76} & \mc{687.50}{611.59}{808.44} & \mc{715.37}{638.65}{857.00} & \mcb{661.97}{590.09}{775.44} \\
 & 5 & \mc{884.08}{751.29}{1070.49} & \mc{987.30}{836.60}{1230.85} & \mc{938.47}{772.42}{1155.77} & \mc{936.81}{798.45}{1161.45} & \mc{984.73}{843.45}{1245.69} & \mc{916.81}{780.00}{1136.77} & \mc{910.56}{748.65}{1146.98} & \mc{937.83}{797.60}{1176.32} & \mc{886.54}{733.08}{1121.20} & \mc{880.98}{756.03}{1087.82} & \mcb{852.42}{714.60}{1061.71} \\
 & 10 & \mc{1200.34}{1020.37}{1466.55} & \mc{1335.54}{1142.11}{1634.63} & \mc{1264.95}{1051.57}{1545.90} & \mc{1247.69}{1062.09}{1527.40} & \mc{1359.00}{1155.69}{1623.12} & \mc{1189.41}{991.62}{1444.44} & \mc{1206.73}{1030.75}{1439.75} & \mc{1239.98}{1068.35}{1492.79} & \mc{1170.60}{974.38}{1440.27} & \mc{1184.83}{1003.48}{1461.78} & \mcb{1138.83}{970.09}{1374.26} \\
\midrule
\multirow{4}{*}{S\&P 500} & 1 & \mc{33.90}{30.35}{39.27} & \mc{38.52}{34.08}{44.59} & \mc{38.64}{34.39}{43.96} & \mc{37.47}{33.69}{43.24} & \mc{39.96}{35.29}{46.70} & \mc{36.56}{32.59}{42.49} & \mc{36.77}{32.28}{42.66} & \mc{36.31}{32.01}{41.71} & \mc{36.31}{32.27}{41.88} & \mc{36.22}{32.55}{41.74} & \mcb{33.82}{30.10}{39.64} \\
 & 3 & \mc{62.57}{55.02}{72.94} & \mc{69.77}{60.77}{81.02} & \mc{67.61}{59.25}{80.94} & \mc{66.25}{58.55}{78.80} & \mc{68.48}{61.05}{81.27} & \mc{64.39}{57.44}{76.87} & \mc{64.76}{57.02}{76.48} & \mc{66.39}{57.81}{77.79} & \mc{65.41}{57.50}{78.15} & \mc{64.74}{55.99}{75.69} & \mcb{59.75}{51.85}{71.45} \\
 & 5 & \mc{79.84}{67.62}{99.00} & \mc{90.65}{75.52}{113.74} & \mc{86.77}{73.17}{107.46} & \mc{83.88}{69.42}{105.44} & \mc{90.93}{74.74}{113.75} & \mc{80.51}{68.10}{97.97} & \mc{83.75}{69.82}{105.58} & \mc{82.60}{70.19}{103.56} & \mc{82.71}{68.02}{100.47} & \mc{84.17}{70.71}{102.50} & \mcb{77.29}{66.36}{94.00} \\
 & 10 & \mc{116.64}{99.22}{140.53} & \mc{127.26}{107.29}{156.59} & \mc{122.95}{103.39}{148.69} & \mc{122.20}{104.28}{147.21} & \mc{128.29}{110.79}{156.63} & \mc{118.79}{100.57}{146.00} & \mc{118.47}{99.36}{146.50} & \mc{119.02}{99.12}{145.38} & \mc{119.30}{100.40}{146.14} & \mc{116.40}{97.70}{143.96} & \mcb{111.65}{95.56}{132.90} \\
\midrule
\multirow{4}{*}{DJI} & 1 & \mc{223.69}{200.53}{254.28} & \mc{263.66}{232.26}{301.58} & \mc{250.26}{222.30}{290.78} & \mc{243.37}{215.83}{282.71} & \mc{255.09}{229.79}{292.69} & \mc{243.17}{216.54}{283.64} & \mc{239.03}{213.22}{278.75} & \mc{233.91}{209.72}{273.76} & \mc{238.13}{210.92}{274.25} & \mc{240.43}{216.45}{282.13} & \mcb{223.31}{198.53}{256.08} \\
 & 3 & \mc{419.83}{364.14}{499.05} & \mc{465.70}{409.57}{544.67} & \mc{455.61}{406.45}{540.18} & \mc{438.86}{383.73}{524.74} & \mc{464.37}{412.55}{541.00} & \mc{434.20}{383.70}{505.54} & \mc{431.02}{377.84}{515.82} & \mc{424.28}{377.69}{505.73} & \mc{426.41}{372.31}{503.14} & \mc{431.91}{380.19}{509.94} & \mcb{404.85}{359.42}{484.36} \\
 & 5 & \mc{525.75}{443.32}{656.95} & \mc{591.13}{495.45}{747.79} & \mc{578.75}{477.81}{708.88} & \mc{566.67}{467.84}{705.68} & \mc{593.65}{488.58}{735.38} & \mc{555.56}{474.84}{701.97} & \mc{546.71}{455.29}{671.50} & \mc{557.51}{462.11}{679.34} & \mc{549.99}{468.38}{690.75} & \mc{554.94}{460.60}{674.11} & \mcb{517.33}{427.26}{646.86} \\
 & 10 & \mc{805.00}{679.39}{974.14} & \mc{856.60}{712.41}{1046.87} & \mc{873.76}{738.43}{1054.12} & \mc{813.57}{692.47}{991.00} & \mc{932.02}{780.13}{1119.84} & \mc{821.41}{690.71}{1017.44} & \mc{827.95}{703.23}{984.90} & \mc{843.56}{730.64}{1014.50} & \mc{785.05}{658.36}{948.01} & \mc{835.97}{715.79}{1010.68} & \mcb{766.82}{642.39}{935.32} \\
\midrule
\multirow{4}{*}{DAX} & 1 & \mc{117.52}{104.33}{137.15} & \mc{133.24}{117.15}{152.92} & \mc{129.74}{114.64}{152.01} & \mc{123.71}{110.29}{144.96} & \mc{134.46}{119.83}{156.12} & \mc{120.08}{107.77}{140.35} & \mc{119.90}{105.73}{137.86} & \mc{122.20}{107.68}{139.09} & \mc{124.96}{112.80}{143.89} & \mc{122.49}{109.92}{142.09} & \mcb{114.86}{102.54}{131.30} \\
 & 3 & \mc{211.36}{188.19}{247.17} & \mc{235.79}{209.73}{283.04} & \mc{227.12}{198.04}{264.82} & \mc{224.70}{200.13}{267.19} & \mc{232.20}{203.45}{269.40} & \mc{213.28}{189.66}{255.60} & \mc{211.63}{183.02}{251.28} & \mc{214.57}{188.49}{257.91} & \mc{210.82}{185.14}{246.24} & \mc{221.30}{195.83}{257.19} & \mcb{202.20}{176.95}{236.51} \\
 & 5 & \mc{273.90}{226.34}{339.77} & \mc{310.89}{261.54}{377.44} & \mc{296.35}{253.76}{366.40} & \mc{294.26}{242.11}{361.64} & \mc{327.20}{278.38}{397.78} & \mc{279.71}{231.41}{340.35} & \mc{291.04}{247.16}{355.36} & \mc{296.45}{253.14}{363.52} & \mc{290.56}{240.15}{359.13} & \mc{282.27}{238.00}{353.55} & \mcb{270.53}{222.75}{340.76} \\
 & 10 & \mc{412.84}{354.67}{491.92} & \mc{447.75}{386.05}{542.50} & \mc{423.05}{352.24}{503.43} & \mc{407.43}{350.49}{486.71} & \mc{465.27}{402.29}{571.17} & \mc{395.54}{339.38}{473.57} & \mc{402.63}{339.13}{488.57} & \mc{425.66}{354.41}{510.29} & \mc{404.82}{344.48}{482.46} & \mc{399.96}{335.39}{493.89} & \mcb{384.40}{324.40}{473.12} \\
\midrule
\multirow{4}{*}{NASDAQ} & 1 & \mc{150.71}{134.03}{175.44} & \mc{167.78}{149.51}{196.68} & \mc{164.95}{147.79}{191.33} & \mc{162.25}{143.25}{188.73} & \mc{164.67}{148.07}{192.79} & \mc{155.08}{139.57}{178.27} & \mc{156.16}{140.97}{177.70} & \mc{157.71}{140.98}{181.32} & \mc{154.94}{137.44}{181.08} & \mc{158.97}{143.48}{184.61} & \mcb{146.89}{132.22}{170.48} \\
 & 3 & \mc{255.66}{222.66}{302.08} & \mc{281.33}{247.46}{331.05} & \mc{284.88}{251.85}{342.56} & \mc{269.32}{237.11}{314.89} & \mc{287.03}{249.54}{333.91} & \mc{267.94}{233.57}{320.72} & \mc{265.23}{234.92}{317.76} & \mc{270.51}{236.29}{316.88} & \mc{262.20}{230.52}{313.25} & \mc{273.49}{240.89}{319.02} & \mcb{250.22}{221.24}{293.20} \\
 & 5 & \mc{326.49}{274.53}{412.14} & \mc{365.46}{307.49}{460.59} & \mc{356.93}{300.81}{443.66} & \mc{339.91}{286.47}{418.84} & \mc{373.17}{310.44}{468.23} & \mc{342.64}{282.04}{425.80} & \mc{337.19}{277.58}{422.79} & \mc{359.52}{302.06}{454.36} & \mc{346.56}{292.91}{420.52} & \mc{336.85}{283.63}{407.26} & \mcb{321.68}{274.33}{399.50} \\
 & 10 & \mc{456.24}{382.67}{562.48} & \mc{507.00}{438.14}{618.31} & \mc{476.92}{401.42}{587.62} & \mc{467.64}{398.26}{571.68} & \mc{518.11}{438.25}{630.62} & \mc{469.52}{397.29}{563.06} & \mc{483.60}{412.15}{599.46} & \mc{495.27}{415.56}{613.18} & \mc{478.33}{401.43}{589.06} & \mc{470.80}{406.45}{579.68} & \mcb{445.78}{378.37}{544.01} \\
\midrule
\multirow{4}{*}{NYSE} & 1 & \mc{119.92}{105.22}{139.92} & \mc{133.14}{120.23}{155.61} & \mc{130.43}{115.32}{151.99} & \mc{125.43}{112.39}{146.06} & \mc{128.11}{114.69}{146.59} & \mc{126.98}{111.99}{144.64} & \mc{123.52}{111.54}{143.08} & \mc{128.73}{115.17}{149.31} & \mc{126.67}{111.69}{145.82} & \mc{121.67}{109.19}{142.01} & \mcb{116.40}{104.32}{136.06} \\
 & 3 & \mc{234.34}{205.01}{271.93} & \mc{263.08}{232.01}{306.61} & \mc{257.10}{229.09}{298.64} & \mc{242.27}{213.44}{283.83} & \mc{261.18}{231.43}{306.50} & \mc{244.55}{218.21}{283.70} & \mc{237.35}{210.01}{280.70} & \mc{244.34}{215.19}{287.04} & \mc{249.26}{216.20}{290.41} & \mc{244.42}{218.31}{288.58} & \mcb{228.26}{202.17}{269.49} \\
 & 5 & \mc{319.67}{264.82}{391.20} & \mc{359.06}{308.40}{449.18} & \mc{327.83}{274.23}{408.77} & \mc{326.82}{275.02}{407.39} & \mc{340.67}{283.29}{415.13} & \mc{325.84}{269.53}{404.47} & \mc{325.31}{277.29}{408.97} & \mc{324.70}{272.16}{407.42} & \mc{319.33}{269.49}{393.26} & \mc{324.12}{266.89}{394.15} & \mcb{303.07}{256.21}{366.63} \\
 & 10 & \mc{444.46}{374.08}{535.21} & \mc{486.80}{414.71}{587.00} & \mc{466.72}{393.51}{566.26} & \mc{447.31}{381.47}{545.32} & \mc{516.68}{434.85}{627.26} & \mc{456.39}{384.01}{552.12} & \mc{445.56}{385.70}{541.21} & \mc{458.17}{386.37}{562.33} & \mc{453.05}{391.42}{540.36} & \mc{459.70}{388.58}{556.53} & \mcb{430.02}{371.31}{512.22} \\
\midrule
\multirow{4}{*}{Gold} & 1 & \mc{16.69}{14.97}{19.33} & \mc{19.34}{17.29}{22.50} & \mc{18.28}{16.27}{21.28} & \mc{17.32}{15.44}{19.88} & \mc{18.90}{16.76}{21.83} & \mc{16.86}{15.20}{19.23} & \mc{17.30}{15.34}{19.86} & \mc{17.06}{15.08}{19.70} & \mc{17.27}{15.23}{20.21} & \mc{17.33}{15.48}{19.97} & \mcb{16.21}{14.24}{18.86} \\
 & 3 & \mc{29.60}{26.11}{34.72} & \mc{34.59}{30.09}{40.46} & \mc{32.85}{28.92}{39.40} & \mc{31.07}{27.72}{36.71} & \mc{33.24}{29.17}{39.68} & \mc{31.33}{27.68}{37.29} & \mc{30.40}{26.63}{35.47} & \mc{31.05}{27.66}{37.05} & \mc{30.19}{26.91}{35.50} & \mc{30.59}{26.81}{36.45} & \mcb{28.98}{25.55}{34.21} \\
 & 5 & \mc{43.27}{36.67}{52.71} & \mc{48.25}{41.13}{59.96} & \mc{46.96}{39.25}{58.43} & \mc{45.35}{38.23}{55.72} & \mc{51.04}{43.02}{64.03} & \mc{44.47}{37.96}{54.37} & \mc{45.22}{38.44}{55.95} & \mc{44.50}{36.81}{54.08} & \mc{43.60}{36.67}{53.75} & \mc{44.13}{36.67}{53.90} & \mcb{42.08}{35.30}{51.78} \\
 & 10 & \mc{65.10}{56.18}{79.06} & \mc{70.81}{59.72}{88.81} & \mc{69.97}{59.08}{83.68} & \mc{66.99}{57.90}{82.58} & \mc{72.99}{62.44}{88.52} & \mc{65.02}{55.35}{79.01} & \mc{64.02}{54.90}{77.88} & \mc{67.94}{57.29}{83.47} & \mc{66.71}{57.76}{80.78} & \mc{67.56}{57.83}{83.20} & \mcb{62.09}{53.38}{76.42} \\
\bottomrule
\end{tabular}
\end{adjustbox}
\end{sidewaystable*}

\begin{sidewaystable*}[p]
\caption{Per-horizon test MAPE (\%) across the full evaluation panel (KOSPI, Nikkei~225, S\&P 500, DJI, DAX, NASDAQ, NYSE, Gold). 95\% moving-block bootstrap CIs in brackets. Lower is better. Boldface indicates the lowest point estimate in each row and is descriptive rather than a claim of statistical significance.}
\label{tab:results-finance-mape}
\centering
\fontsize{6.3}{7.2}\selectfont
\setlength{\tabcolsep}{2pt}
\renewcommand{\arraystretch}{1.0}
\begin{adjustbox}{max totalsize={\textheight}{0.88\textwidth}}
\begin{tabular}{@{}ll@{\hspace{4pt}}ccccccccccc@{}}
\toprule
\textbf{Asset} & \textbf{H} & \textbf{Persistence} & \textbf{LSTM} & \textbf{BiLSTM} & \textbf{CNN-BiLSTM} & \textbf{XGB} & \textbf{Transformer} & \textbf{DLinear} & \textbf{N-BEATS} & \textbf{PatchTST} & \textbf{iTransformer} & \textbf{DSTNet} \\
\midrule
\multirow{4}{*}{KOSPI} & 1 & \mc{0.777}{0.701}{0.886} & \mc{0.909}{0.813}{1.053} & \mc{0.865}{0.760}{0.993} & \mc{0.822}{0.728}{0.950} & \mc{0.872}{0.765}{1.006} & \mc{0.822}{0.725}{0.962} & \mc{0.820}{0.731}{0.959} & \mc{0.829}{0.742}{0.958} & \mc{0.824}{0.727}{0.955} & \mc{0.823}{0.734}{0.941} & \mcb{0.771}{0.677}{0.903} \\
 & 3 & \mc{1.321}{1.154}{1.548} & \mc{1.515}{1.347}{1.779} & \mc{1.428}{1.260}{1.704} & \mc{1.388}{1.214}{1.635} & \mc{1.507}{1.336}{1.756} & \mc{1.380}{1.231}{1.620} & \mc{1.374}{1.207}{1.643} & \mc{1.392}{1.213}{1.619} & \mc{1.373}{1.188}{1.610} & \mc{1.402}{1.232}{1.651} & \mcb{1.302}{1.145}{1.544} \\
 & 5 & \mc{1.747}{1.454}{2.176} & \mc{1.986}{1.635}{2.480} & \mc{1.872}{1.552}{2.334} & \mc{1.830}{1.549}{2.250} & \mc{2.012}{1.724}{2.528} & \mc{1.789}{1.486}{2.175} & \mc{1.816}{1.540}{2.255} & \mc{1.854}{1.533}{2.339} & \mc{1.800}{1.480}{2.229} & \mc{1.826}{1.561}{2.302} & \mcb{1.704}{1.440}{2.122} \\
 & 10 & \mc{2.561}{2.137}{3.159} & \mc{2.753}{2.348}{3.356} & \mc{2.723}{2.279}{3.339} & \mc{2.634}{2.266}{3.159} & \mc{2.828}{2.442}{3.473} & \mc{2.575}{2.198}{3.082} & \mc{2.615}{2.198}{3.166} & \mc{2.654}{2.297}{3.167} & \mc{2.607}{2.208}{3.181} & \mc{2.634}{2.247}{3.170} & \mcb{2.469}{2.063}{2.964} \\
\midrule
\multirow{4}{*}{Nikkei~225} & 1 & \mc{0.881}{0.786}{1.011} & \mc{1.035}{0.912}{1.190} & \mc{0.984}{0.863}{1.148} & \mc{0.951}{0.856}{1.111} & \mc{1.009}{0.900}{1.182} & \mc{0.917}{0.823}{1.077} & \mc{0.916}{0.805}{1.067} & \mc{0.940}{0.827}{1.094} & \mc{0.924}{0.811}{1.052} & \mc{0.929}{0.835}{1.060} & \mcb{0.873}{0.777}{1.000} \\
 & 3 & \mc{1.516}{1.321}{1.787} & \mc{1.728}{1.497}{2.078} & \mc{1.694}{1.474}{1.974} & \mc{1.607}{1.396}{1.882} & \mc{1.681}{1.491}{1.982} & \mc{1.566}{1.382}{1.818} & \mc{1.562}{1.394}{1.831} & \mc{1.586}{1.377}{1.850} & \mc{1.564}{1.394}{1.881} & \mc{1.573}{1.375}{1.889} & \mcb{1.491}{1.321}{1.741} \\
 & 5 & \mc{2.004}{1.683}{2.458} & \mc{2.275}{1.917}{2.786} & \mc{2.165}{1.841}{2.696} & \mc{2.115}{1.757}{2.589} & \mc{2.246}{1.902}{2.718} & \mc{2.048}{1.688}{2.572} & \mc{2.052}{1.729}{2.559} & \mc{2.094}{1.770}{2.607} & \mc{2.052}{1.758}{2.551} & \mc{2.057}{1.746}{2.583} & \mcb{1.959}{1.634}{2.379} \\
 & 10 & \mc{2.730}{2.287}{3.310} & \mc{3.022}{2.579}{3.673} & \mc{2.911}{2.432}{3.542} & \mc{2.814}{2.348}{3.365} & \mc{3.067}{2.653}{3.724} & \mc{2.755}{2.324}{3.386} & \mc{2.789}{2.356}{3.446} & \mc{2.815}{2.349}{3.423} & \mc{2.761}{2.306}{3.352} & \mc{2.766}{2.304}{3.335} & \mcb{2.636}{2.264}{3.255} \\
\midrule
\multirow{4}{*}{S\&P 500} & 1 & \mc{0.577}{0.520}{0.661} & \mc{0.658}{0.589}{0.749} & \mc{0.635}{0.573}{0.734} & \mc{0.626}{0.554}{0.713} & \mc{0.665}{0.594}{0.768} & \mc{0.604}{0.534}{0.701} & \mc{0.607}{0.538}{0.697} & \mc{0.616}{0.556}{0.713} & \mc{0.602}{0.532}{0.688} & \mc{0.613}{0.551}{0.708} & \mcb{0.572}{0.504}{0.656} \\
 & 3 & \mc{1.054}{0.917}{1.247} & \mc{1.221}{1.076}{1.458} & \mc{1.167}{1.035}{1.361} & \mc{1.118}{0.967}{1.299} & \mc{1.211}{1.059}{1.455} & \mc{1.090}{0.950}{1.298} & \mc{1.094}{0.966}{1.288} & \mc{1.121}{0.981}{1.333} & \mc{1.107}{0.962}{1.323} & \mc{1.097}{0.957}{1.276} & \mcb{1.037}{0.926}{1.230} \\
 & 5 & \mc{1.349}{1.123}{1.671} & \mc{1.516}{1.278}{1.842} & \mc{1.442}{1.217}{1.793} & \mc{1.409}{1.184}{1.744} & \mc{1.508}{1.282}{1.824} & \mc{1.377}{1.167}{1.668} & \mc{1.396}{1.156}{1.754} & \mc{1.420}{1.220}{1.779} & \mc{1.381}{1.146}{1.711} & \mc{1.411}{1.163}{1.751} & \mcb{1.315}{1.107}{1.654} \\
 & 10 & \mc{2.052}{1.771}{2.468} & \mc{2.256}{1.882}{2.750} & \mc{2.172}{1.864}{2.627} & \mc{2.116}{1.759}{2.575} & \mc{2.269}{1.903}{2.736} & \mc{2.091}{1.772}{2.572} & \mc{2.088}{1.761}{2.542} & \mc{2.116}{1.759}{2.619} & \mc{2.079}{1.776}{2.475} & \mc{2.101}{1.764}{2.596} & \mcb{1.978}{1.711}{2.429} \\
\midrule
\multirow{4}{*}{DJI} & 1 & \mc{0.485}{0.433}{0.557} & \mc{0.572}{0.509}{0.656} & \mc{0.542}{0.488}{0.624} & \mc{0.518}{0.461}{0.607} & \mc{0.540}{0.480}{0.621} & \mc{0.513}{0.458}{0.584} & \mc{0.504}{0.454}{0.584} & \mc{0.511}{0.460}{0.583} & \mc{0.510}{0.447}{0.587} & \mc{0.512}{0.458}{0.595} & \mcb{0.481}{0.433}{0.558} \\
 & 3 & \mc{0.924}{0.824}{1.079} & \mc{1.046}{0.912}{1.222} & \mc{1.006}{0.873}{1.198} & \mc{0.970}{0.852}{1.158} & \mc{1.025}{0.913}{1.217} & \mc{0.958}{0.832}{1.126} & \mc{0.959}{0.835}{1.126} & \mc{0.963}{0.859}{1.138} & \mc{0.956}{0.850}{1.119} & \mc{0.968}{0.853}{1.135} & \mcb{0.909}{0.803}{1.085} \\
 & 5 & \mc{1.196}{1.027}{1.509} & \mc{1.318}{1.115}{1.651} & \mc{1.295}{1.069}{1.593} & \mc{1.260}{1.048}{1.528} & \mc{1.325}{1.124}{1.647} & \mc{1.215}{0.999}{1.487} & \mc{1.229}{1.028}{1.524} & \mc{1.247}{1.059}{1.576} & \mc{1.213}{1.014}{1.505} & \mc{1.234}{1.034}{1.534} & \mcb{1.163}{0.970}{1.469} \\
 & 10 & \mc{1.787}{1.516}{2.214} & \mc{1.929}{1.612}{2.325} & \mc{1.919}{1.597}{2.330} & \mc{1.831}{1.548}{2.253} & \mc{2.036}{1.734}{2.491} & \mc{1.799}{1.496}{2.225} & \mc{1.815}{1.547}{2.217} & \mc{1.837}{1.583}{2.253} & \mc{1.786}{1.501}{2.125} & \mc{1.827}{1.544}{2.255} & \mcb{1.718}{1.481}{2.095} \\
\midrule
\multirow{4}{*}{DAX} & 1 & \mc{0.546}{0.482}{0.621} & \mc{0.615}{0.544}{0.720} & \mc{0.603}{0.541}{0.688} & \mc{0.584}{0.520}{0.685} & \mc{0.619}{0.545}{0.714} & \mc{0.571}{0.508}{0.668} & \mc{0.568}{0.501}{0.649} & \mc{0.584}{0.524}{0.672} & \mc{0.575}{0.512}{0.671} & \mc{0.575}{0.510}{0.671} & \mcb{0.541}{0.487}{0.619} \\
 & 3 & \mc{0.989}{0.882}{1.180} & \mc{1.133}{1.000}{1.351} & \mc{1.091}{0.951}{1.288} & \mc{1.056}{0.935}{1.266} & \mc{1.103}{0.983}{1.311} & \mc{1.029}{0.894}{1.225} & \mc{1.021}{0.911}{1.198} & \mc{1.036}{0.905}{1.234} & \mc{1.027}{0.904}{1.210} & \mc{1.039}{0.912}{1.239} & \mcb{0.972}{0.865}{1.160} \\
 & 5 & \mc{1.369}{1.158}{1.685} & \mc{1.532}{1.284}{1.898} & \mc{1.453}{1.228}{1.802} & \mc{1.433}{1.206}{1.742} & \mc{1.575}{1.303}{1.932} & \mc{1.400}{1.154}{1.743} & \mc{1.414}{1.188}{1.759} & \mc{1.438}{1.225}{1.799} & \mc{1.402}{1.171}{1.719} & \mc{1.409}{1.180}{1.750} & \mcb{1.327}{1.140}{1.646} \\
 & 10 & \mc{1.998}{1.691}{2.436} & \mc{2.170}{1.850}{2.651} & \mc{2.099}{1.800}{2.522} & \mc{2.023}{1.691}{2.450} & \mc{2.291}{1.957}{2.730} & \mc{1.990}{1.668}{2.417} & \mc{2.015}{1.696}{2.489} & \mc{2.074}{1.750}{2.477} & \mc{2.018}{1.712}{2.469} & \mc{2.023}{1.732}{2.452} & \mcb{1.908}{1.629}{2.346} \\
\midrule
\multirow{4}{*}{NASDAQ} & 1 & \mc{0.795}{0.706}{0.921} & \mc{0.895}{0.789}{1.019} & \mc{0.900}{0.794}{1.044} & \mc{0.860}{0.773}{0.982} & \mc{0.885}{0.793}{1.026} & \mc{0.835}{0.733}{0.954} & \mc{0.833}{0.750}{0.968} & \mc{0.839}{0.738}{0.972} & \mc{0.831}{0.737}{0.973} & \mc{0.838}{0.737}{0.953} & \mcb{0.788}{0.711}{0.897} \\
 & 3 & \mc{1.423}{1.268}{1.669} & \mc{1.599}{1.384}{1.898} & \mc{1.580}{1.403}{1.898} & \mc{1.519}{1.343}{1.807} & \mc{1.631}{1.431}{1.899} & \mc{1.477}{1.292}{1.759} & \mc{1.482}{1.301}{1.753} & \mc{1.512}{1.350}{1.810} & \mc{1.478}{1.296}{1.742} & \mc{1.512}{1.350}{1.773} & \mcb{1.398}{1.223}{1.656} \\
 & 5 & \mc{1.853}{1.523}{2.315} & \mc{2.015}{1.698}{2.513} & \mc{1.983}{1.673}{2.500} & \mc{1.907}{1.615}{2.406} & \mc{2.119}{1.798}{2.625} & \mc{1.886}{1.579}{2.370} & \mc{1.899}{1.592}{2.312} & \mc{1.955}{1.679}{2.453} & \mc{1.912}{1.583}{2.414} & \mc{1.914}{1.584}{2.392} & \mcb{1.799}{1.503}{2.253} \\
 & 10 & \mc{2.623}{2.224}{3.220} & \mc{2.910}{2.487}{3.589} & \mc{2.734}{2.365}{3.266} & \mc{2.681}{2.238}{3.303} & \mc{2.903}{2.464}{3.474} & \mc{2.657}{2.250}{3.271} & \mc{2.693}{2.324}{3.206} & \mc{2.729}{2.308}{3.308} & \mc{2.643}{2.215}{3.153} & \mc{2.659}{2.240}{3.172} & \mcb{2.517}{2.145}{3.074} \\
\midrule
\multirow{4}{*}{NYSE} & 1 & \mc{0.572}{0.509}{0.664} & \mc{0.652}{0.587}{0.742} & \mc{0.637}{0.563}{0.745} & \mc{0.611}{0.547}{0.700} & \mc{0.638}{0.573}{0.732} & \mc{0.606}{0.543}{0.690} & \mc{0.601}{0.537}{0.704} & \mc{0.612}{0.549}{0.712} & \mc{0.606}{0.543}{0.705} & \mc{0.603}{0.539}{0.703} & \mcb{0.568}{0.506}{0.664} \\
 & 3 & \mc{1.125}{0.996}{1.352} & \mc{1.266}{1.115}{1.495} & \mc{1.238}{1.071}{1.442} & \mc{1.193}{1.063}{1.398} & \mc{1.285}{1.126}{1.507} & \mc{1.160}{1.005}{1.371} & \mc{1.174}{1.046}{1.398} & \mc{1.181}{1.023}{1.371} & \mc{1.177}{1.022}{1.375} & \mc{1.190}{1.033}{1.408} & \mcb{1.108}{0.969}{1.307} \\
 & 5 & \mc{1.525}{1.283}{1.913} & \mc{1.738}{1.481}{2.160} & \mc{1.625}{1.375}{2.016} & \mc{1.578}{1.334}{1.963} & \mc{1.701}{1.414}{2.105} & \mc{1.555}{1.280}{1.898} & \mc{1.580}{1.356}{1.983} & \mc{1.603}{1.340}{2.009} & \mc{1.572}{1.326}{1.983} & \mc{1.568}{1.300}{1.962} & \mcb{1.484}{1.246}{1.865} \\
 & 10 & \mc{2.315}{1.950}{2.786} & \mc{2.531}{2.166}{3.101} & \mc{2.420}{2.070}{2.933} & \mc{2.346}{2.013}{2.892} & \mc{2.664}{2.260}{3.234} & \mc{2.331}{1.991}{2.776} & \mc{2.335}{2.011}{2.893} & \mc{2.400}{2.033}{2.877} & \mc{2.314}{1.971}{2.816} & \mc{2.335}{1.989}{2.875} & \mcb{2.215}{1.886}{2.733} \\
\midrule
\multirow{4}{*}{Gold} & 1 & \mc{0.586}{0.518}{0.683} & \mc{0.677}{0.603}{0.789} & \mc{0.644}{0.568}{0.740} & \mc{0.632}{0.569}{0.726} & \mc{0.663}{0.596}{0.778} & \mc{0.613}{0.549}{0.715} & \mc{0.611}{0.537}{0.708} & \mc{0.622}{0.547}{0.709} & \mc{0.616}{0.554}{0.700} & \mc{0.622}{0.558}{0.716} & \mcb{0.582}{0.515}{0.662} \\
 & 3 & \mc{1.111}{0.981}{1.307} & \mc{1.279}{1.132}{1.538} & \mc{1.218}{1.055}{1.452} & \mc{1.180}{1.039}{1.392} & \mc{1.251}{1.091}{1.466} & \mc{1.154}{1.019}{1.371} & \mc{1.148}{1.018}{1.338} & \mc{1.177}{1.025}{1.401} & \mc{1.151}{0.999}{1.345} & \mc{1.172}{1.021}{1.387} & \mcb{1.094}{0.962}{1.314} \\
 & 5 & \mc{1.600}{1.341}{1.951} & \mc{1.801}{1.502}{2.189} & \mc{1.717}{1.464}{2.125} & \mc{1.679}{1.405}{2.079} & \mc{1.836}{1.555}{2.305} & \mc{1.631}{1.341}{2.025} & \mc{1.625}{1.396}{2.010} & \mc{1.670}{1.381}{2.072} & \mc{1.621}{1.388}{1.980} & \mc{1.640}{1.391}{2.010} & \mcb{1.551}{1.288}{1.943} \\
 & 10 & \mc{2.400}{2.033}{2.861} & \mc{2.579}{2.145}{3.117} & \mc{2.535}{2.152}{3.018} & \mc{2.431}{2.079}{2.966} & \mc{2.672}{2.266}{3.273} & \mc{2.406}{2.017}{2.949} & \mc{2.420}{2.024}{2.904} & \mc{2.466}{2.069}{2.947} & \mc{2.422}{2.053}{2.886} & \mc{2.443}{2.088}{2.936} & \mcb{2.302}{1.948}{2.821} \\
\bottomrule
\end{tabular}
\end{adjustbox}
\end{sidewaystable*}

\begin{table*}[!t]
\caption{$H=1$ directional accuracy (\%).}
\label{tab:baseline-da-h1}
\centering
\small
\setlength{\tabcolsep}{6pt}
\begin{tabular}{@{}lrrrrrrrr@{}}
\toprule
\textbf{Model} & \textbf{KOSPI} & \textbf{Nikkei~225} & \textbf{S\&P 500} & \textbf{DJI} & \textbf{DAX} & \textbf{NASDAQ} & \textbf{NYSE} & \textbf{Gold} \\
\midrule
LSTM & 46.38 & 48.98 & 52.87 & 56.16 & 51.45 & 53.43 & 45.91 & 49.55 \\
BiLSTM & 47.69 & 49.69 & 52.82 & 56.53 & 51.28 & 54.99 & 46.33 & 51.39 \\
CNN-BiLSTM & 48.45 & 51.56 & 54.62 & 56.98 & 52.06 & 55.11 & 47.02 & 52.04 \\
XGBoost & 46.70 & 49.58 & 52.63 & 56.97 & 51.15 & 54.12 & 46.44 & 50.92 \\
DLinear & 48.65 & 51.80 & 55.14 & 57.81 & 53.45 & 55.66 & 47.29 & 52.76 \\
N-BEATS & 48.40 & 50.67 & 54.48 & 57.54 & 52.32 & 56.17 & 47.54 & 51.37 \\
PatchTST & 48.19 & 51.59 & 55.08 & 58.87 & 53.67 & 55.87 & 47.17 & 52.18 \\
iTransformer & 48.48 & 51.93 & 55.07 & 57.32 & 53.49 & 55.87 & 47.94 & 51.58 \\
Vanilla Transformer & 48.66 & 51.91 & 55.04 & 57.76 & 52.82 & 56.20 & 47.47 & 53.07 \\
\midrule
DSTNet & \textbf{49.36} & \textbf{52.54} & \textbf{55.79} & \textbf{59.09} & \textbf{54.07} & \textbf{57.02} & \textbf{48.76} & \textbf{53.31} \\
\bottomrule
\end{tabular}
\end{table*}

\begin{table*}[!t]
\caption{Centred HAD gate-variation analysis. $r_H$ is the number of right-singular directions needed to explain 95\% of the squared singular-value mass of the centred gate matrix on the evaluated test sample. Angles are mean principal angles in degrees between adjacent estimated variation subspaces. Neither compact rank nor large angles alone establishes purposeful horizon specialization or stability beyond the evaluated sample.}
\label{tab:had-subspace}
\centering
\small
\setlength{\tabcolsep}{7pt}
\begin{tabular}{@{}lrrrrrrrr@{}}
\toprule
\textbf{Series} & $r_1$ & $r_3$ & $r_5$ & $r_{10}$ &
$\angle_{1,3}$ & $\angle_{3,5}$ & $\angle_{5,10}$ & \textbf{Mean} \\
\midrule
KOSPI      & 3 & 3 & 3 & 2 & 76.02 & 77.72 & 81.13 & 78.29 \\
S\&P 500   & 3 & 2 & 2 & 2 & 82.85 & 83.10 & 84.87 & 83.61 \\
Nikkei~225 & 3 & 4 & 4 & 3 & 77.92 & 78.99 & 73.77 & 76.89 \\
DJI        & 3 & 2 & 4 & 2 & 79.59 & 79.58 & 77.61 & 78.93 \\
DAX        & 2 & 2 & 2 & 2 & 86.03 & 83.67 & 84.05 & 84.58 \\
NASDAQ     & 1 & 1 & 1 & 1 & 85.79 & 76.46 & 88.91 & 83.72 \\
NYSE       & 2 & 2 & 2 & 2 & 80.72 & 80.78 & 75.66 & 79.05 \\
Gold       & 2 & 2 & 2 & 2 & 82.79 & 76.84 & 82.18 & 80.60 \\
\bottomrule
\end{tabular}
\end{table*}

\begin{figure*}[!t]
\centering
\includegraphics[width=0.88\textwidth]{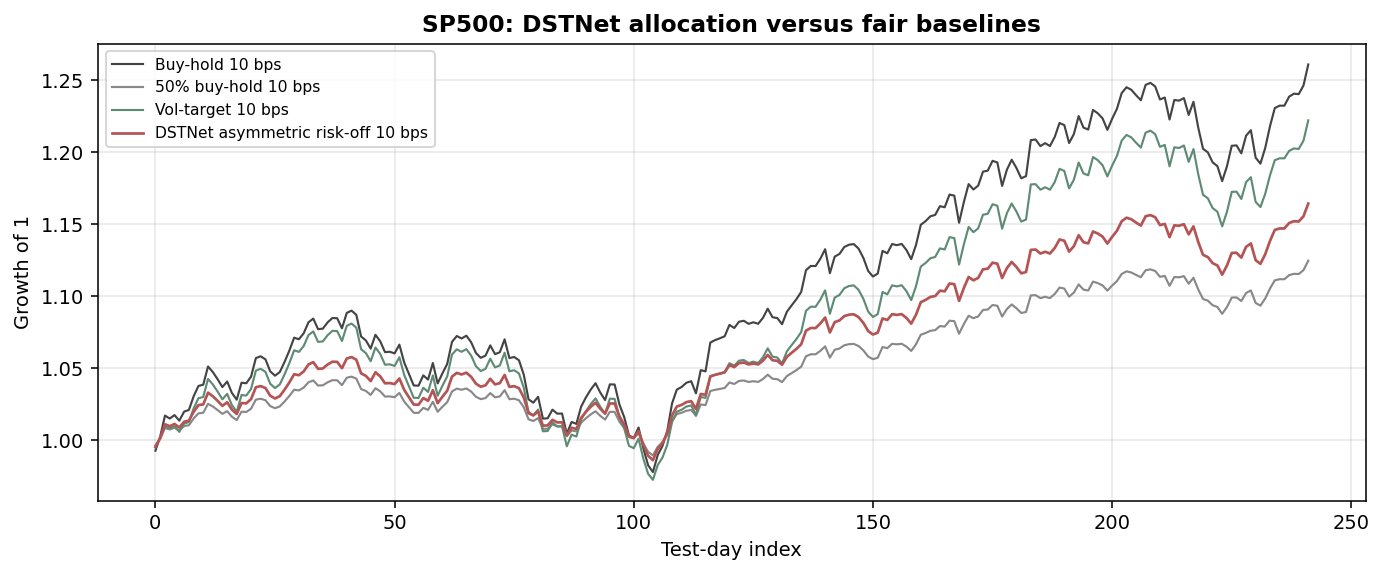}

\vspace{2mm}

\includegraphics[width=0.88\textwidth]{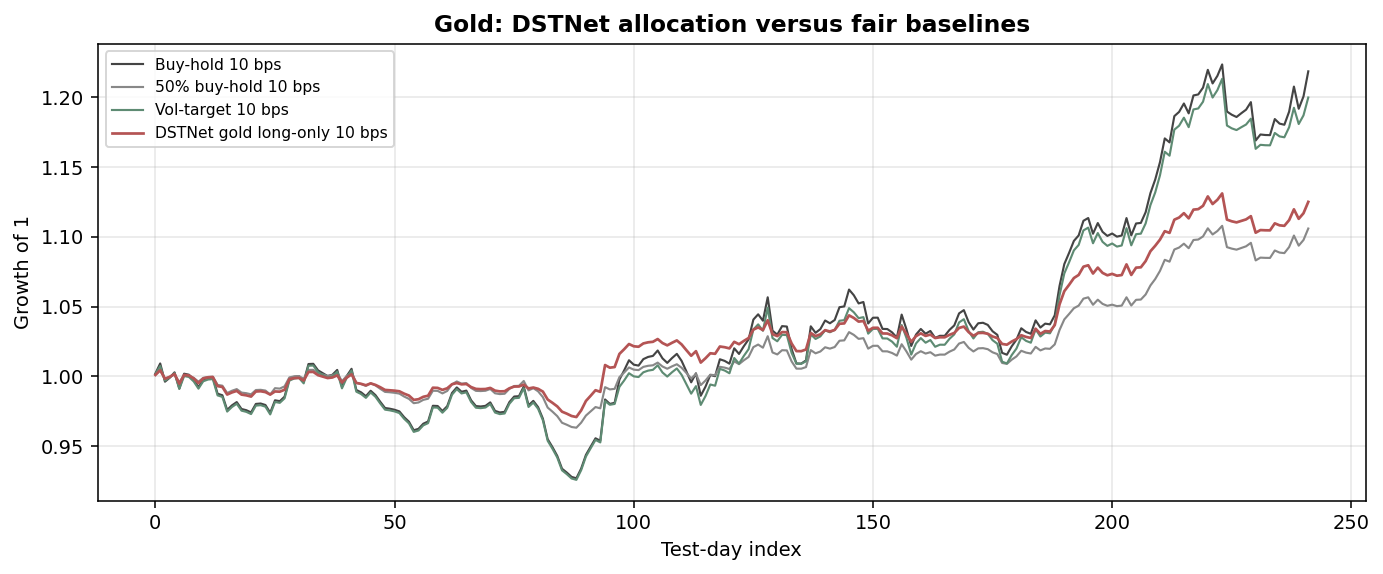}
\caption{Growth of one unit for two representative allocation diagnostics at 10 bps. Top: S\&P~500, where the DSTNet risk-off rule makes only the initial allocation and thereafter behaves as a static 65\% exposure. Its lower drawdown relative to buy-and-hold therefore reflects exposure scaling rather than forecast-conditioned timing. Bottom: Gold, where the DSTNet long-only rule changes exposure during the hold-out and experiences a shallower maximum drawdown than buy-and-hold. Its Sharpe ratio is 1.99 versus 1.74 for buy-and-hold, while cumulative return is lower at 12.51\% versus 21.88\%. The contrasting panels illustrate why the allocation experiment is treated as a diagnostic rather than evidence of a general trading advantage.}
\label{fig:allocation-curves}
\end{figure*}

\section{Random-Seed Stability: Cellwise Results}
\label{app:multiseed}
Tables~\ref{tab:multiseed-mae-appendix}--\ref{tab:multiseed-rmse-appendix} report the three-seed mean and sample standard deviation for every series--horizon cell. These values underlie the coefficient-of-variation and win-count summaries in Table~\ref{tab:multiseed-summary}.

\begin{table*}[!t]
\caption{DSTNet MAE across seeds 42, 101, and 202 on the primary hold-out. Values are mean $\pm$ sample standard deviation.}
\label{tab:multiseed-mae-appendix}
\centering
\small
\begin{tabular}{@{}lcccc@{}}
\toprule
\textbf{Series} & \textbf{$H=1$} & \textbf{$H=3$} & \textbf{$H=5$} & \textbf{$H=10$} \\
\midrule
KOSPI      & 20.17 $\pm$ 0.72 & 35.17 $\pm$ 3.69 & 46.01 $\pm$ 4.88 & 66.04 $\pm$ 5.35 \\
Nikkei 225 & 300.99 $\pm$ 0.22 & 516.20 $\pm$ 1.43 & 682.22 $\pm$ 0.46 & 924.19 $\pm$ 1.58 \\
S\&P 500   & 26.82 $\pm$ 0.14 & 48.32 $\pm$ 0.17 & 61.27 $\pm$ 0.39 & 91.43 $\pm$ 1.07 \\
DJI        & 174.18 $\pm$ 1.04 & 329.09 $\pm$ 3.14 & 420.07 $\pm$ 3.41 & 619.99 $\pm$ 1.84 \\
DAX        & 88.95 $\pm$ 0.25 & 159.15 $\pm$ 0.38 & 217.07 $\pm$ 2.25 & 314.20 $\pm$ 3.05 \\
NASDAQ     & 114.47 $\pm$ 0.39 & 208.46 $\pm$ 10.93 & 274.12 $\pm$ 23.94 & 366.24 $\pm$ 3.60 \\
NYSE       & 89.15 $\pm$ 3.35 & 168.78 $\pm$ 11.21 & 220.57 $\pm$ 20.23 & 331.38 $\pm$ 31.46 \\
Gold       & 11.91 $\pm$ 0.04 & 22.24 $\pm$ 0.33 & 31.52 $\pm$ 0.59 & 47.20 $\pm$ 0.72 \\
\bottomrule
\end{tabular}
\end{table*}

\begin{table*}[!t]
\caption{DSTNet MAPE across seeds 42, 101, and 202 on the primary hold-out. Values are mean $\pm$ sample standard deviation.}
\label{tab:multiseed-mape-appendix}
\centering
\small
\begin{tabular}{@{}lcccc@{}}
\toprule
\textbf{Series} & \textbf{$H=1$} & \textbf{$H=3$} & \textbf{$H=5$} & \textbf{$H=10$} \\
\midrule
KOSPI      & 0.786 $\pm$ 0.027 & 1.373 $\pm$ 0.138 & 1.802 $\pm$ 0.180 & 2.575 $\pm$ 0.196 \\
Nikkei 225 & 0.874 $\pm$ 0.001 & 1.490 $\pm$ 0.005 & 1.958 $\pm$ 0.002 & 2.641 $\pm$ 0.005 \\
S\&P 500   & 0.576 $\pm$ 0.003 & 1.041 $\pm$ 0.004 & 1.325 $\pm$ 0.009 & 1.970 $\pm$ 0.023 \\
DJI        & 0.483 $\pm$ 0.003 & 0.913 $\pm$ 0.008 & 1.169 $\pm$ 0.008 & 1.717 $\pm$ 0.004 \\
DAX        & 0.541 $\pm$ 0.001 & 0.971 $\pm$ 0.002 & 1.318 $\pm$ 0.009 & 1.899 $\pm$ 0.008 \\
NASDAQ     & 0.788 $\pm$ 0.003 & 1.440 $\pm$ 0.070 & 1.893 $\pm$ 0.154 & 2.533 $\pm$ 0.025 \\
NYSE       & 0.544 $\pm$ 0.020 & 1.031 $\pm$ 0.066 & 1.348 $\pm$ 0.118 & 2.013 $\pm$ 0.182 \\
Gold       & 0.580 $\pm$ 0.002 & 1.081 $\pm$ 0.012 & 1.525 $\pm$ 0.022 & 2.271 $\pm$ 0.027 \\
\bottomrule
\end{tabular}
\end{table*}

\begin{table*}[!t]
\caption{DSTNet RMSE across seeds 42, 101, and 202 on the primary hold-out. Values are mean $\pm$ sample standard deviation.}
\label{tab:multiseed-rmse-appendix}
\centering
\small
\begin{tabular}{@{}lcccc@{}}
\toprule
\textbf{Series} & \textbf{$H=1$} & \textbf{$H=3$} & \textbf{$H=5$} & \textbf{$H=10$} \\
\midrule
KOSPI      & 26.76 $\pm$ 0.78 & 45.64 $\pm$ 4.42 & 57.96 $\pm$ 5.60 & 82.05 $\pm$ 5.92 \\
Nikkei 225 & 384.05 $\pm$ 0.05 & 661.63 $\pm$ 0.38 & 852.16 $\pm$ 1.12 & 1140.73 $\pm$ 1.79 \\
S\&P 500   & 33.90 $\pm$ 0.07 & 59.95 $\pm$ 0.17 & 77.61 $\pm$ 0.29 & 111.79 $\pm$ 0.54 \\
DJI        & 223.58 $\pm$ 0.44 & 406.47 $\pm$ 2.47 & 519.45 $\pm$ 3.50 & 767.55 $\pm$ 1.56 \\
DAX        & 114.96 $\pm$ 0.10 & 201.80 $\pm$ 0.84 & 268.77 $\pm$ 1.84 & 379.93 $\pm$ 3.88 \\
NASDAQ     & 146.81 $\pm$ 0.72 & 257.07 $\pm$ 10.95 & 337.41 $\pm$ 25.85 & 448.04 $\pm$ 2.55 \\
NYSE       & 112.16 $\pm$ 3.67 & 211.24 $\pm$ 14.82 & 275.76 $\pm$ 23.70 & 391.78 $\pm$ 33.65 \\
Gold       & 16.17 $\pm$ 0.05 & 28.47 $\pm$ 0.46 & 40.64 $\pm$ 1.25 & 62.15 $\pm$ 0.06 \\
\bottomrule
\end{tabular}
\end{table*}

\section{Ablation Appendix: Full Panel Summary}
\label{app:ablations-appendix}
Table~\ref{tab:ablations-appendix} reports the macro-averaged MAE degradation across all eight series for the four ablation variants, and Table~\ref{tab:ablations-full-panel} gives the per-series MAE at $H=1$ and $H=10$ for the six indices not shown in the main text. The horizon-dependent point-estimate penalty for removing the DST tensor has the same direction across the panel, which is consistent with value beyond a single market's microstructure but does not by itself establish a significant per-series effect.

\begin{table*}[!t]
\caption{Panel-Mean Ablation Summary Across All Eight Series (MAE Degradation Relative to DSTNet Full). Values Are the Macro-Average Percentage Degradation Across the Seven Equity Indices and Gold, Averaged Over All Four Horizons ($H \in \{1, 3, 5, 10\}$).}
\label{tab:ablations-appendix}
\centering
\footnotesize
\begin{tabular}{@{}lc@{}}
\toprule
\textbf{Variant} & \makecell{\textbf{Mean MAE}\\\textbf{degradation (\%)}} \\
\midrule
DSTNet full (factorized)                  & 0.00 (reference) \\
\hspace{1em}(a) Without burn-in exclusion & +0.92 \\
\hspace{1em}(b) Joint STST                & +0.05 \\
\hspace{1em}(c) No-DST (static snapshot)  & +5.79 \\
\hspace{1em}(d) No-VAF (mean fusion)      & +2.94 \\
\bottomrule
\end{tabular}
\end{table*}

\begin{table*}[!t]
\caption{Full panel ablation study for KOSPI, Nikkei 225, DJI, DAX, NASDAQ, and NYSE. We report MAE at the shortest ($H=1$) and longest ($H=10$) horizons, with 95\% moving-block bootstrap CIs in brackets.}
\label{tab:ablations-full-panel}
\centering
\footnotesize
\setlength{\tabcolsep}{4pt}
\renewcommand{\arraystretch}{1.15}
\begin{tabular}{@{}lcccccc@{}}
\toprule
& \multicolumn{2}{c}{\textbf{Panel A: KOSPI}} & \multicolumn{2}{c}{\textbf{Panel B: Nikkei 225}} & \multicolumn{2}{c}{\textbf{Panel C: DJI}} \\
\cmidrule(lr){2-3}\cmidrule(lr){4-5}\cmidrule(lr){6-7}
\textbf{Variant} & \textbf{H=1} & \textbf{H=10} & \textbf{H=1} & \textbf{H=10} & \textbf{H=1} & \textbf{H=10} \\
\midrule
DSTNet full                   & \mc{19.78}{17.45}{23.11} & \mc{63.08}{52.64}{75.78} & \mc{300.74}{271.14}{349.52} & \mc{922.44}{791.03}{1101.35} & \mc{173.14}{151.94}{198.38} & \mc{620.80}{520.96}{751.92} \\
(a) Without burn-in exclusion & \mc{20.09}{17.68}{23.39} & \mc{63.52}{54.98}{78.52} & \mc{302.02}{268.84}{347.81} & \mc{926.32}{794.39}{1103.50} & \mc{173.83}{154.98}{202.31} & \mc{630.48}{534.29}{763.76} \\
(b) Joint STST                & \mc{19.73}{17.53}{22.54} & \mc{63.44}{54.51}{76.11} & \mc{300.61}{267.98}{347.30} & \mc{920.07}{765.86}{1106.02} & \mc{172.40}{153.66}{202.11} & \mc{621.41}{538.19}{745.48} \\
(c) No-DST (static snapshot)  & \mc{20.57}{18.31}{23.59} & \mc{67.58}{56.42}{82.86} & \mc{308.76}{273.34}{355.87} & \mc{979.92}{833.01}{1202.23} & \mc{180.09}{158.67}{210.16} & \mc{670.08}{561.06}{800.12} \\
(d) No-VAF (mean fusion)      & \mc{20.26}{17.93}{23.06} & \mc{65.01}{54.69}{77.50} & \mc{306.50}{273.22}{352.18} & \mc{949.03}{812.06}{1145.93} & \mc{177.39}{156.05}{206.98} & \mc{643.43}{538.82}{796.62} \\
\midrule
& \multicolumn{2}{c}{\textbf{Panel D: DAX}} & \multicolumn{2}{c}{\textbf{Panel E: NASDAQ}} & \multicolumn{2}{c}{\textbf{Panel F: NYSE}} \\
\cmidrule(lr){2-3}\cmidrule(lr){4-5}\cmidrule(lr){6-7}
\textbf{Variant} & \textbf{H=1} & \textbf{H=10} & \textbf{H=1} & \textbf{H=10} & \textbf{H=1} & \textbf{H=10} \\
\midrule
DSTNet full                   & \mc{88.92}{79.95}{103.10} & \mc{317.71}{275.01}{385.34} & \mc{114.36}{100.92}{132.10} & \mc{364.00}{304.05}{434.01} & \mc{93.01}{81.63}{107.80} & \mc{366.27}{306.97}{438.58} \\
(a) Without burn-in exclusion & \mc{89.87}{79.55}{103.50} & \mc{319.16}{267.40}{385.10} & \mc{115.36}{103.49}{134.52} & \mc{366.45}{313.14}{440.38} & \mc{93.54}{82.57}{106.45} & \mc{372.15}{317.27}{444.57} \\
(b) Joint STST                & \mc{89.44}{80.52}{102.09} & \mc{317.45}{264.88}{385.30} & \mc{114.02}{101.40}{132.51} & \mc{366.42}{308.43}{438.22} & \mc{93.24}{82.63}{108.78} & \mc{364.87}{313.07}{436.52} \\
(c) No-DST (static snapshot)  & \mc{92.42}{82.01}{105.36} & \mc{343.23}{285.64}{412.66} & \mc{118.67}{106.11}{138.88} & \mc{389.52}{332.76}{471.15} & \mc{97.16}{86.70}{113.60} & \mc{392.57}{333.24}{480.15} \\
(d) No-VAF (mean fusion)      & \mc{91.96}{81.49}{105.75} & \mc{326.87}{275.37}{404.14} & \mc{117.61}{105.18}{133.75} & \mc{378.16}{318.54}{458.26} & \mc{95.37}{84.67}{108.75} & \mc{375.35}{323.88}{454.89} \\
\bottomrule
\end{tabular}
\end{table*}


\end{document}